\documentclass[11pt]{article}

\usepackage{subfig}
\usepackage{tikz}
\usepackage{dsfont}
\usepackage{pgf,pgfplots}
\usepackage{mathrsfs}
\def\colorful{1}

\usepackage{amsthm,amsfonts,amsmath,amssymb,epsfig,color,float,graphicx,verbatim, enumitem}
\usepackage{multirow}
\usepackage{algorithm}
\usepackage[noend]{algpseudocode}

\newif\ifhyper\IfFileExists{hyperref.sty}{\hypertrue}{\hyperfalse}
\hypertrue
\ifhyper\usepackage{hyperref}\fi

\usepackage{enumitem}

\usepackage{environ}

\newcommand{\repeattheorem}[1]{\begingroup
  \renewcommand{\thetheorem}{\ref{#1}}\expandafter\expandafter\expandafter\theorem
  \csname reptheorem@#1\endcsname
  \endtheorem
  \endgroup
}

\newcommand{\repeatlemma}[1]{\begingroup
  \renewcommand{\thelemma}{\ref{#1}}\expandafter\expandafter\expandafter\lemma
  \csname replemma@#1\endcsname
  \endlemma
  \endgroup
}

\NewEnviron{replemma}[1]{\global\expandafter\xdef\csname replemma@#1\endcsname{\unexpanded\expandafter{\BODY}}\expandafter\lemma\BODY\unskip\label{#1}\endlemma
}

\NewEnviron{reptheorem}[1]{\global\expandafter\xdef\csname reptheorem@#1\endcsname{\unexpanded\expandafter{\BODY}}\expandafter\theorem\BODY\unskip\label{#1}\endtheorem
}

\def\nnewcolor{0}
\ifnum\nnewcolor=1
\newcommand{\nnew}[1]{{\color{red} #1}}
\fi
\ifnum\nnewcolor=0
\newcommand{\nnew}[1]{#1}
\fi

\ifnum\colorful=1
\newcommand{\new}[1]{{ #1}}

\else
\newcommand{\new}[1]{{#1}}

\fi

\newtheorem{theorem}{Theorem}[section]

\newtheorem{lemma}[theorem]{Lemma}
\newtheorem{informal theorem}[theorem]{Theorem (informal statement)}

\newtheorem{proposition}[theorem]{Proposition}

\newtheorem{claim}[theorem]{Claim}
\newtheorem{fact}[theorem]{Fact}

\newtheorem{remark}[theorem]{Remark}

\theoremstyle{definition}
\newtheorem{definition}[theorem]{Definition}
\newcommand{\eqdef}{\stackrel{{\mathrm {\footnotesize def}}}{=}}

\newtheorem{innercustomgeneric}{\customgenericname}
\providecommand{\customgenericname}{}
\newcommand{\newcustomtheorem}[2]{\newenvironment{#1}[1]
  {\renewcommand\customgenericname{#2}\renewcommand\theinnercustomgeneric{##1}\innercustomgeneric
  }
  {\endinnercustomgeneric}
}

\newcustomtheorem{customlem}{Lemma}
\newcustomtheorem{customclm}{Claim}
\newcustomtheorem{customfc}{Fact}

\newcommand{\wperp}{(\vec w^*)^{\perp_\vec w}}
\newcommand{\lp}{\left}
\newcommand{\rp}{\right}

\newcommand\snorm[2]{\left\| #2 \right\|_{#1}}
\renewcommand\vec[1]{\mathbf{#1}}

\DeclareMathOperator*{\E}{\mathbf{E}}
\newcommand{\proj}{\mathrm{proj}}

\newcommand{\mrm}{\mathrm}

\def\d{\mathrm{d}}

\newcommand{\sample}[2]{#1^{(#2)}}
\newcommand{\tr}{\mathrm{tr}}
\newcommand{\bx}{\mathbf{x}}
\newcommand{\by}{\mathbf{y}}
\newcommand{\bv}{\mathbf{v}}
\newcommand{\bu}{\mathbf{u}}
\newcommand{\bz}{\mathbf{z}}
\newcommand{\bw}{\mathbf{w}}
\newcommand{\br}{\mathbf{r}}

\newcommand{\Sp}{\mathbb{S}}

\newcommand{\err}{\mathrm{err}}

\newcommand{\R}{\mathbb{R}}

\newcommand{\Z}{\mathbb{Z}}

\newcommand{\eps}{\epsilon}

\newcommand{\pr}{\mathbf{Pr}}
\newcommand{\poly}{\mathrm{poly}}
\newcommand{\var}{\mathbf{Var}}
\newcommand{\cov}{\mathbf{Cov}}

\newcommand{\sgn}{\mathrm{sign}}
\newcommand{\sign}{\mathrm{sign}}

\newcommand{\opt}{\mathrm{OPT}}
\newcommand{\D}{\mathcal{D}}

\newcommand{\Ind}{\mathds{1}}
\newcommand{\1}{\Ind}

\newcommand{\littlesum}{\mathop{\textstyle \sum}}

\newcommand{\wh}{\widehat}
\usetikzlibrary{angles, quotes}
\usepackage{pgfplots}

\newcommand{\tsya}{{A}}
\newcommand{\tsyb}{{\alpha}}

\newcommand{\conb}{{\beta}}

\newcommand{\dotp}[2]{\left\langle #1, #2 \right\rangle}

\newcommand{\wstar}{\bw^{\ast}}
\newcommand{\citep}{\cite}
\newcommand{\x}{\vec x}

\newcommand{\bounded}{(3, L, R, \conb)}
\newcommand{\boundedU}{(3, L, R, U, \conb)}

\newcommand{\Dperp}{\D^{\pi_{\vec w}}_B}
\newcommand{\nr}{\zeta}

\title{A Polynomial Time Algorithm for Learning Halfspaces with Tsybakov Noise}

\author{
Ilias Diakonikolas\thanks{Supported by NSF Award CCF-1652862 (CAREER), a Sloan Research Fellowship, and
a DARPA Learning with Less Labels (LwLL) grant.}\\
UW Madison\\
{\tt ilias@cs.wisc.edu}\\
\and
Daniel M. Kane\thanks{Supported by NSF Award CCF-1553288 (CAREER) and a Sloan Research Fellowship.}\\
UC San-Diego \\
{\tt dakane@ucsd.edu}\\
\and
Vasilis Kontonis\\
UW Madison\\
{\tt kontonis@wisc.edu }\\
\and
Christos Tzamos\\
UW Madison\\
{\tt tzamos@wisc.edu}
\and
Nikos Zarifis\thanks{Supported in part by NSF Award CCF-1652862 (CAREER) and a DARPA Learning with Less Labels (LwLL) grant.}\\
UW Madison\\
{\tt zarifis@wisc.edu}\\
}

\begin{document}

\maketitle

We study the problem of PAC learning homogeneous halfspaces in the presence of Tsybakov noise.
In the Tsybakov noise model, the label of every sample is independently flipped with an adversarially
controlled probability that can be arbitrarily close to $1/2$ for a fraction of the samples.
{\em We give the first polynomial-time algorithm for this fundamental learning problem.}
Our algorithm learns the true halfspace within any desired accuracy $\eps$ and
succeeds under a broad family of well-behaved distributions including log-concave distributions.
Prior to our work, the only previous algorithm for this problem
required quasi-polynomial runtime in $1/\eps$.

Our algorithm employs a recently developed reduction~\cite{DKTZ20b} from learning
to certifying the non-optimality of a candidate halfspace. This prior work
developed a quasi-polynomial time certificate algorithm based on polynomial regression.
{\em The main technical contribution of the current paper is the first polynomial-time certificate algorithm.} Starting from a non-trivial warm-start, our algorithm performs
a novel ``win-win'' iterative process which, at each step, either finds a valid certificate
or improves the angle between the current halfspace and the true one.
Our warm-start algorithm for isotropic log-concave distributions involves
a number of analytic tools that may be of broader interest. These include
a new efficient method for reweighting the distribution in order to recenter it
and a novel characterization of the spectrum of the degree-$2$ Chow parameters.

\setcounter{page}{0}
\thispagestyle{empty}
\newpage

\section{Introduction} \label{sec:intro}

The main result of this paper is the first polynomial-time algorithm
for learning halfspaces in the presence of Tsybakov
noise under a broad family of distributions. Before we explain our contributions
in detail, we provide some context and motivation for this work.

\subsection{Background} \label{ssec:background}
Learning in the presence of noise is a central challenge in machine learning.
In this paper, we study the (supervised) binary classification setting,
where the goal is to learn a Boolean function from random
labeled examples with noisy labels. In more detail, we focus on the problem of learning
{\em homogeneous halfspaces} in Valiant's PAC learning model~\cite{val84} when the labels have
been corrupted by {\em Tsybakov noise}~\cite{tsybakov2004optimal}.

A (homogeneous) halfspace is any function $h_{\bw}: \R^d \to \{ \pm 1\}$
of the form $h_{\bw}(\bx) = \sgn(\langle \bw, \bx \rangle)$, where the vector $\bw \in \R^d$
is called the weight vector of $h_{\bw}$ and $\sgn: \R \to \{\pm 1\}$
is defined by $\sgn(t) = 1$ if $t \geq 0$ and $\sgn(t) = -1$ otherwise.
Halfspaces (or Linear Threshold Functions) are arguably the most fundamental
and extensively studied concept class in the learning theory and machine learning literature,
starting with early work in the 1950s and 60s~\cite{Rosenblatt:58, Novikoff:62, MinskyPapert:68}
and leading to fundamental and practically important techniques~\cite{Vapnik:98, FreundSchapire:97}.

Halfspaces are known to be efficiently learnable without noise, i.e., when the labels are consistent
with a halfspace, see, e.g.,~\cite{MT:94}. In the presence of noisy labels, the picture
is more muddled. In the agnostic model~\cite{Haussler:92, KSS:94} (when a constant fraction of the
labels can be adversarially chosen), learning halfspaces
is computationally hard~\cite{GR:06, FGK+:06short, Daniely16},
even under the Gaussian distribution~\cite{DKZ20, GGK20}.
This motivates the study of ``benign'' noise models, where positive results may be possible.
The most basic such model, known as Random Classification Noise (RCN)~\cite{AL88},
prescribes that each label is flipped independently with probability {\em exactly} $\eta<1/2$.
In the RCN model, halfspaces are known to be learnable in polynomial time~\cite{BlumFKV96}.

\new{The uniform noise assumption in the RCN model is commonly accepted to be unrealistic.
To address this issue, various natural noise models have been proposed and studied,
capturing a number of realistic noise sources. The two most prominent such models are,
in order of increasing difficulty, the Massart (or bounded) noise model~\cite{Massart2006},
and the Tsybakov noise model~\cite{tsybakov2004optimal}.}
In the Massart model, each label is flipped independently with probability {\em at most} $\eta<1/2$,
but the flipping probability can depend on the example.
The Tsybakov noise condition prescribes that the label of each example
is independently flipped with some probability which is controlled by an adversary
but is not uniformly bounded by a constant less than $1/2$.
In particular, the Tsybakov condition allows the flipping probabilities to be {\em arbitrarily close to $1/2$}
for a fraction of the examples.
More formally, we have the following definition:

\begin{definition}[PAC Learning with Tsybakov Noise] \label{def:tsybakov-learning}
Let $\mathcal{C}$ be a concept class of Boolean-valued functions over $X= \R^d$,
$\mathcal{F}$ be a family of distributions on $X$, $0< \eps <1$ be the error parameter,
and $0 \leq \tsyb < 1$, $\tsya> 0$ be parameters of the noise model.

Let $f$ be an unknown target function in $\mathcal{C}$.
A {\em Tsybakov example oracle}, $\mathrm{EX}^{\mathrm{Tsyb}}(f, \mathcal{F})$, works as follows:
Each time $\mathrm{EX}^{\mathrm{Tsyb}}(f, \mathcal{F})$ is invoked, it returns a
labeled example $(\bx, y)$, such that:
(a) $\bx \sim \D_{\bx}$, where $\D_{\bx}$ is a fixed distribution in $\mathcal{F}$, and
(b) $y = f(\bx)$ with probability $1-\eta(\bx)$ and $y = -f(\bx)$ with probability $\eta(\bx)$.
Here $\eta(\bx)$ is an {\em unknown} function  that satisfies the $(\tsyb, \tsya)$-Tsybakov noise condition.
That is, for any $0<t \leq 1/2$, $\eta(\bx)$ satisfies
$\pr_{\bx \sim \D_{\bx}}[\eta(\bx) \geq 1/2 - t] \leq \tsya \, t^{\frac{\tsyb}{1-\tsyb}}$.

Let $\D$ denote the joint distribution on $(\bx, y)$ generated by the above oracle.
A learning algorithm is given i.i.d. samples from $\D$ and its goal is to output
a hypothesis function $h: X \to \{\pm 1\}$ such that with high probability $h$ is $\eps$-close to $f$,
i.e., it holds $\pr_{\bx \sim \D_{\bx}} [h(\bx) \neq f(\bx)] \leq \eps$.
\end{definition}

The Tsybakov noise model was proposed in~\cite{MT99}, then refined in~\cite{tsybakov2004optimal},
and subsequently studied in a number of works, see, e.g.,~\cite{tsybakov2004optimal, BBL05, BJM06, BalcanBZ07, Hanneke2011, HannekeY15}. All these prior works address information-theoretic aspects of the model,
i.e., do not provide computationally efficient algorithms in high dimensions.
In fact, until very recently,  no non-trivial algorithm was known in the Tsybakov model for any non-trivial concept class,
{\em even under Gaussian marginals}.

The only algorithmic result we are aware of in this model
is the prior work by a subset of the authors~\cite{DKTZ20b}, which gave a {\em quasi-polynomial}
time algorithm for learning homogeneous halfspaces under a family of well-behaved distributions
(including log-concave distributions).

It is easy to see that the Tsybakov model becomes more challenging as the parameter $\tsyb$
in Definition~\ref{def:tsybakov-learning} decreases. In particular, it is well-known that
$\poly(d, 1/\eps^{1/\tsyb})$ samples are necessary (and sufficient)
to learn halfspaces in this model. That is, an exponential dependence in $1/\alpha$
is information-theoretically required for any algorithm that solves this problem.

We note that the error guarantee of Definition~\ref{def:tsybakov-learning} is a strong
identifiability guarantee for the true function, which is information-theoretically
impossible in the agnostic model. In the following remark, we emphasize that even
a constant factor approximation to the optimal misclassification error is insufficient
for identifiability. This is important as it implies a computational separation between
the Tsybakov and agnostic models, even under Gaussian marginals.

\begin{remark}[Identifiability versus Misclassification Error]\label{rem:misclass}
Definition~\ref{def:tsybakov-learning} requires that the learning algorithm identifies {\em the true function}
$f \in \mathcal{C}$ within arbitrary accuracy $\eps$. A related commonly used loss function is the misclassification
error, i.e., the probability $\pr_{(\bx, y) \sim \D} [h(\bx) \neq y]$. We note that having an efficient
algorithm with misclassification error $\opt+\eps$ for all $\eps>0$,
where $\opt = \inf_{g \in \mathcal{C}} \pr_{(\bx, y) \sim \D} [g(\bx) \neq y]$,
is equivalent to having an efficient algorithm with the guarantee of Definition~\ref{def:tsybakov-learning}.
We emphasize however that there is a major qualitative difference between achieving
misclassification error of $\opt+\eps$ and achieving error $c \cdot \opt+\eps$, for a constant $c>1$.
The latter guarantee only allows us to approximate $f$ within error $\Omega(\opt)$.
\end{remark}

Obtaining error $\opt+\eps$ in the agnostic model is known to require time
$d^{\poly(1/\eps)}$ for halfspaces under Gaussian marginals~\cite{KKMS:08, DKZ20, GGK20}.
On the positive side, \cite{ABL17, Daniely15, DKS18a, DKTZ20c}
gave $\poly(d/\eps)$ time algorithms for agnostically learning halfspaces under log-concave marginals.
These algorithms have error of $O(\opt)+\eps$, which
is significantly weaker as explained in Remark~\ref{rem:misclass}.

\subsection{Our Contributions} \label{ssec:results}

\nnew{The existence of a computationally efficient learning algorithm in the presence of Tsybakov noise
for any natural concept class and under any distributional assumptions has been a long-standing open
problem in learning theory. {\em In this work, we make significant progress in this direction by essentially
resolving the complexity of learning halfspaces in this model.}}

In this section, we formally state our contributions.
We start by defining the distribution family for which our algorithms succeed.

\begin{definition}[Well-Behaved Distributions] \label{def:wb}
For $L, R, U>0$ and $k \in \Z_+$, a distribution $\D_{\bx}$ on $\R^d$ is called $(k, L, R, U)$-well-behaved if for any projection $(\D_{\bx})_V$ of
$\D_{\bx}$ on a $k$-dimensional subspace $V$ of $\R^d$, the corresponding pdf
$\gamma_V$ on $V$ satisfies the following properties: (i) $\gamma_V(\bx) \geq L$, for all $\bx \in V$ with $\snorm{2}{\bx} \leq R$
(anti-anti-concentration), and (ii) $\gamma_V(\bx)\leq U$ for all $\x\in V$ (anti-concentration).
If, additionally, there exists $\beta \geq 1$ such that, for any $t > 0$ and unit vector $\vec w\in \R^d$, we have that
$\pr_{\bx \sim \D_{\bx}} [|\dotp{\vec w}{\vec x}|\geq t] \leq \exp(1-t/\conb)$ (sub-exponential concentration),
we call $\D_{\bx}$ $(k, L, R, U, \beta)$-well-behaved.
\end{definition}

\nnew{We focus on the case that the marginal distribution $\D_{\bx}$ on the examples
is well-behaved for some values of the relevant parameters.}
Definition~\ref{def:wb} specifies the concentration and anti-concentration conditions on the
low-dimensional projections of the data distribution that are required for our learning algorithm.
Throughout this paper, we will take $k=3$, i.e., we only require $3$-dimensional projections to have
such properties.

Interestingly, the class of well-behaved distributions is quite broad. In particular, it is easy
to show that the broad class of isotropic log-concave distributions is well-behaved for $L, R, U, \beta$ being universal constants.
Moreover, as Definition~\ref{def:wb} does not require a specific functional form for the underlying density function,
it encompasses a much more general set of distributions.

Since the complexity of our algorithm depends (polynomially) on $1/L, 1/R, U, \beta$,
we state here a simplified version of our main result for the case that these parameters
are bounded by a universal constant.
\nnew{To simplify the relevant theorem statements, we will sometimes
say that a distribution $\D$ of labeled examples in $\R^d \times \{ \pm 1\}$
is well-behaved to mean that its marginal distribution $\D_{\bx}$ is well-behaved.}
We show:

\begin{theorem}[Learning Tsybakov Halfspaces under Well-Behaved Distributions] \label{thm:pac-wb-inf}
Let $\D$ be a well-behaved isotropic distribution on $\R^{d} \times \{\pm 1\}$ that
satisfies the $(\tsyb,\tsya)$-Tsybakov noise condition with respect to an unknown halfspace
$f(\bx) = \sgn(\dotp{\vec w^{\ast}}{\bx} )$. There exists an algorithm that  draws
$N=  O_{\tsya, \tsyb}(d/\eps)^{O(1/\tsyb)}$ samples from $\D$, runs in $\poly(N,d)$ time,
and computes a vector $\wh{\vec w}$ such that, with high probability
we have that $\err_{0-1}^{\D_{\bx}}(h_{\wh{\bw}}, f) \leq \eps$.
\end{theorem}

\noindent See Theorem~\ref{thm:pac-wb} for a more detailed statement.

\medskip

For the class of log-concave distributions, we give a significantly more efficient algorithm:

\begin{theorem}[Learning Tsybakov Halfspaces under Log-concave Distributions] \label{thm:pac-lc-inf}
Let $\D$ be a distribution on $\R^{d} \times \{\pm 1\}$ that satisfies the $(\tsyb,\tsya)$-Tsybakov noise condition
with respect to an unknown halfspace $f(\bx) = \sgn(\dotp{\vec w^{\ast}}{\bx} )$ and is such that $\D_{\bx}$ is isotropic log-concave.
There exists an algorithm that  draws $N= \poly(d) \, O(\tsya/\eps)^{O(1/\tsyb^2)}$
samples from $\D$, runs in $\poly(N,d)$ time, and computes a vector $\wh{\vec w}$ such that,
with high probability, we have that $\err_{0-1}^{\D_{\bx}}(h_{\wh{\bw}}, f) \leq \eps$.
\end{theorem}

\noindent See Theorem~\ref{thm:pac-lc} for a more detailed statement.
Since the sample complexity of the problem is $\poly(d, 1/\eps^{1/\tsyb})$,
the algorithm of Theorem~\ref{thm:pac-lc-inf} is qualitatively close to best possible.

\subsection{Overview of Techniques} \label{ssec:techniques}
Here we give an intuitive summary of our techniques in tandem with a comparison
to the most relevant prior work.
A more detailed technical discussion is provided in the proceeding sections.

Our learning algorithms employ the certificate-based framework of \cite{DKTZ20b}.
At a high-level, this framework allows us to efficiently reduce the problem of
{\em finding} a near-optimal halfspace $h_{\widehat{\bw}}(\bx) = \sgn(\langle \widehat{\bw}, \bx \rangle)$
to the (easier) problem of {\em certifying} whether a candidate halfspace
$h_{\bw}(\bx) = \sgn(\langle \bw, \bx \rangle)$ is ``far'' from the optimal halfspace
$f(\bx) = \sgn(\langle \wstar, \bx \rangle)$. The idea is to use
a certificate algorithm (as a black-box) and combine it with an online convex optimization routine.
Roughly speaking, starting from an initial guess $\bw_0$ for $\wstar$,
a judicious combination of these two ingredients allows us to efficiently
compute a near-optimal halfspace $\widehat{\bw}$, i.e., one that the certifying algorithm cannot reject.
We note that a similar approach has been used in \cite{CKMY20} for converting non-proper learners
to proper learners in the Massart noise model.

With the aforementioned approach as the starting point,
the learning problem reduces to that of designing an efficient certifying algorithm.
In recent work~\cite{DKTZ20b}, the authors developed a certifying algorithm for Tsybakov
halfspaces based on high-dimensional polynomial regression.
This method leads to a certifying algorithm with sample complexity and runtime
$d^{\mrm{polylog}(1/\eps)}$, i.e., a quasi-polynomial upper bound.
As we will explain in Section~\ref{ssec:cert-int}, the~\cite{DKTZ20b} approach is inherently limited
to quasi-polynomial time and new ideas are needed to obtain a polynomial time algorithm.
{\em The main contribution of this paper is the design of a polynomial-time certificate algorithm
for Tsybakov halfspaces under well-behaved distributions.}

The key idea to design a certificate in the Tsybakov noise model
is the following simple but crucial observation: If $\wstar$ is the normal vector to true halfspace,
then for any non-negative function $T(\bx)$, it holds that $\E_{(\bx, y) \sim \D} [ T(\bx) y \, \dotp{\wstar}{\bx} ] \geq 0$.
On the other hand, for any $\vec w \neq \wstar$ there exists a non-negative function
$T(\bx)$ such that $\E_{(\bx, y) \sim \D} [T(\bx) \, y \, \dotp{\bw}{\bx}] < 0$.
In other words, there exists a {\em reweighting of the space} that makes the
expectation of $y \dotp{\vec w}{\vec x}$ negative (Fact~\ref{obs:optimal_condition}).
Note that we can always use as $T(\bx)$ the indicator of the disagreement region
between the candidate halfspace $h_{\vec w}(\bx)$ and the optimal halfspace $f(\bx) = h_{\wstar}(\bx)$.
Of course, since optimizing over the space of non-negative functions is intractable, we need
to restrict our search space to a ``simple'' parametric family of functions.
In \cite{DKTZ20b}, squares of low-degree polynomials were used, which led to a quasi-polynomial
upper bound.

In this work, we consider certifying functions of the form:
$$T(\bx) = \frac{1}{\dotp{\bw}{\vec x}}\1\left\{ \sigma_1 \leq \dotp{\bw}{\vec x} \leq \sigma_2 \;,
        -t_1 \leq \dotp{\vec v}{\proj_{\bw^{\perp}} \frac{\vec x}{\dotp{\bw}{\bx}}} \leq - t_2  \right\}$$
that are parameterized by a vector $\bv$ and scalar thresholds $\sigma_1, \sigma_2, t_1, t_2>0$.
Here $\proj_{\bw^{\perp}}$ denotes the orthogonal projection on the subspace orthogonal to $\bw$.
It will be important for our approach that functions of this form are specified by $O(d)$ parameters.

Of course, it may not be a priori clear why functions of this form can be used
as certifying functions in our setting. The intuition behind choosing functions of this simple form is
given in Section~\ref{ssec:cert-int}. In particular, in Claim~\ref{clm:exist_vector}, we show that
for any incorrect guess $\bw$ there {\em exists} a \emph{certifying vector} $\bv$ that makes the expectation
$\E_{(\bx, y) \sim \D} [T(\bx) \, y \, \dotp{\bw}{\bx}]$ negative.  In fact, the vector
$\vec v = \proj_{\bw^{\perp}} \wstar/\snorm{2}{\proj_{\bw^{\perp}} \wstar} := \wperp$ suffices for this purpose.

The key challenge is in finding such a certifying vector $\bv$
algorithmically. We note that our algorithm in general does not find $\wperp$.
But it does find a vector $\bv$ with similar behavior, in the sense of making the
$\E_{(\bx, y) \sim \D} [ T(\bx) \, y \, \dotp{\bw}{\bx}]$ sufficiently negative.
To achieve this goal, we take a two-step approach: The first
step involves computing an initialization vector $\vec v_{0}$
that has non-trivial correlation with $\wperp$.
In our second step, we give a perceptron-like update rule that iteratively improves
the initial guess until it converges to a certifying vector $\bv$. While this algorithm
is relatively simple, its correctness relies on a win-win analysis (Lemma~\ref{lem:improving_w_perp})
whose proof is  quite elaborate. In more detail, we show that for any {\em non-certifying} vector $\bv$
that is sufficiently correlated with $\wperp$, we can efficiently compute a direction that improves
its correlation to $\wperp$. We then argue (Lemma~\ref{lem:gradient}) that by choosing an appropriate
step size this iteration converges to a certifying vector within a small number of steps.

A subtle point is that the aforementioned analysis does not take place in the initial
space, where the underlying distribution is well-behaved and the labels are Tsybakov homogeneous halfspaces,
but in a transformed space. The transformed space is obtained by restricting our points in a band
and then performing an appropriate ``perspective'' projection on the subspace orthogonal to $\bw$
(Section~\ref{ssec:reduction}).
Fortunately, we are able to show (Proposition~\ref{prop:reduction})
that this transformation preserves the structure of the problem:
The transformed distribution remains well-behaved (albeit with somewhat worse parameters) and
satisfies the Tsybakov noise condition (again with somewhat worse parameters)
with respect to a potentially biased halfspace. In fact, this consideration motivated our use
of the perspective projection in the definition of $T(\bx)$.

It remains to argue how to compute an initialization vector $\vec v_{0}$ that acts
as a warm-start for our algorithm. Naturally, the sample complexity and runtime
of our certificate algorithm depend on the quality of the initialization.
The simplest way to initialize is by using a random unit vector.
With random initialization, we achieve initial correlation roughly $1/\sqrt{d}$, which
leads to a certifying algorithm with complexity $(d/\eps)^{O(1/\alpha)}$ (Theorem~\ref{thm:cert-wb}).
This simple initialization suffices to obtain Theorem~\ref{thm:pac-wb-inf} for the general class
of well-behaved distributions.

To obtain our faster algorithm for log-concave marginals (Theorem~\ref{thm:pac-lc-inf}),
we use the exact same approach described above
starting from a better initialization. Our algorithm to obtain a better starting vector
leverages additional structural properties of log-concave distributions.
Our initialization algorithm runs in $\poly(d)$ time (independent of $1/\alpha$)
and computes a unit vector  whose correlation with $\wperp$ is $\Omega(\eps^{1/\alpha})$
(Theorem~\ref{thm:init-lc}).

Specifically, our initialization algorithm works as follows:
\begin{enumerate}
\item It starts by conditioning on a random sufficiently narrow band around the current candidate $\bw$  and projecting the samples on the subspace $\vec w^{\perp}$.
\item It transforms the resulting distribution to ensure that it is isotropic log-concave through rescaling and rejection sampling.
\item It then computes the degree-$2$ \emph{Chow parameters} and uses them
to construct a low-dimensional subspace $V$ inside which $\wperp$ has sufficiently large projection.
This subspace $V$ is the span of the degree-$1$ Chow vector and the large
eigenvectors of the degree-$2$ Chow matrix.
\item Finally, the algorithm outputs a uniformly random vector in $V$
that can be shown to have the desired correlation with $\wperp$.
\end{enumerate}

The resulting distribution after the initial conditioning in Step~1 is still log-concave
and approximately satisfies the Tsybakov noise condition with respect to a near-origin
centered halfspace orthogonal to $\vec w$. However, the distribution may no longer
be zero-centered and may contain a tiny amount of non-Tsybakov noise ---
in the sense that we may end with points $\bx$ having $\eta(\bx)>1/2$.
As we can control the total non-Tsybakov noise, the latter is not a significant issue.
We address the former issue by reweighting the distribution to make it isotropic.
We do this by applying rejection sampling with probability $\min(1,\exp(-\langle \bx, \br\rangle))$,
for some vector $\vec r$ that we compute via SGD (so that the resulting mean is near-zero)
and then rescaling by the inverse covariance matrix.

After the first two steps, our goal is to find any vector with non-trivial correlation $\wperp$,
given that the underlying distribution is isotropic log-concave. We show that the labels $y$
must correlate with some degree-$2$ polynomial in $\dotp{\wperp}{\vec x}$ (Lemma~\ref{lem:quad-id}).
Our algorithm crucially exploits this property, along with recently established ``thin shell'' estimates~\cite{LV17}
for log-concave distributions, to show that a large part of this correlation is explained
by the vector of degree-$1$ Chow parameters and
the top few eigenvectors of the degree-$2$ Chow matrix (Lemma~\ref{lem:dim-ub}).
This implies that the subspace $V$ spanned by those vectors contains a non-trivial part of $\wperp$,
and thus a random vector from $V$  has non-trivial correlation with $\wperp$ with constant probability.

\subsection{Related Work}\label{ssec:related}

Recent work by a subset of the authors~\cite{DKTZ20b} gave the first non-trivial
algorithm for learning homogeneous halfspaces with Tsybakov noise under a family of ``well-behaved'' distributions.
The notion of well-behaved distributions in that work is somewhat different than ours, but also
contains log-concave distributions. The sample complexity and runtime of the
~\cite{DKTZ20b} algorithm is $d^{\mathrm{polylog}(1/\eps)}$
and the quasi-polynomial upper bound is tight for their techniques.

The Tsybakov noise model lies in between the Massart model~\cite{Sloan88, Massart2006}
and the agnostic model~\cite{Haussler:92, KSS:94}.
During the past five years, substantial algorithmic progress has been made
on learning with Massart noise in both the distribution-specific
setting~\cite{AwasthiBHU15, AwasthiBHZ16, ZhangLC17, YanZ17, Zhang20, DKTZ20}
and the distribution-free PAC model~\cite{DGT19, CKMY20}.
\new{The algorithmic techniques in these prior works
are known to inherently fail for the more challenging
Tsybakov noise model, and new ideas are needed for this more general setting.}

Learning in the agnostic model is known to be computationally hard, even under
well-behaved marginals. \nnew{Specifically, recent work~\cite{DKZ20, GGK20}
proved Statistical Query lower bounds of $d^{\poly(1/\eps)}$ for agnostically learning halfspaces
to error $\opt+\eps$ under Gaussian marginals. This lower bound bound is qualitatively matched
by the $L_1$ regression algorithm~\cite{KKMS:08}.}
A related line of work~\cite{KLS09, ABL17, Daniely15, DKS18a, DKTZ20c}
gave efficient algorithms for agnostically learning halfspaces under log-concave marginals.
While these algorithms run in $\poly(d/\eps)$ time, they achieve a ``semi-agnostic'' error guarantee
of $O(\opt)+\eps$, instead of $\opt +\eps$. As already mentioned in Remark~\ref{rem:misclass},
this guarantee is significantly weaker and cannot be used to approximate the true function
within any desired accuracy.

This work is part of the broader direction of designing robust learning algorithms
for a range of statistical models with respect to natural and challenging noise models.
A line of work~\cite{KLS09, ABL17, DKKLMS16, LaiRV16, DKK+17, DKKLMS18-soda,
DKS18a, KlivansKM18, DKS19, DKK+19-sever} has given efficient robust learners
for a range of settings in the presence of adversarial corruptions.
See~\cite{DK19-rs} for a recent survey on the topic.

\subsection{Structure of This Paper} \label{ssec:structure}
After the required preliminaries in Section~\ref{sec:prelims}, in Section~\ref{sec:cert-wb} we
give our certifying algorithm for the class of well-behaved distributions. In Section~\ref{sec:lc},
we give our more efficient certifying algorithm for log-concave distributions.
Finally, in Section~\ref{sec:ogd}, we review the certificate framework and put everything together to prove
our main results.

 \newcommand{\capfun}{\mathrm{cap}}

\section{Preliminaries}\label{sec:prelims}

For $n \in \Z_+$, let $[n] \eqdef \{1, \ldots, n\}$.  We will use small
boldface characters for vectors.  For $\bx \in \R^d$ and $i \in [d]$, $\bx_i$
denotes the $i$-th coordinate of $\bx$, and $\|\bx\|_2 \eqdef
(\littlesum_{i=1}^d \bx_i^2)^{1/2}$ denotes the $\ell_2$-norm of $\bx$.
We will use $\langle \bx, \by \rangle$ for the inner product of $\bx, \by \in
\R^d$ and $ \theta(\bx, \by)$ for the angle between $\bx, \by$.
We will use $\1_A$ to denote the characteristic function of the set $A$,
i.e., $\1_A(\x)= 1$ if $\x\in A$ and $\1_A(\x)= 0$ if $\x\notin A$.

Let $\vec e_i$ be the $i$-th standard basis vector in $\R^d$.
For $d\in \mathbb{N}$, let $\Sp^{d-1} \eqdef \{\bx \in \R^d:\|\bx\|_2 = 1 \}$ be the unit sphere.
We will denote by $\proj_U(\vec x)$ the projection of $\vec x$ onto the subspace
$U \subset \R^d$. For a subspace $U\subset\R^d$, let $U^{\perp}$ be the orthogonal complement of
$U$. For a vector $\vec w\in\R^d$, we use $\vec w^\perp$ to denote the subspace spanned by vectors
orthogonal to $\vec w$, i.e., $\vec w^\perp=\{\vec u\in \R^d: \dotp{\vec w}{\vec u}=0\}$.
Finally, we denote by $\vec w^{\perp_{\vec v}}$ the projection of the vector $\vec w$ on
the subspace $\vec v^{\perp}$ after normalization, i.e.,
$\vec w^{\perp_{\vec v}}= \frac{ \vec w - \dotp{\vec w}{\vec v} \, \vec v}{\snorm{2}{\vec w - \dotp{\vec w}{\vec v} \, \vec v}}$.

We use $\E[X]$ for the expectation of the random variable $X$ and
$\pr[\mathcal{E}]$ for the probability of event $\mathcal{E}$.

We study the binary classification setting where labeled examples $(\bx,y)$ are drawn
i.i.d. from a distribution $\D$ on $\R^d \times \{ \pm 1\}$.
We denote by $\D_{\bx}$ the marginal of $\D$ on $\vec x$.
The zero-one error between two hypotheses $f, h$ (with respect to $\D_{\bx}$) is
$\err_{0-1}^{\D_{\bx}}(f, h) \eqdef \pr_{\bx \sim \D_{\bx}}[f(\bx) \neq h(\bx)]$.

\section{Efficiently Certifying Non-Optimality}\label{sec:cert-wb}

In this section, we give an efficient algorithm that can certify whether a
candidate weight vector $\vec w$ defines a halfspace
$h_{\vec w}(\vec x) = \sgn(\langle \vec w, \vec x \rangle)$
that is far from the optimal halfspace $f(\vec x) = \sgn(\langle \vec w^{\ast}, \vec x \rangle)$.
Before we formally describe and analyze our algorithm, we provide some intuition.

\paragraph{Background: Certifying Non-Optimality.}
Our approach relies on the following simple but powerful idea,
introduced in~\cite{DKTZ20b}: If a candidate weight vector $\vec w$
defines a halfspace $h_{\vec w}(\vec x)  = \sign(\langle \vec w, \vec x \rangle)$ that differs
from the target halfspace $f(\vec x) = \sign(\langle \wstar, \vec x \rangle)$,
there exists a \emph{certifying function} of its non-optimality.
In more detail, there exists a {\em reweighting of the space} that makes the
expectation of $y \dotp{\vec w}{\vec x}$ negative.
This intuition is captured in Fact~\ref{obs:optimal_condition}, stated below.
We note that the only assumption required for this to hold
is that the underlying distribution on examples \nnew{assigns positive mass
to the symmetric difference of any two distinct halfspaces.}

\begin{fact}[Certifying Function] \label{obs:optimal_condition}
Let $\D$ be a distribution on $\R^d \times \{\pm 1\}$ such that:
(a) For any pair of distinct unit vectors $\bv,  \bu \in \R^d$, we have that
$\pr_{\bx \sim \D_\bx}[h_{\vec v}(\vec x) \neq  h_{\vec u}(\vec x)] > 0$.
(b) $\D$ satisfies the Tsybakov noise condition with optimal classifier $f(\vec x) = \sign(\langle \wstar, \vec x \rangle)$.
Then we have:
\begin{enumerate}
\item For any $T: \R^d \mapsto \R_+$,  we have that
$\E_{(\vec x, y) \sim \D}[ T(\vec x) \, y  \dotp{\wstar}{\vec x} ] \geq 0$.

\item For any non-zero vector $\vec w \in \R^d$ such that $\theta(\vec w, \vec w^{\ast}) >0$,
there exists a function $T: \R^d \mapsto \R_+$ satisfying
$\E_{(\vec x, y) \sim \D}[ T(\vec x) \, y  \dotp{\vec w}{\vec x} ] < 0$.
\end{enumerate}
\end{fact}
\begin{proof}
For the first statement, note that
\begin{align*}
\E_{(\vec x, y) \sim \D}[ T(\vec x) \, y \dotp{\wstar}{\vec x} ]
& = \E_{\vec x \sim \D_{\bx}}[T(\vec x) |\dotp{\wstar}{\bx}| (1-\eta(\x))]
- \E_{\vec x \sim \D_{\bx}}[ T(\vec x) |\dotp{\wstar}{\bx}| \, \eta(\x)] \\
&=  \E_{\vec x \sim \D_{\bx}} [T(\vec x) |\dotp{\wstar}{\bx}| \, (1-2 \eta(\vec x))] \geq 0\,,
\end{align*}
where we used the fact that $\eta(\bx) \leq 1/2$ and $T(\x)\geq 0$.

For the second statement, let $\vec w \neq \mathbf{0}$ and $\theta(\vec w, \vec w^{\ast}) >0$.
By picking as a certifying function $T$ the indicator function of the disagreement region between $f$ and $h_{\vec w}$,
i.e., $T(\bx) \eqdef \1\{ h_{\vec w}(\vec x) \neq f(\vec x) \}$, we have that
$$\E_{(\vec x, y) \sim \D}[T(\vec x) \, y \dotp{\bw}{\vec x} ]= - \E_{\vec x \sim \D_{\bx}} \left[T(\bx)|\dotp{\bw}{\bx}| \, (1-2 \eta(\vec x))\right] \;.$$
We claim that $\E_{\vec x \sim \D_{\bx}} [T(\bx)|\dotp{\bw}{\bx}| \, (1-2 \eta(\vec x))] >0$,
which proves the second statement. To see this, we use our assumption that the
symmetric difference between any pair of distinct homogeneous halfspaces has positive probability mass.
First, we note that from the Tsybakov condition (for any choice of parameters)
we have that $\pr_{\bx \sim \D_{\bx}}[ \eta(\bx)=1/2] = 0$. So, it suffices to show that $\E_{\vec x \sim \D_{\bx}} [T(\bx)|\dotp{\bw}{\bx}| ] >0$.

Let $\vec w'$ be a non-zero vector such that the hyperplane $\{\vec x : \langle \vec w' , \vec x \rangle = 0\}$
is contained in the disagreement region $\{\vec x : h_{\vec w}(\vec x) \neq f(\vec x) \}$ and
$\theta(\vec w, \vec w'), \theta(\vec w^{\ast}, \vec w') >0$. This implies that
$\{\vec x : h_{\vec w}(\vec x) \neq f(\vec x) \} \supset \{\vec x : h_{\vec w'}(\vec x) \neq f(\vec x) \}$
and $\pr_{\bx \sim \D_\bx}[ h_{\vec w'}(\vec x) \neq f(\vec x)]>0$. Note that $ | \langle \vec w , \vec x \rangle | > 0$
for all $\vec x$ with $h_{\vec w'}(\vec x) \neq f(\vec x)$. Therefore, we get that
$$\E_{\vec x \sim \D_{\bx}} [T(\bx)|\dotp{\bw}{\bx}| ] \geq
\E_{\vec x \sim \D_{\bx}} [\1\{ h_{\vec w'}(\vec x) \neq f(\vec x) \} |\dotp{\bw}{\bx}| ] >0 \;.$$
This completes the proof of Fact~\ref{obs:optimal_condition}.
\end{proof}

\paragraph{Main Result of this Section.}
Fact~\ref{obs:optimal_condition} shows that a certifying function exists.
However, in general, finding such a function is information-theoretically and computationally
hard. By leveraging our distributional assumptions,
we show that a certifying function of a specific simple form exists
and can be computed in polynomial time.

For the rest of this section, we work with distributions that are
$(3, L, R, \beta)$-well-behaved. These distributions satisfy the same
properties as those in Definition~\ref{def:wb}, except the anti-concentration
condition. (The anti-concentration condition is only required at the end of our analysis
in Section~\ref{sec:ogd} to deduce that small angle between two halfspaces implies
small 0-1 error.)

\begin{definition} \label{def:wb-no-U}
For $L, R>0$, $\beta \geq 1$, and $k \in \Z_+$, a distribution $\D_{\bx}$ on $\R^d$ is called $(k, L, R, \beta)$-well-behaved if
the following conditions hold: (i) For any projection $(\D_{\bx})_V$ of
$\D_{\bx}$ on a $k$-dimensional subspace $V$ of $\R^d$, the corresponding pdf
$\gamma_V$ on $V$ satisfies $\gamma_V(\bx) \geq L$, for all $\bx \in V$ with $\snorm{2}{\bx} \leq R$
(anti-anti-concentration).
(ii) For any $t > 0$ and unit vector $\vec w\in \R^d$, we have that
$\pr_{\bx \sim \D_{\bx}} [|\dotp{\vec w}{\vec x}|\geq t] \leq \exp(1-t/\conb)$ (sub-exponential concentration).
\end{definition}

Specifically, we have:

\begin{theorem}[Efficiently Certifying Non-Optimality] \label{thm:cert-wb}
Let $\D$ be a $(3, L, R, \beta)$-well-behaved isotropic distribution on $\R^{d} \times \{\pm 1\}$
that satisfies the $(\tsyb,\tsya)$-Tsybakov noise condition with respect to an unknown
halfspace $f(\bx) = \sgn(\dotp{\vec w^{\ast}}{\bx} )$.
Let $\vec w$ be a unit vector with $\theta(\vec w,\vec w^{\ast})\geq \theta$, where $\theta \in (0, \pi]$.
There is an algorithm that, given as input $\vec w$, $\theta$, and
$N= \large((\tsya/(L R)) \cdot (d/\theta) \large)^{O(1/\tsyb)}\log(1/\delta)$ samples from $\D$,
it runs in $\poly(N, d)$ time, and with probability at least $1-\delta$ returns a certifying function
$T_\bw:\R^d\mapsto \R_+$ such that
\begin{equation} \label{eqn:cert-ineq}
\E_{(\vec x, y) \sim \D} \left[T_{\bw}(\bx)  \, y  \dotp{\vec w}{\bx} \right] \leq
- \frac{1}{\beta} \left(\frac{L R \, \theta}{ \tsya \, d}\right)^{O(1/\alpha)} \;.
\end{equation}
\end{theorem}

\subsection{Intuition and Roadmap of the Proof} \label{ssec:cert-int}
In this subsection, we give an intuitive proof overview of Theorem~\ref{thm:cert-wb}
along with pointers to the corresponding subsections where the proof
of each component appears.  First, we discuss the specific
form of the certifying function that we compute. The proof of Fact~\ref{obs:optimal_condition}
shows that a valid choice for the certifying function would be the characteristic function of the
disagreement region between the candidate hypothesis $\vec w$ and the optimal halfspace $\wstar$,
i.e., $T_{\vec w}(\bx)= \1\{\sgn(\dotp{\vec w}{\bx}) \neq \sgn(\dotp{\wstar}{\bx}\}$.
Unfortunately, we do not know $\wstar$ (this is the vector we are trying to approximate!),
and therefore it is unclear how to algorithmically use this certifying function.

Our goal is to judiciously define a parameterized family of ``simple'' certifying functions
and optimize over this family to find one that acts similarly to the indicator of the disagreement region.
A natural attempt to construct a certifying function for a guess $\vec w$
would be to focus on a small ``band'' around the candidate halfspace $\vec w$.
This idea bears some similarity with the technique of ``localization", an approach going back to~\cite{Bartlett05},
which has previously seen success for the problem of efficiently learning homogeneous halfspaces with Massart noise~\cite{AwasthiBHU15, AwasthiBHZ16, Zhang20, DKTZ20}.
Unfortunately, this idea is inherently insufficient to provide us with a certifying function for the following reason:
Even an arbitrarily thin band around $\bw$ will assign more probability mass on points
that do not belong in the disagreement region, and therefore the expectation
$\E_{(\bx, y) \sim \D}[ \1\{\sigma_1 \leq \dotp{\bw}{\bx} \leq \sigma_2\} y\dotp{\bw}{\bx} ]$
will be positive. See Figure~\ref{fig:hardband} for an illustration.

\usetikzlibrary{patterns}
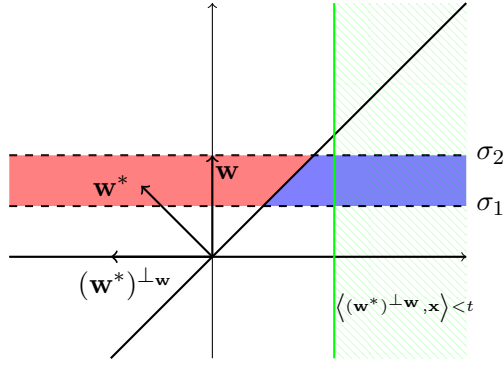
\begin{figure}	[h]
		\centering
    \begin{tikzpicture}[scale=1.35]
	\coordinate (start) at (0.5,0);
	\coordinate (center) at (0,0);
	\coordinate (end) at (0.5,0.5);
	\draw (-2,1) node[left] {};
	\draw (-2,0.5) node[left] {};
	\draw (2,2) node[above] {};
	\draw (3,1.2) node[above] {};
	\draw[black,dashed, thick](-2,1) -- (2.5,1);
	\draw[black,dashed, thick](-2,0.5) -- (2.5,0.5);
	\draw[fill=blue, opacity=0.5,draw=none] (0.5,0.5) -- (1 ,1)--(2.5,1)--(2.5,0.5);
\path[pattern=north west lines, pattern color=green,opacity=0.4] (1.2,-1)--(1.2,2.5)  --
	   (2.5,2.5)--(2.5,-1);
	\draw[fill=red, opacity=0.5,draw=none] (0.5,0.5) -- (1 ,1)--(-2,1)--(-2,0.5);
	\draw[->] (-2,0) -- (2.5,0) node[anchor=north west,black] {};
	\draw[->] (0,-1) -- (0,2.5) node[anchor=south east] {};
	\draw[thick,->] (0,0) -- (-0.7,0.7) node[anchor= south east,below,left=0.1mm] {$\wstar$};
	\draw[black,thick] (-1,-1) -- (2.5,2.5);
		\draw[black,thick] (-2,0) -- (2.5,0);
\draw[thick ,->] (0,0) -- (0,1) node[right=2mm,below] {$\bw$};
	 \draw[thick ,->] (0,0) -- (-1,0) node[right=2mm,below] {$(\bw^{\ast})^{\bot_{\vec w}}$};
	       \draw[green,thick](1.2,-1) -- (1.2,2.5);
	       	\draw (1.1,-0.5) node[right] {$\scriptscriptstyle \dotp{(\bw^{\ast})^{\bot_{\vec w}}}{\x}<t$};
	     \draw (2.5,1) node[right] {$\sigma_2$};
	 \draw (2.5,0.5) node[right] {$\sigma_1$};
\end{tikzpicture}
\caption{
  The indicator of a band $\{\bx: \sigma_1 \leq \dotp{\bw}{\bx} \leq
  \sigma_2\}$ cannot be used as a certificate even when there is no noise
  and the underlying distribution is the standard Gaussian: the
  contribution of the positive points (red region) is larger than the
  contribution of the negative points (blue region). On the other hand, taking the intersection
  of the band and the halfspace with normal vector $\wperp$ and a sufficiently
  negative threshold $t<0$ gives us a subset of the disagreement region
  (intersection of blue and green regions).
}
\label{fig:hardband}
\end{figure}

Intuitively, we need a way to \emph{boost} the contribution of the disagreement region.
One way to achieve this is by constructing a smooth reweighting of the space.
In particular, we can look in the direction of the projection of $\wstar$
on the orthogonal complement of $\bw$, i.e., the vector
$$
\wperp = \frac{\proj_{\bw^{\perp}}(\wstar)}{\snorm{2}{\proj_{\bw^{\perp}}(\wstar)}}\;,
$$
that lies in the $2$-dimensional subspace spanned by $\vec w$ and $\wstar$; see
Figure~\ref{fig:hardband}. Notice that the disagreement region is a
subset of the points that have negative inner product with $\wperp$.
Therefore, a candidate reweighting can be obtained by using
a polynomial $p(\dotp{\wperp}{\bx})$ of moderately large degree
that will boost the points that lie in the disagreement region.
This was the approach used in the recent work \cite{DKTZ20b}.
Since $\wperp$ is not known, one needs to formulate a convex program
(SDP) over the space of all $d$-variate polynomials of sufficiently large
degree $k$ implying that the corresponding SDP has $d^{\Omega(k)}$ variables.
Unfortunately, it is not hard to show that the required degree cannot be smaller than
$\Omega(\log(1/\eps))$. Therefore, this approach can only give
a $d^{\Omega(\log(1/\eps))}$, i.e., quasi-polynomial, certificate algorithm.

In this work, we instead use a \emph{hard threshold function} together with a band to
isolate (a non-trivial subset of) the disagreement region. In more detail, we consider a function
of the form $\1\{\dotp{\wperp}{\bx} < t \}$ for some scalar threshold $t$; see
Figure~\ref{fig:hardband}.  Since $\wperp$ is unknown, we need to find a certifying vector
$\vec v$ that is perpendicular to $\vec w$, i.e., $\vec v \in \bw^\perp$ and acts similarly to $\wperp$.
This leads us to the following \textbf{non-convex} optimization problem
$$
\min_{t\in \R, \vec v \in \bw^{\perp}} \E_{(\bx, y) \sim \D}[ \1\{\sigma_1 \leq \dotp{\bw}{\bx} \leq \sigma_2\}
\1\{\dotp{\bv}{\bx} < t\} \dotp{\bw}{\bx} ] \,.
$$
Thus far, we have succeeded in reducing the number of parameters that we want to compute
down to $O(d)$, but now we are faced with a non-convex optimization problem.
Our main result is an efficient algorithm that computes a \emph{certifying vector} $\vec v$
and a threshold $t$ that does not necessarily minimize the above non-convex objective,
but still suffice to make the corresponding expectation sufficiently negative.

We now describe the main steps we use to compute the certifying vector $\vec v$.  The
first obstacle we need to overcome is that, for $\vec v \in \bw^\perp$, the corresponding
instance fails to satisfy the Tsybakov noise condition.  In particular, when we
project the datapoints on $\vec w^\perp$, the region close to the boundary of
the optimal halfspace becomes ``fuzzy" even without noise: Points with
different labels are mapped to the same point of $\bw^{\perp}$; see
Figure~\ref{fig:orth_proj}. We bypass this difficulty by using a {\em perspective projection}
to map the datapoints onto $\bw^{\perp}$.
For non-zero vectors $\vec w, \vec x \in \R^d$, the perspective projection of
$\vec x$ on $\vec w$ is defined as follows:
\begin{equation} \label{eq:perspective_projection}
  \pi_{\vec w}(\bx) \eqdef \proj_{\bw^{\perp}}  \frac{\bx}{\dotp{\bw}{\bx}} \,.
\end{equation}
Notice that without noise the perspective projection keeps the dataset linearly
separable (see Figure~\ref{fig:pros_projection}), which means that after we
perform this projection the label noise of the resulting instance will again
satisfy the Tsybakov noise condition.  In addition, we show that this
transformation will preserve the crucial distributional properties
(concentration, anti-anti-concentration) of the underlying marginal distribution $\D_{\bx}$.
For a detailed discussion and analysis of this data transformation,
see Subsection~\ref{ssec:reduction}.

\usetikzlibrary{shapes.misc}

\tikzset{cross/.style={cross out, draw=black, minimum size=2*(#1-\pgflinewidth), inner sep=0pt, outer
sep=0pt},cross/.default={1pt}}

\begin{figure}
\subfloat[ Orthogonal projection.]{
  \begin{minipage}[t]{0.47\textwidth}

    \centering
    \begin{tikzpicture}[scale=1.35]
      \coordinate (start) at (0.5,0);
      \coordinate (center) at (0,0);
      \coordinate (end) at (0.5,0.5);
      \draw (-2,1) node[left] {};
      \draw (-2,0.5) node[left] {};
      \draw (2,2) node[above] {};
      \draw (3,1.2) node[above] {};
      \draw[black,dashed, thick,red](-2,1) -- (3.75,1);
      \draw[black,dashed, thick, red](-2,0.5) -- (3.75,0.5);
      \draw[fill=blue, opacity=0.5,draw=none] (0.5,0.5) -- (1 ,1)--(1,0.5)--(0.5,0.5);
      \draw[fill=red, opacity=0.5,draw=none] (0.5,0.5) -- (1 ,1)--(0.5,1)--(0.5,0.5);
      \draw[black,dashed, thick](-2,2) -- (3.75,2);
      \draw[black,dashed, thick](0.5,0) -- (0.5,2);
      \draw[black,dashed, thick](1,0) -- (1,2);
      \draw (0.5,2) node[cross=2.5pt,red,rotate=10] {};
      \draw (0.55,2) node[cross=2.5pt,red,rotate=10] {};
      \draw (0.6,2) node[cross=2.5pt,red,rotate=10] {};
      \draw (0.65,2) node[cross=2.5pt,red,rotate=10] {};
      \draw (0.7,2) node[cross=2.5pt,red,rotate=10] {};
      \draw (0.75,2) node[cross=2.5pt,blue,rotate=10] {};
      \draw (0.8,2) node[cross=2.5pt,blue,rotate=10] {};
      \draw (0.85,2) node[cross=2.5pt,blue,rotate=10] {};
      \draw (0.9,2) node[cross=2.5pt,blue,rotate=10] {};
      \draw (0.95,2) node[cross=2.5pt,blue,rotate=10] {};
      \draw (0.53,2) node[cross=2.5pt,blue,rotate=25] {};
      \draw (0.63,2) node[cross=2.5pt,blue,rotate=25] {};
      \draw (0.85,2) node[cross=2.5pt,red,rotate=25] {};
      \draw (0.90,2) node[cross=2.5pt,red,rotate=25] {};
      \draw[->] (-2,0) -- (3.8,0) node[anchor=north west,black] {};
      \draw[->] (0,-1) -- (0,2.5) node[anchor=south east] {};
      \draw[thick,->] (0,0) -- (-0.7,0.7) node[anchor= south east,below,left=0.1mm] {$\wstar$};
      \draw[black] (-1,-1) -- (2,2);
\draw[thick ,->] (0,0) -- (0,1) node[right=2mm,below] {$\bw$};
      \draw[thick ,->] (0,0) -- (-1,0) node[right=2mm,below] {$(\bw^{\ast})^{\bot_{\vec w}}$};
    \end{tikzpicture}
    \label{fig:orth_proj}
  \end{minipage}
}
\subfloat[Perspective projection.]{
  \begin{minipage}[t]{0.47\textwidth}
    \centering
    \begin{tikzpicture}[scale=1.35]
      \coordinate (start) at (0.5,0);
      \coordinate (center) at (0,0);
      \coordinate (end) at (0.5,0.5);
      \draw (-2,1) node[left] {};
      \draw (-2,0.5) node[left] {};
      \draw (2,2) node[above] {};
      \draw (3,1.2) node[above] {};
      \draw[black,dashed, thick,red](-2,1) -- (3.75,1);
      \draw[black,dashed, thick, red](-2,0.5) -- (3.75,0.5);
      \draw[fill=blue, opacity=0.5,draw=none] (0.5,0.5) -- (1 ,1)--(3/2,1)--(1.5/2,0.5);
      \draw[fill=red, opacity=0.5,draw=none] (0.5,0.5) -- (1 ,1)--(0.5,1)--(0.25,0.5);
      \draw[black,dashed, thick](-2,2) -- (3.75,2);
\draw[thick,->] (0,0) -- (-0.7,0.7) node[anchor= south east,below,left=0.1mm] {$\wstar$};
      \draw[black,dashed] (0,0) -- (1,2);
      \draw[black,dashed] (0,0) -- (3,2);
      \draw[->] (-2,0) -- (3.8,0) node[anchor=north west,black] {};
      \draw[->] (0,-1) -- (0,2.5) node[anchor=south east] {};
      \draw[thick,->] (0,0) -- (-0.7,0.7) node[anchor= south east,below,left=0.1mm] {$\wstar$};
      \draw[black] (-1,-1) -- (2,2);
\draw[thick ,->] (0,0) -- (0,1) node[right=2mm,below] {$\bw$};
      \draw[thick ,->] (0,0) -- (-1,0) node[right=2mm,below] {$(\bw^{\ast})^{\bot_{\vec w}}$};

      \draw (1,2) node[cross=2.5pt,red,rotate=10] {};
      \draw (1.1,2) node[cross=2.5pt,red,rotate=10] {};
      \draw (1.2,2) node[cross=2.5pt,red,rotate=10] {};
      \draw (1.3,2) node[cross=2.5pt,red,rotate=10] {};
      \draw (1.4,2) node[cross=2.5pt,red,rotate=10] {};
      \draw (1.5,2) node[cross=2.5pt,red,rotate=10] {};
      \draw (1.6,2) node[cross=2.5pt,red,rotate=10] {};
      \draw (1.7,2) node[cross=2.5pt,red,rotate=10] {};
      \draw (1.8,2) node[cross=2.5pt,red,rotate=10] {};
      \draw (1.9,2) node[cross=2.5pt,red,rotate=10] {};

      \draw (2.1,2) node[cross=2.5pt,blue,rotate=10] {};
      \draw (2.2,2) node[cross=2.5pt,blue,rotate=10] {};
      \draw (2.3,2) node[cross=2.5pt,blue,rotate=10] {};
      \draw (2.4,2) node[cross=2.5pt,blue,rotate=10] {};
      \draw (2.5,2) node[cross=2.5pt,blue,rotate=10] {};
      \draw (2.6,2) node[cross=2.5pt,blue,rotate=10] {};
      \draw (2.7,2) node[cross=2.5pt,blue,rotate=10] {};
      \draw (2.8,2) node[cross=2.5pt,blue,rotate=10] {};
      \draw (2.9,2) node[cross=2.5pt,blue,rotate=10] {};
    \end{tikzpicture}
    \label{fig:pros_projection}
  \end{minipage}
}
  \caption{
      The dotted line on top of the figures corresponds to the subspace
      $\bw^{\perp}$.  When we project the points to $\bw^{\perp}$ orthogonally,
      we map points with different labels to the same point of $\bw^{\perp}$
      and obtain the ``fuzzy" region where blue points (classified as negative by
      $\wstar$) overlap with red points (positive according to $\wstar$).  On the
      other hand, the perspective projection defined in
      Equation~\ref{eq:perspective_projection} preserves linear
      separability.
      }
\end{figure}
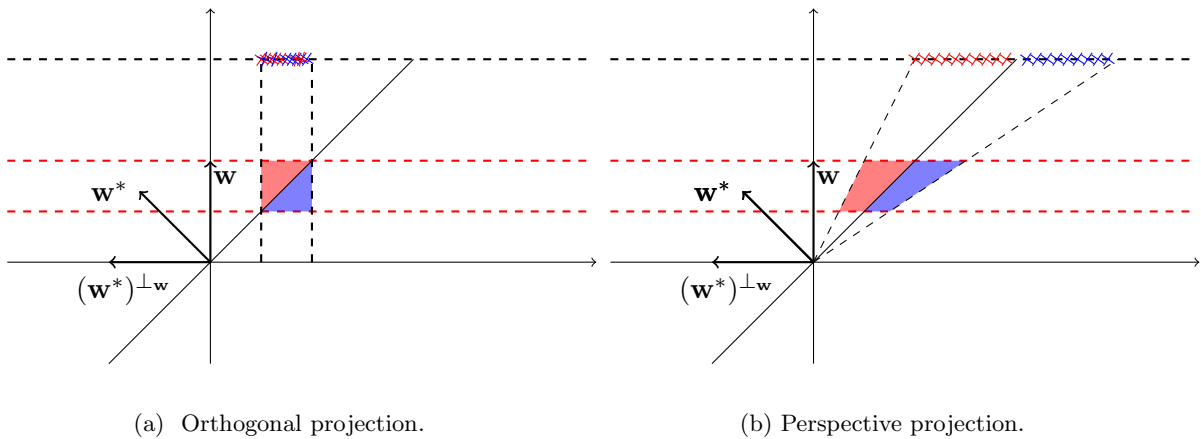

Given this setup, the certificate that our algorithm will compute for a
candidate weight vector $\vec w \in \R^d$ is a function of the form
\begin{align} \label{eq:certificate_form}
  T_\bw(\bx) =
  \frac{1}{\dotp{\bw}{\vec x}}
  \1\left\{ \sigma_1 \leq \dotp{\bw}{\vec x} \leq \sigma_2
    \, ,
    -t_1 \leq \dotp{\vec v}{\pi_{\vec w}(\vec x)} \leq - t_2
  \right\}
  =: \frac{\psi(\bx)}{\dotp{\bw}{\vec x}}
  \,,
\end{align}
for some vector $\vec v \in \R^{\nnew{d}}$ and scalars
$\sigma_1, \sigma_2, t_1, t_2 > 0$.
For an illustration, in Figure~\ref{fig:pros_projection} we plot the set of the
indicator function $\psi(\bx)$ which is a (high-dimensional) trapezoid.

It is not difficult to verify that by choosing $\vec v = \wperp$ and appropriately picking
$\sigma_1, \sigma_2, t_1, t_2 $, the corresponding certificate function $T_{\bw}$
resembles the indicator function of the disagreement region and certifies the
\emph{non-optimality} of the candidate halfspace $\bw$.  In the following claim,
we prove that for any non-optimal halfspace there exists a certifying function of the above form.

\begin{claim}\label{clm:exist_vector}
Let $\D$ be a $(3, L, R, \beta)$-well-behaved isotropic distribution on $\R^{d} \times \{\pm 1\}$
that satisfies the $(\tsyb,\tsya)$-Tsybakov noise condition with respect to an unknown halfspace
$f(\bx) = \sgn(\dotp{\vec w^{\ast}}{\bx} )$.
\nnew{Fix any non-zero vector $\vec w$ such that  $\theta(\vec w, \wstar)>0$}.
Then, by setting $\vec v = \wperp$ in the definition \eqref{eq:certificate_form} of $T_{\bw}(\bx)$,
there exist $\sigma_1, \sigma_2, t_1, t_2 >0$ such that
$\E_{(\bx, y) \sim \D}[T_\bw(\bx) \, y \dotp{\vec w}{\vec x}] < 0$.
\end{claim}

\nnew{We note here that the proof of Claim~\ref{clm:exist_vector} is sketched below for the sake of intuition
and is not required for the subsequent analysis.}

\begin{proof}[Proof Sketch.]
Setting $\vec v = \wperp$ in \eqref{eq:certificate_form}, we have
\begin{align*}
\E_{(\bx, y) \sim \D}[ T_\bw(\bx) \, y \dotp{\vec w}{\vec x}] =
\E_{(\bx, y) \sim \D}\left[\psi(\bx) \, y \right]
= \E_{(\bx, y) \sim \D}\left[\psi(\bx) \, (1-2 \eta(\bx)) \, \sgn(\dotp{\wstar}{\bx})\right] \;.
\end{align*}
We will show that by appropriate choices of $\sigma_1, \sigma_2, t_1, t_2$
the indicator $\psi(\bx)$ above corresponds to a subset of the disagreement region
$\{\bx: \sgn(\dotp{\vec w}{\vec x}) \neq \sgn(\dotp{\wstar}{\vec x})\}$.
See Figure~\ref{fig:Tsybakov_Regions} for an illustration.
More precisely, since the distribution satisfies an anti-anti-concentration property,
we can choose $\sigma_1, \sigma_2 = \Theta(R)$,
so that inside the band $\{ \sigma_1 \leq \dotp{\bw}{\bx} \leq \sigma_2\}$ there is non-zero probability mass.
In particular, by setting $\sigma_1 = \rho R/2$ and $\sigma_2 = \rho R/\sqrt{2}$, for some $\rho \in (0,1]$,
we have that the band has mass roughly $\Omega(\rho R^3)$.
For these choices of $\sigma_1$ and $\sigma_2$, we can pick
$t_1 = \Theta(R/\rho)$ and guarantee that the slope of the corresponding line
in the two-dimensional subspace is sufficiently small, so that we get a trapezoid whose intersection
with the aforementioned horizontal band is large (see Figure~\ref{fig:Tsybakov_Regions}).
It remains to tune the parameter $t_2$.  Since $\theta = \theta(\vec w, \wstar)$ is known,
we may pick $t_2 = \Theta(R\tan\theta/\rho)$ in order to make sure that the trapezoid is a subset
of the disagreement region between $\wstar$ and $\bw$.
\end{proof}

\begin{figure}[h!]
  \centering
  \begin{tikzpicture}[scale=1.5]
    \coordinate (start) at (0.5,0);
    \coordinate (center) at (0,0);
    \coordinate (end) at (0.5,0.5);
         \path[pattern=north west lines, pattern color=red,opacity=0.2] (0,0) --
         (2,2)--(3.8,2.0)--(3.8,0);
                  \path[pattern=north west lines, pattern color=red,opacity=0.2] (0,0) --
         (-1,-1)--(-2,-1.0)--(-2,0);
    \draw (-2,1) node[left] {$\sigma_2$};
    \draw (-2,0.5) node[left] {$\sigma_1$};
    \draw (2,2) node[above] {$\dotp{(\vec w^{\ast})^{\bot_{\vec w}} }{\pi_{\vec w}(\bx) }= -t_2$};
    \draw (3,1.2) node[above] {$\dotp{(\vec w^{\ast})^{\bot_{\vec w}}} { \pi_{\vec w}(\bx)} = -t_1$};
    \draw[black,dashed, thick,red](-2,1) -- (3.75,1);
    \draw[black,dashed, thick, red](-2,0.5) -- (3.75,0.5);
    \draw[fill=blue, opacity=0.5,draw=none] (0.5,0.5) -- (1 ,1)--(3/1.2,1)--(1.5/1.2,0.5);
\draw[->] (-2,0) -- (3.8,0) node[anchor=north west,black] {};
    \draw[->] (0,-1) -- (0,2.5) node[anchor=south east] {};
    \draw[thick,->] (0,0) -- (-0.7,0.7) node[anchor= south east,below,left=0.1mm] {$\wstar$};
    \draw[black,thick] (-1,-1) -- (2,2);
    \draw[black] (0,0) -- (3,1.2);
\draw[thick ,->] (0,0) -- (0,1) node[right=2mm,below] {$\bw$};
    \draw[thick ,->] (0,0) -- (-1,0) node[right=2mm,below] {$(\bw^{\ast})^{\bot_{\vec w}}$};

\end{tikzpicture}
  \caption{The function $\psi(\bx)$ for $\vec v = \wperp =
    \frac{\proj_{\bw^{\perp}}(\wstar)}{\snorm{2}{\proj_{\bw^{\perp}}(\wstar)}}$ defined in \eqref{eq:certificate_form} and
    appropriate scalars $\sigma_1,\sigma_2, t_1, t_2$ is the indicator of a subset of the disagreement region
    $\{\bx: \sgn(\dotp{\vec w}{\vec x}) \neq \sgn(\dotp{\wstar}{\bx})\}$.}
  \label{fig:Tsybakov_Regions}
\end{figure}
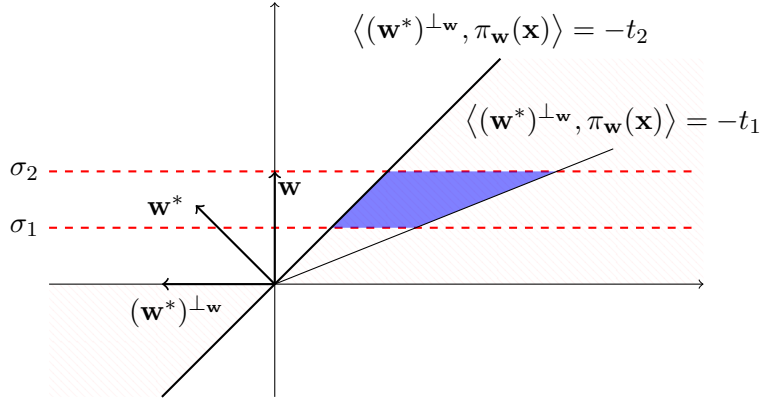

From the above proof, it is clear that one does not really need to optimize the
scalars $\sigma_1, \sigma_2, t_1$. Their values can be chosen according to
the parameters of the underlying well-behaved distribution.
Our optimization problem will be with respect to the vector $\vec v$ and the threshold $t_2$.
However, optimizing the expectation of the certifying function $T_{\bw}$ of
Equation~\eqref{eq:certificate_form} is still a non-convex problem.  Given a
candidate certifying vector $\bv_0$ that has non-trivial correlation with
$\wperp$, our main structural result is a \textbf{win-win} statement showing
that either there exists a threshold $t_2$ that, together with $\vec v_0$,
makes the corresponding expectation of $T_{\vec w}$ sufficiently negative,
or a perceptron-like update rule \emph{will improve the correlation} between
$\wperp$ and $\bw$. In particular, we show that after roughly $\poly(d/\eps)$ updates
the correlation between the guess $\vec v$ and $\wperp$ will be sufficiently large
so that there exists some threshold $t_2$ that makes $\vec v$ a certifying
vector. Having such a vector $\vec v$, it is easy to optimize over all possible
thresholds and find a value for $t_2$ that works.
For the formal statement of this claim and its proof, see
Subsection~\ref{ssec:impr-cert} and Proposition~\ref{prop:cert-wb}.

\subsection{Data Transformation} \label{ssec:reduction}

In this subsection, we show that we can simplify the problem of searching for a
certifying vector $\vec v$ in $T_{\vec w}(\bx)$ defined in
Equation~\eqref{eq:certificate_form} by projecting the samples to an appropriate
$(d-1)$-dimensional subspace via the perspective projection~\eqref{eq:perspective_projection}.
The main proposition of this subsection (Proposition~\ref{prop:reduction}) shows
that this operation \nnew{in some sense preserves the structure of the problem.
In more detail, the transformed distribution remains well-behaved and
satisfies the Tsybakov noise condition (albeit with somewhat worse parameters).}

The transformation we perform is as follows:
\begin{enumerate}
  \item We first condition on the band $B=\{\bx: \dotp{\bx}{\bw} \in [\sigma_1, \sigma_2] \}$, for some
    positive parameters
    $\sigma_1, \sigma_2$.
  \item We then perform the perspective projection on the samples, $\pi_{\vec w}(\cdot)$,
    defined in Equation~\eqref{eq:perspective_projection}.
\end{enumerate}

To facilitate the proceeding formal description, we introduce the following definition.
\begin{definition} [Transformed Distribution] \label{def:transformed_distribution}
Let $\D$ be a distribution on $\R^d \times \{\pm 1\}$, $B \subseteq \R^d$ and $(\bx, y) \sim \D$.
\begin{itemize}
\item We use $\D_B$ to denote $\D$ conditioned on $\bx$ being in the set $B$.
\item Let $q: \R^d \mapsto \R^d$. We denote by $\D^q$ the distribution of the random variable $(q(\bx), y)$.
\end{itemize}
With the above notation, $\D_B^q$ is the distribution obtained by first conditioning on $B$ and
then applying the transformation $q(\cdot)$ to $\D_B$.
\end{definition}

With Definition~\ref{def:transformed_distribution} in place,
the distribution obtained from $\D$ after we condition on the band $B$ is $\D_B$,
and the distribution obtained from $\D_B$ after we perform the perspective projection is $\D_B^{\pi_{\bw}}$.
We can now state the main proposition of this subsection.

\begin{proposition}[Properties of $\D_B^{\pi_{\bw}}$] \label{prop:reduction}
Let $\D$ be a $(3, L, R, \beta)$-well-behaved isotropic distribution on $\R^{d} \times \{\pm 1\}$
that satisfies the $(\tsyb,\tsya)$-Tsybakov noise condition with respect to an unknown
halfspace $f(\bx) = \sgn(\dotp{\vec w^{\ast}}{\bx})$. Fix any unit vector $\bw$ such that $\theta(\bw, \wstar) = \theta$,
and let $B=\{\bx: \dotp{\bx}{\bw} \in [\rho R/2, \rho R/\sqrt{2}] \}$, for some $\rho \in (0, 1]$.
Then, for some $c=(LR)^{O(1)}$, the following conditions hold:
\begin{enumerate}
\item The distribution $\Dperp$ on $\R^{d} \times \{\pm 1\}$ is $\left(2, c \rho^3, \frac{1}{\rho}, \frac{\beta}{c \rho}\log\frac 1 \rho\right)$-well-behaved.

\item The distribution $\Dperp$ satisfies the $\left(\tsyb,\frac A{c \rho}\right)$-Tsybakov noise condition with
optimal classifier $\sgn\left(\dotp{\wperp}{\bx} + 1/\tan\theta\right)$.
\end{enumerate}
\end{proposition}

The rest of this subsection is devoted to the proof of Proposition~\ref{prop:reduction}.
Before we proceed with the proof, we express the problem of finding a certifying vector
$\vec v$ satisfying \eqref{eq:certificate_form} in the transformed domain.
Indeed, it is not hard to see that after we condition on $B$ and perform the perspective projection $\pi_{\bw}$,
our goal is to find a vector $\vec v$ and scalars $t_1, t_2 > 0$ such that
\begin{align}\label{eq:certifying_reduction}
\E_{(\bz, y) \sim \Dperp}[ \1\{ -t_1 \leq \dotp{\bv}{\bz} \leq -t_2\} \, y] < 0 \,.
\end{align}
More formally, we have the following simple lemma showing that if we find a
certifying vector $\vec v$ and parameters $t_1, t_2$ in the transformed instance $\Dperp$
satisfying Equation~\eqref{eq:certifying_reduction}, the same vector and parameters will be a
certificate with respect to the initial well-behaved distribution $\D$.
The relevant expectation remains negative but is slightly closer to zero.
\begin{lemma}\label{lem:certificate_reduction}
Let $\D$ be a $(3, L, R, \beta)$-well-behaved distribution on $\R^d$ and let
$B = \{\bx: \dotp{\bx}{\vec w} \in [\rho R/2, \rho R/\sqrt{2}]\}$, for some $\rho \in (0,1]$.
Let $\vec w \in \R^d$ be a unit vector and let $\vec v \in \bw^{\perp}$, $t_1, t_2>0$
be such that $\E_{(\bz, y) \sim \Dperp}[ \1\{ -t_1 \leq \dotp{\bv}{\bz} \leq -t_2\} \, y] < - C$,
for some $C > 0$.  Then we have that
$\E_{(\bx, y) \sim \D}[T_{\vec w}(\bx) \, y \dotp{\vec w}{\bx}] = - \Omega(C L R^3 \rho).$
\end{lemma}
\begin{proof}
It holds
\begin{align*}
\E_{(\bz, y) \sim \Dperp}[ \1\{ -t_1 \leq \dotp{\bv}{\bz} \leq -t_2\} y]
&= \E_{(\bx, y) \sim \D_B}[ \1\{ -t_1 \leq \dotp{\bv}{\pi_{\bw}(\bx)} \leq -t_2\} y] \\
&= \frac{1}{\pr_{\D}[B]} {\E_{(\bx, y) \sim \D}[ T_{\vec w}(\bx) \dotp{\vec w}{\bx} y]} \,.
\end{align*}
Using the anti-anti concentration property of $\D_{\bx}$, we can bound
$\pr_{\D}[B]$ from below. Observe that since the lower bound $L$ on the $3$-dimensional marginal density
holds inside a ball of radius $R$, to bound the above probability from below,
we can multiply $L$ by the volume of the intersection of $B$ with the ball of radius $R$.
Using the formula for the volume of spherical segments, we obtain
$\pr_{\D}[B] = \Omega(L R^3 \rho)$.
This completes the proof.
\end{proof}

\paragraph{Proof of Proposition~\ref{prop:reduction}.}
Our goal is to compute a certificate of the form~\eqref{eq:certificate_form}.
As we already discussed, if we had chosen to simply project the points on the subspace $\vec w^{\perp}$,
we would have obtained an instance that is not linearly separable --- even if the noise rate
$\eta(\bx)$ was identically zero.  By first conditioning on the set $B = \{\bx: \dotp{\bx}{\vec w} \in [\sigma_1, \sigma_2]\}$,
where $\sigma_1, \sigma_2 > 0$, and then performing the perspective projection $\pi_{\vec w}$,
we keep the dataset linearly separable (with respect to the noiseless distribution, i.e., for $\eta(\bx) = 0$),
albeit by a \emph{biased} linear classifier.

We have the following lemma.
\begin{lemma} \label{lem:biased_halfspace}
Let $\D$ be a distribution on $\R^{d} \times \{\pm 1\}$ such that for $(\bx, y) \sim \D$ we have that $y = \sgn(\dotp{\wstar}{\bx})$.
Let $\vec w$ be any unit vector such that $\theta(\vec w, \wstar) = \theta \in (0, \pi]$.
For $(\bz, y) \sim \Dperp$ it holds
$y = \sgn\left(\dotp{\wperp}{\vec z} + \frac{1}{\tan \theta} \right)$,
i.e., the transformed distribution is linearly separable by a biased hyperplane.
\end{lemma}
\begin{proof}
Observe that $\wstar = \lambda_1 \wperp + \lambda_2 \bw$, where $\lambda_1 \nnew{>} 0$.
We then have
\begin{align*}
\sgn(\dotp{\wstar}{\bx})
&= \sgn\left(\lambda_1 \dotp{\wperp}{\bx} + \lambda_2 \dotp{\bw}{\bx} \right)
= \sgn\left(\lambda_1 \dotp{\bw}{\bx} \left(\frac{\dotp{\wperp}{\bx}}{\dotp{\bw}{\bx}} +  \frac{\lambda_2}{\lambda_1} \right) \right) \\
&= \sgn\left(\dotp{\wperp}{ \pi_{\vec w}(\bx)} +  \frac{\lambda_2}{\lambda_1} \right) \,,
\end{align*}
where to get the last equality we use the fact that $\lambda_1$ and
$\dotp{\bw}{\bx}$ are both positive given that we conditioned on the band $B$.
Observe that if the angle between $\bw$ and $\wstar$ is $\theta$, then $\lambda_1 = \sin \theta$ and $\lambda_2 = \cos \theta$.
This completes the proof.
\end{proof}

We next show that conditioning on the band $B$ will not make the Tsybakov
noise condition substantially worse.

\begin{lemma}\label{lem:tsybakov_condition}
Let $\D$ be a $(3, L, R, \beta)$-well-behaved isotropic distribution on $\R^{d} \times \{\pm 1\}$
that satisfies the $(\tsyb,\tsya)$-Tsybakov noise condition with respect to an unknown halfspace $f(\bx) = \sgn(\dotp{\vec w^{\ast}}{\bx} )$.
Let $B = \{\bx: \dotp{\bx}{\vec w} \in [\rho R/2, \rho R/\sqrt{2}]\}$, for some $\rho \in(0,1]$.
Then $\D_B$ satisfies the Tsybakov noise condition with parameters $(\tsyb,O(A/(R^3 L \rho)))$ and optimal linear classifier $\wstar$.
\end{lemma}
\begin{proof}
We have that $\pr_{\bx   \sim \D_\bx}[1-2\eta(\bx) > t | \bx \in B] \leq \pr_{\bx \sim \D_{\bx}}[1-2\eta(\bx) > t]/\pr_{\bx \sim \D_\bx}[B]$.
From the proof of Lemma~\ref{lem:certificate_reduction}, we have seen that we can use the anti-anti-concentration property of $\D_{\bx}$
to bound $\pr_{\bx \sim \D_\bx}[B]$ from below.
Specifically, we have $\pr_{\bx \sim \D_\bx}[B] \geq \Omega(L R^3 \rho)$.
Therefore, $\D_B$ satisfies the Tsybakov noise condition with parameters $(\tsyb,O(A/(R^3\rho L))$.
\end{proof}

Finally, we show that the transformation of Equation~\eqref{eq:perspective_projection}
also preserves the anti-anti-concentration and concentration properties
of the marginal distribution $\D_{\bx}$.

\begin{lemma}\label{lem:well_beh_projection}
Let $\D$ be a $(3, L, R, \beta)$-well-behaved distribution.
Fix any unit vector $\bw$ and let $B=\{\bx: \dotp{\bx}{\bw} \in [\rho R/2, \rho R/\sqrt{2}] \}$, for some $\rho \in (0, 1]$.
Then the transformed distribution $\Dperp$ is
$\left(2,  \Omega(L \rho^3 R^3), 1/\rho, O(\beta/(R \rho) \log(1/(L R \rho)))\right)$-well-behaved.
\end{lemma}

\begin{proof}
Let $\gamma(\bx):\R^d \mapsto \R_+$ be the probability density function of $\D_{\bx}$ and
$B = \{\bx : \rho R/2 \leq \dotp{\bw}{\bx} \leq \rho R/\sqrt{2} \}$.
\nnew{Note that the conditional distribution $(\D_{\bx})_B$ of the random vector $\bx \sim \D_{\bx}$ on
the band $B$ has density $\gamma_B(\bx) = \1_{B}(\bx) \gamma(\bx)/(\int_B \gamma(\bx) \d \bx)$.}
Since the transformation $ \pi_{\vec w}(\cdot)$ is not injective, we consider the
transformation $\phi(\bx) = ( \dotp{\bw}{\bx},  \pi_{\vec w}(\bx))$
and observe that $\phi(\bx) : \R^{d} \mapsto \R^{d}$ is injective.
Denote by \nnew{$\vec U$} the random variable corresponding to the image of $\bx$, \nnew{$\bx \sim(\D_{\bx})_B$}, under $\phi$.
Without loss of generality, we may assume that $\bw = \vec e_1$.
By computing the Jacobian of the above one-to-one transformation. we get that the
density function of the random vector $\vec U$ is given by
$\gamma_{\vec U}(\vec u) = |\vec u_1|^{d-1}  \gamma_B(\vec u_1 (1,  \vec u_2, \ldots, \vec u_d))$.
We can marginalize out the ``dummy" variable $\vec u_1$ to obtain
the density function $g$ of \nnew{$\vec z \sim (\D_{\bx})_B^{\pi_{\vec w}}$},
i.e.,
$$g(\vec z) = \int_{-\infty}^\infty |\vec u_1|^{d-1} \gamma_B(\vec u_1 (1, \vec z)) \d \vec u_1\,.
$$
Let $V$ be any $2$-dimensional subspace of $\vec w^{\perp}$.
Without loss of generality, we may assume that $V = \mathrm{span}(\vec e_2, \vec e_3)$.
Denote $\nnew{{\vec z}_{[3, d-1]}} = (\vec z_3, \ldots, \vec z_{d-1})$, $U = \mathrm{span}(\vec e_1, \vec e_2, \vec e_3)$,
and $U^{\perp} = \mathrm{span}(\vec e_4, \ldots, \vec e_d)$.
The marginal density of \nnew{$\vec z \sim(\D_{\bx})_B^{\pi_{\vec w}}$} on $V$ is then given by
\begin{align*}
g_V(\vec z_1,\vec z_2)
&= \int_{U^{\perp}} \int_{-\infty}^\infty |\vec u_1|^{d-1} \gamma_B(\vec u_1  (1, \vec z)) \d \vec u_1\ \d {\vec z}_{[3, d-1]} \\
&= \int_{-\infty}^\infty |\vec u_1|^{d-1} \int_{U^{\perp}} \gamma_B(\vec u_1 (1, \vec z)) \d {\vec z}_{[3, d-1]} \ \d \vec u_1 \\
&= \frac{1}{\int_B \gamma(\bx) \d \bx} \int_{\rho  R/2}^{\rho  R/\sqrt{2}} |\vec u_1|^{d-1} \int_{U^{\perp}} \gamma(\vec u_1  (1, \vec z)) \d {\vec z}_{[3, d-1]} \d \vec u_1 \\
&= \frac{1}{\int_B \gamma(\bx) \d \bx} \int_{\rho  R/2}^{\rho  R/\sqrt{2} } |\vec u_1|^{2} \gamma_U(\vec u_1  (1, \vec z_1, \vec z_2)) \d \vec u_1\,,
\end{align*}
where to get the third equality we used the definition of the conditional density on $B$ and the fact that
the set $B$ only depends on the first coordinate. The last equality follows by a change of variables.
Since $\D_{\bx}$ is $(3,L,R,\beta)$-well-behaved, we have that
if $\vec u_1^2 (1 + \vec z_2^2 + \vec z_3^2) \leq R^2$ we have that
$\gamma_U(\vec u_1(1, \vec z_1, \vec z_2) ) \geq L$.
Therefore, using the fact that $\vec u_1^2 \leq  \rho^2  R^2/2$, we obtain that
for $\vec z_1^2 + \vec z_2^2 \leq 2/\rho^2-1 $ it holds
$\gamma_U(\vec u_1 (1, \vec z_1, \vec z_2)) \geq L$.
Observe that since $\rho \leq 1$, we can get the slightly looser bound
$\vec z_2^2 + \vec z_3^2 \leq 1/\rho^2$.
Note that $\int_B \gamma(\bx) \d \bx \leq 1 $ and also
$\int_{\rho  R/2}^{\rho  R/\sqrt{2}} |\vec u_1|^{2} \d \vec u_1 = \Omega(\rho^3 R^3)$.
Combining these bounds, we obtain that $g_V(\vec z_1, \vec z_2) \geq \Omega(L \rho^3 R^3)$.

It remains to prove that the transformed distribution still has exponentially decaying tails.
In the proof of Lemma~\ref{lem:tsybakov_condition}, we have already argued that
the probability mass of $B$ is bounded below by $C_B = \Omega(L R^3 \rho)$.
Therefore, the distribution \nnew{$(\D_{\bx})_B$} obtained after conditioning
has exponential concentration with parameter $\beta ( 1 - \log C_B)$.
After we perform the perspective projection (Equation~\eqref{eq:perspective_projection})
to obtain \nnew{$(\D_{\bx})_B^{\pi_{\vec w}}$}, the concentration parameter
becomes $2 \beta  (1 - \log C_B)/(\rho R)$, since we divide each coordinate of $\bx$ by
a quantity that is bounded from below by $R \rho/2$.
This completes the proof of Lemma~\ref{lem:well_beh_projection}.
\end{proof}

Proposition~\ref{prop:reduction} follows by
combining Lemmas~\ref{lem:biased_halfspace},~\ref{lem:tsybakov_condition},~\ref{lem:well_beh_projection}.

\subsection{Efficient Certificate Computation Given Initialization} \label{ssec:impr-cert}

In this subsection, we give our main algorithm for computing a non-optimality certificate \nnew{in
the transformed instance}, i.e., a vector $\vec v$ and parameters \nnew{$t_1, t_2 >0$} satisfying
Equation~\eqref{eq:certifying_reduction}. Recall that after the perspective projection
transformation of Subsection~\ref{ssec:reduction}, we now have sample access to i.i.d. labeled
examples $(\bx, y)$ from a well-behaved distribution $\D$ on $\R^{d} \times \{ \pm 1 \}$ satisfying
the Tsybakov noise condition (albeit with somewhat worse parameters) with the optimal classifier
being a non-homogeneous halfspace (see Proposition~\ref{prop:reduction}.)

Our certificate algorithm in this subsection assumes the existence of an initialization vector,
i.e., a vector that has non-trivial correlation with $\wperp$. The simplest way to find such a
vector is by picking a uniformly random unit vector. A random initialization suffices for the guarantees of this
subsection (and in particular for Theorem~\ref{thm:cert-wb}). We note that for the family of log-concave distributions, we
can leverage additional structure to design a fairly sophisticated initialization algorithm
that in turn leads to a faster certificate algorithm (see Section~\ref{sec:lc}).

The main algorithmic result of this section is an efficient algorithm to compute a certifying vector
satisfying Equation~\eqref{eq:certifying_reduction}.  Note that we are essentially working in
$(d-1)$ dimensions, since we have already projected the examples to the subspace $\vec w^{\perp}$.
As shown in Proposition~\ref{prop:reduction}, the transformed distribution $\D_{B}^{\pi_{\bw}}$ is
still well-behaved and follows the Tsybakov noise condition, but with somewhat worse parameters than
the initial distribution $\D$.

To avoid clutter in the relevant expressions, we overload the notation and use $\D$ instead
of $\D_{B}^{\pi_{\bw}}$ in the rest of this section. Moreover, we use the notation
$(L, R, \beta)$ and $(\tsyb, \tsya)$ to denote the well-behaved distribution's parameters
and the Tsybakov noise parameters.
The actual parameters of $\D_{B}^{\pi_{\bw}}$ (quantified in Proposition~\ref{prop:reduction})
are used in the proof of Theorem~\ref{thm:cert-wb}. To simplify notation, we will henceforth denote by $\vec v^{\ast}$ the vector $\wperp$. We show:

\begin{proposition}\label{prop:cert-wb}
Let $\D$ be a $(2, L, R, \beta)$-well-behaved distribution on $\R^{d} \times \{\pm 1\}$ satisfying
the $(\tsyb,\tsya)$-Tsybakov noise condition with respect to an unknown halfspace
$f(\bx) = \sgn(\dotp{\vec v^{\ast}}{\bx} + b)$. Let $\vec v_0 \in \R^{d}$ be a unit vector
such that $\dotp{\vec v_0}{\vec v^{\ast}}\geq 4b/R$. There is an algorithm
(Algorithm~\ref{alg:find_certificate}) with the following performance guarantee: Given $\vec v_0$
and $N = d \, \frac{\conb^2R^2}{b^2}\left(\frac{A}{RL}\right)^{O(1/\tsyb)}\log(1/\delta)$
samples from $\D$, the algorithm runs in $\poly(N, d)$ time, and with probability at least
$1-\delta$ returns a unit vector $\vec v \in \R^{d}$ and a scalar $t \in \R_+$ such that
$$
\E_{(\vec x,y) \sim \D} \left[ \1[- R \leq \dotp{\vec v}{\vec x} \leq - t] \, y \right] \leq
- \frac{b}{R \beta}\left(\frac{RL}{A}\right)^{O(1/\alpha)}\;.
$$
\end{proposition}

Algorithm~\ref{alg:find_certificate} employs a ``perceptron-like" update rule that in polynomially
many rounds succeeds in improving the angle between the initial guess $\vec v_0$ and the target
vector $\wperp = \vec v^{\ast}$.  While the algorithm is relatively simple, its proof of correctness
relies on a novel structural result (Lemma~\ref{lem:improving_w_perp}) whose proof is the main
technical contribution of this section. Roughly speaking, our structural result establishes the following win-win statement: Given a vector
whose correlation with $\vec v^{\ast}$ is non-trivial, either this vector is already a certifying
vector (see Item 1 of Lemma~\ref{lem:improving_w_perp} and Lemma~\ref{lem:certificate_reduction}) or
the update step will improve the angle with $\vec v^{\ast}$ (Item 2 of Lemma~\ref{lem:improving_w_perp}).

In more detail, starting with a vector $\vec v_0$ that has non-trivial correlation with $\vec v^{\ast}$,
we consider the following update rule
\begin{equation}
  \vec v^{(t+1)} = \vec v^{(t)} + \lambda\vec { g} \,,
\end{equation}
where $\lambda > 0$ is an appropriately chosen step size and
$$\vec g=\E_{(\vec x,y) \sim \D} [\1\{-R \leq \langle \vec v^{(t)}, \vec x \rangle \leq - R / 2\} \, y \, \proj_{(\vec  v^{(t)})^{\perp}}(\bx )] \;,$$
where $\proj_{(\vec  v^{(t)})^{\perp}}(\bx )$ is the projection of $\bx$ to the subspace $(\vec v^{(t)})^\perp $.
In Lemma~\ref{lem:gradient}, we show that if $\vec v^{(t)}$ is not a certifying vector, i.e., it does not satisfy Item 1 of Lemma~\ref{lem:gradient}, then there exists an appropriately small step size $\lambda$
that improves the correlation with $\vec v^{\ast}$ after the update. This is guaranteed by Item 2 of Lemma~\ref{lem:gradient},
which shows that $\vec g$ has positive correlation with $(\vec v^{\ast})^{\perp_{\vec v}}$
(the normalized projection of $\vec v^{\ast}$ onto $\bv^\perp$),
and thus will turn $\vec v^{(t)}$ towards the direction of $\vec v^{\ast}$ decreasing the angle between them.

\begin{algorithm}[H] \caption{Computing a Certificate Given Initialization}
  \label{alg:find_certificate}
  \begin{algorithmic}[1] \Procedure{ComputeCertificate}{$(L, R, \beta),(\tsya, \tsyb), \delta, \vec v_0, \widehat{\D}$}\\
    \textbf{Input:} Empirical distribution $\widehat{\D}$ of a $(2, L, R,
    \beta)$-well-behaved distribution that satisfies the $(\tsyb,\tsya)$-Tsybakov noise
    condition, initialization vector $\vec v_0$, confidence probability $\delta$. \\
    \textbf{Output:} A certifying vector $\vec v$ and positive scalars $t_1, t_2$ that
    satisfy~\eqref{eq:certifying_reduction}. \State ${\vec v}^{(0)} \gets \vec v_0$
    \State $T  \gets  \poly(1/L, 1/R, A)^{1/\alpha} \cdot \poly(1/b, 1/\beta)$ \State $\lambda \gets \frac 1 {\beta^3} \poly(L, R, 1/A)^{1/\alpha}$;
    $c \gets  \frac {b} {R \beta} \poly(L, R, 1/A)^{1/\alpha}$
    \State \textbf{for} $t = 1,  \dots, T$ \textbf{do}
    \State \qquad $B^{t'} = \{\bx :  - R \leq \dotp{\vec v^{(t-1)}}{\vec x} \leq -t' \}$
    \State \qquad \textbf{if} there exists $t_0\in(R/2,R]$ such that $\E_{(\vec
    x,y) \sim \widehat{\D}} \left[\1_{B^{t_0}}(\bx) \,  y \right] \leq- c $\label{alg1:line_check_certificate}
    \State \qquad \qquad $  \textbf{return}
    (\vec v^{(t-1)}, R, t_0)$ \label{alg1:return}\State \qquad    $\sample{{\vec {\hat  g}}}{t}\gets\E_{(\vec x,y) \sim \widehat{\D}}
    \left[\1_{B^{R/2}}(\bx) \, y \, \proj_{(\vec  v^{(t-1)})^{\perp}}(\bx )\right] $
    \State\qquad  ${\vec v}^{(t)} \gets \frac{{\vec v}^{(t-1)} +\lambda  \sample{{\vec {\hat  g}}}{t}}{\snorm{2}{ {\vec v}^{(t-1)} +\lambda
    \sample{{\vec {\hat  g}}}{t}}} $\label{alg:OPGDestep} \EndProcedure
  \end{algorithmic}
\end{algorithm}

\begin{figure}[h!]
\subfloat[The regions $B_1^{t}, B_2^{t}, B_3^{t}$ used
    in the definition of $I_1^t$ in the proof of Lemma~\ref{lem:improving_w_perp}.]{
  \begin{minipage}[t]{0.50\textwidth}
    \centering
    \begin{tikzpicture}[scale=0.8]
      \coordinate (start) at (2.9/5+0.5,0);
      \coordinate (center) at (2.9/5,0);
      \coordinate (end) at (2.9/5+0.5,0.5);

      \draw[black,dashed, thick](-3.75,-2) -- (3.75,-2);
      \draw[black,dashed, thick](-3.75,-1.4) --(3.75,-1.4);
      \draw[fill=red, opacity=0.3,draw=none]  (-1.1 ,-1.4)--(-3.75,-1.4)--(-3.75,-2)--(-1.75225,-2);
      \draw[fill=blue, opacity=0.5,draw=none] (0,-1.4) rectangle (3.75,-2);
      \draw[fill=blue,opacity=0.5,draw=none] (-1.1,-1.4) -- (0 ,-1.4)--(0,-2)--(-1.75225,-2);
      \draw[->] (-3.8,0) -- (3.8,0) node[anchor=north west,black] {};
      \draw[->] (0,-2.5) -- (0,2) node[anchor=south east] {};
      \draw[thick,->] (0,0) -- (-0.7,0.7) node[anchor= south east,below,left=0.1mm] {$\vec v^{\ast}$};
      \draw[black,thick] (-2,-2.22) -- (2.5,4.44/5 * 2.5 -2.22/5);
      \draw[thick ,->] (0,0) -- (0,1) node[right=2mm,below] {$\vec v$};
      \draw[thick ,->] (0,0) -- (-1,0) node[right=2mm,below] {$(\vec v^{\ast})^{\bot_{\vec v}}$};
      \draw (0,-1.4) node[] {$\scriptscriptstyle -t$};
      \draw (0,-2.1) node[] {$\scriptscriptstyle-R$};
      \draw (-3,-1.7)node[] {$B_1^{t}$};
      \fill (0,-1.4) circle [radius=1pt];
      \fill (0,-2) circle [radius=1pt];
      \draw(-0.5,-1.7) node[] {$B_2^{t}$};
      \draw (1,-1.7) node[] {$B_3^{t}$};
      \draw (2.22/5,0) node[below]{$\scriptstyle b$};
      \fill (2.5/5,0) circle [radius=1pt];
      \pic [draw, <->, angle radius=5mm, angle
      eccentricity=1.2, "$\scriptstyle\theta$"] {angle = start--center--end};
  \end{tikzpicture}
  \label{fig:Tsybakov_Regions_Improv_Cert}
\end{minipage}}\quad
\subfloat[ The regions $B_1, B_2, B_3$ defined in the definition of $I_2$ in the proof of Lemma~\ref{lem:improving_w_perp}. ]
{\begin{minipage}[t]{0.50\textwidth}
  \centering
  \begin{tikzpicture}[scale=0.8]
    \coordinate (start) at (2.9/5+0.5,0);
    \coordinate (center) at (2.9/5,0);
    \coordinate (end) at (2.9/5+0.5,0.5);

    \draw[black,dashed, thick](-3.75,-2) -- (3.75,-2); \draw[black,dashed, thick](-3.75,-1) --(3.75,-1);
    \draw[fill=red, opacity=0.3,draw=none]  (-0.626126 ,-1)--(-3.75,-1)--(-3.75,-2)--(-1.75225,-2);
    \draw[fill=red, opacity=0.3,draw=none] (0,-1) rectangle (3.75,-2);
    \draw[fill=blue,opacity=0.5,draw=none] (-0.626126 ,-1) -- (0 ,-1)--(0,-2)--(-1.75225,-2);
    \draw[->] (-3.8,0) -- (3.8,0) node[anchor=north west,black] {};
    \draw[->] (0,-2.5) -- (0,2) node[anchor=south east] {};
    \draw[thick,->] (0,0) -- (-0.7,0.7) node[anchor= south east,below,left=0.1mm] {$\vec v^{\ast}$};
    \draw[black,thick] (-2,-2.22) -- (2.5,4.44/5 * 2.5 -2.22/5);
    \draw[thick ,->] (0,0) -- (0,1) node[right=2mm,below] {$\vec v$};
    \draw[thick ,->] (0,0) -- (-1,0) node[right=2mm,below] {$(\vec v^{\ast})^{\bot_{\vec v}}$};
    \draw (0,-1.1) node[] {$\scriptscriptstyle -\frac R2$};
    \draw (0,-2.1) node[] {$\scriptscriptstyle-R$};
    \draw (-3,-1.5)node[] {$B_1$};
    \fill (0,-1) circle [radius=1pt];
    \fill (0,-2) circle [radius=1pt];
    \draw(-0.5,-1.5) node[] {$B_2$};
    \draw (1,-1.5) node[] {$B_3$};
    \draw (2.22/5,0) node[below]{$\scriptstyle b$};
    \fill (2.5/5,0) circle [radius=1pt];
    \pic [draw, <->, angle radius=5mm, angle
    eccentricity=1.2, "$\scriptstyle\theta$"] {angle = start--center--end};
\end{tikzpicture}
\label{fig:Tsybakov_Regions_Improv}
\end{minipage}
}
\caption{\nnew{
    In the subspace $\bw^\perp$, the certifying function is simply an indicator $\1\{-R \leq \dotp{\bv}{\bx} \leq -t_0\}$,
    for some $t_0 >0$. See also Equation~\eqref{eq:certifying_reduction}. This is shown in
    Figure~\ref{fig:Tsybakov_Regions_Improv_Cert}. The blue regions in
    Figure~\ref{fig:Tsybakov_Regions_Improv_Cert} (resp.
    Figure~\ref{fig:Tsybakov_Regions_Improv}) have negative contribution to the
    value of $I_1^t$ (resp. $I_2$), while the red regions have positive
    contribution.
  }
}
 \end{figure}
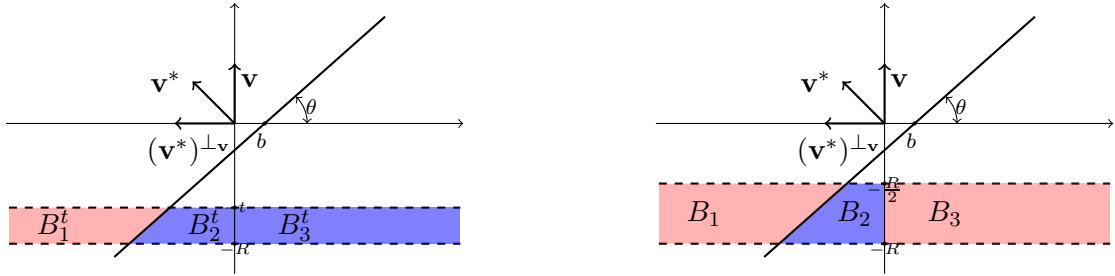

We are now ready to state and prove our win-win structural result:

\begin{lemma}[Win-Win Result] \label{lem:improving_w_perp}
Let $\D$ be a $(2, L, R, \beta)$-well-behaved distribution on $\R^{d} \times \{\pm 1\}$ that satisfies the $(\tsyb,\tsya)$-Tsybakov noise condition with respect to $f(\bx) = \sgn(\dotp{\vec v^{\ast}}{\bx} + b)$, and $\vec v \in \R^{d}$ be a unit vector with
$\dotp{\vec v}{\vec v^{\ast}} \geq 4 b/R$. Consider the band $B^t = \{\bx :  - R \leq \dotp{\vec v}{\vec x} \leq -t \}$
for $t \in [R/2,R]$ and define \nnew{$\vec g = \E_{(\vec x,y) \sim \D} [\1_{B^{R/2}}(\bx) \, y \, \proj_{\bv^\perp}(\bx) ] \,.$}
For some $c = (RL/A)^{O(1/\alpha)}$, one of the following statements is satisfied:
\begin{enumerate}
\item There exists $t_0\in(R/2,R]$, such that $\E_{(\vec x,y) \sim \D} \left[\1_{B^{t_0}}(\bx) \,  y \right] \leq -c^2 \frac{b}{R \beta}$.

\item It holds $\dotp{\vec g}{\vec v^{\ast}} \geq c^2 \frac{\pi b}{4 \beta}$.
\end{enumerate}
Moreover, the first condition always holds if $\theta(\vec v,\vec v^{\ast}) \leq b \, c/\beta$.
\end{lemma}

\begin{proof}
Since $\vec v$ and $\vec v^{\ast}$ span a $2$-dimensional subspace,
we can assume without loss of generality that
$\vec v = \vec e_2$ and $\vec v^{\ast} = (- \sin \theta, \cos \theta)$.
Our analysis will consider the following regions:
$B_1^t = \{ \bx \in B^t : f(\bx) = +1 \}$,
$B_2^t = \left\{ \bx  \in B^t: f(\bx) = -1 \text{ and } \dotp{\proj_{\bv^\perp}\bx}{\vec v^{\ast}} \geq 0 \right \}$,
and $B_3^t = \left\{ \bx \in B^t: f(\bx) = -1 \text{ and } \dotp{\proj_{\vec v^\perp} \bx}{\vec v^{\ast}} <0 \right\}.$
See Figures~\ref{fig:Tsybakov_Regions_Improv_Cert}, \ref{fig:Tsybakov_Regions_Improv} for an illustration.

For notation convenience, we will also denote
$(\vec v^{\ast})^{\perp_{\vec v}} = \proj_{\vec v^\perp}(\bv^{\ast})/\snorm{2}{\proj_{\vec v^\perp}(\bv^{\ast}) } $
and $\nr(\bx)=1-2\eta(\bx)$.

Given the above notation, we can rewrite the two quantities appearing in Items 1, 2 of Lemma~\ref{lem:improving_w_perp}
as follows:
\begin{align}  \label{eq:regions_assumption}
I_1^t &= \E_{(\vec x,y) \sim \D} \left[ \1_{B^t}(\bx)  y \right] =
\underbrace{\E_{\vec x \sim \D_\bx} \left[ ( \1_{B_1^t}(\bx) - \1_{B_2^t}(\bx)) \nr(\bx) \right]}_{I_{1,1}^t} -
\underbrace{\E_{\vec x \sim \D_\bx} \left[\1_{B_3^t}(\bx)   \nr(\bx) \right]}_{I_{1,2}^t}\;, \\
I_2 &= \dotp{\vec g}{ (\vec v^{\ast})^{\perp_{\vec v}} } =
\dotp{ \E_{(\vec x,y) \sim \D} [\1_{B^{R/2}}(\bx) y \bx  ] } {(\vec v^{\ast})^{\perp_{\vec v}}} \nonumber \\
&= \underbrace{\E_{\vec x \sim \D_\bx} \left[(\1_{B_1^{R/2}}(\bx) - \1_{B_2^{R/2}}(\bx)  )  \nr(\bx)  |\bx_1| \right]}_{I_{2,1}}
+ \underbrace{\E_{\vec x \sim \D_\bx} \left[ \1_{B_3^{R/2}}(\bx) \nr(\bx)  |\bx_1| \right]}_{I_{2,2}}\;.
 \end{align}
Since $\vec v^{\ast} = (-\sin \theta, \cos \theta)$, the quantity $\dotp{\vec g}{\bv^{\ast}}$ (that appears in Item 2 of
Lemma~\ref{lem:improving_w_perp}) is equal to $\sin(\theta) I_2$.
We work with the normalized $ (\vec v^{\ast})^{\perp_{\vec v}}$ in order to simplify notation.

\medskip

Before we go into the details of the proof, we give a high-level
description of the main steps with pointers to the relevant claims.
Note that the quantity $I_1^t$ corresponds to the value of the certifying function
(in the subspace $\bw^\perp$) when we use $\bv$ as certifying vector
and $t_1 = -R, t_2 = t$ as thresholds. See Equation~\eqref{eq:certifying_reduction}.
When $I_1^t$ is small (see Item 1 of the lemma), we have a certifying function.
On the other hand, $\sin(\theta) I_2$ corresponds to the inner product of the update $\vec g$
and the optimal vector $\vec v^{\ast}$. Item 2 of the lemma states that this quantity is large,
which means that if we update according to $\vec g$ we shall improve the correlation with $\vec v^{\ast}$.

\paragraph{Heuristic Argument.}
\nnew{Since the formal proof is somewhat technical, we start with a useful (but inaccurate) heuristic argument.}
If we ignore the presence of $|\bx_1|$ in $I_{2,1}$ and $I_{2,2}$,
we see from Figure~\ref{fig:Tsybakov_Regions_Improv_Cert} that if the contribution
of region $B_2^{R/2}$ is sufficiently large compared to the positive contribution of $B_1^{R/2}$
(red region in Figure~\ref{fig:Tsybakov_Regions_Improv_Cert}),
then $I_1$ will be negative in total. \nnew{That is,} Item 1 is true.
On the other hand, if the contribution of $B_2^{R/2}$ is not very large,
then when we add the contribution of $B_3$ (red region in Figure~\ref{fig:Tsybakov_Regions_Improv}) overall,
$I_2$ will be positive and Item 2 now holds. Notice that in \new{this setting}
we could take the threshold $t$ in the definition of $I_1^t$ to simply be $R/2$,
i.e., use the entire band in our certificate.

Unfortunately, in the actual proof, we need to deal with the term $|\bx_1|$ in
the expectations of $I_2$ that makes the previous argument invalid.  Using the
Mean Value Theorem (Fact~\ref{fct:mvt}), we show that there exists a threshold
$t \in [-R, -R/2]$ that makes $I_1^t$ sufficiently negative. This is done in
Claim~\ref{clm:third}.

\medskip

\nnew{We can now proceed with the formal proof.}
We will require several technical claims.
First, we bound $I_{1,2}^{R/2}$ and $I_{2,2}$ from below using the fact that our distribution is well-behaved.
We require the following claim in order to show that the expressions in Item 1 (resp. Item 2)
of our lemma are not simply negative (resp. positive), but have a
non-trivial gap instead.  The proof of the claim relies on two important observations.
First, the fact that the distribution is well-behaved means that the contribution of
region $B_3$ would be sufficiently large if we ignore the noise function $\zeta(\bx)$ in the expectations.
Second, we use the fact that the Tsybakov noise rate $\zeta(\bx) = 1 - 2\eta(\bx)$ cannot
reduce the contribution of a region by a lot.

\begin{claim}\label{clm:first}
We have that $I_{1,2}^{R/2}$ and $I_{2,2}$ are bounded from below by some $c=(RL/A)^{O(1/\alpha)}$.
 \end{claim}

The proof of Claim~\ref{clm:first} can be found in Appendix~\ref{ap:cert}.

Now we show that if the angle between the optimal vector and the current one
is small, then $I_1^{R/2}$ is negative.  In particular, the first condition
always holds if $\theta(\vec v, \vec v^{\ast}) \leq b c /(4\beta )$.

\begin{claim}\label{clm:sec}
If $\theta(\vec v,\vec v^{\ast}) \leq b c /(4\beta )$, then $I_1^{R/2} \leq -c/4$.
\end{claim}

The proof of Claim~\ref{clm:sec} can be found in Appendix~\ref{ap:cert}.

Our next claim shows that when Item 2 does not hold, then Item 1
always does.  Having proved Claim~\ref{clm:sec}, we may also assume that
$\theta(\vec v,\vec v^{\ast}) \geq b c /(4\beta )$.
Observe that, in this case, if $I_2 \geq c/2$, we have
$$I_2\geq c/2=c \sin\theta/(2\sin\theta)\geq \pi c^2b/(4\beta\sin\theta) \,,$$
where we used the fact that $\sin(\theta) \geq 2\theta/\pi$ for all $\theta \in [0, \pi/2]$
and the fact that $\theta \geq b \, c/ (4\beta)$.
Therefore, to complete the proof, we need to show the following claim proving
that when $I_2 \leq c/2$, Item 1 of the lemma is always true.

\begin{claim}\label{clm:third}
If $\theta=\theta(\vec v,\vec v^{\ast}) \geq b c/(4\beta )$ and $I_2\leq c/2$,
there exists $t_0\in(-R,-R/2]$ such that $I_1^{t_0}\leq -bc^2 /(16 R\beta )$.
\end{claim}
\begin{proof}
Given the lower bounds on $I_{2,2}$ and $I_{1,2}^R$, we distinguish two
cases.  Assume that $I_2 \leq c/2$. This implies, from Claim~\ref{clm:first}, that
$I_{2,1} \leq - c/2$.  We show that in this case there exists a $t_0$ such that
$I_{1,1}^{t_0} \leq -bc^2 /(16 R\beta)$.
To show this, we are going to use the following variant of the standard Mean Value Theorem (MVT) for integrals.
\begin{fact}[Second Integral MVT] \label{fct:mvt}
Let $G : \R\mapsto \R_+$ be a non-negative, non-increasing, continuous function.
There exists $s \in (a, b]$ such that $\int_a^b G(t) F(t) \d t = G(a) \int_a^s F(t) \d t$.
\end{fact}
Let $\xi(\bx_2) = \bx_2/\tan \theta + b/\sin \theta$ be the first coordinate of a point
$(\bx_1, \bx_2)$ that lies on the halfspace defined by $f$, where
$f(\x)= \sgn(\dotp{\vec v^{\ast}}{\bx} + b)$ (see Figure~\ref{fig:Tsybakov_Regions_Improv}). We have
\begin{align*}
I_{1,1}^t &= \int_{-R}^{-t} \left( \int_{-\infty}^{\xi(\bx_2)}\nr(\bx_1, \bx_2) \gamma(\bx_1, \bx_2) \d \bx_1
                    -\int_{\xi(\bx_2)}^{0} \nr(\bx_1, \bx_2) \gamma(\bx_1, \bx_2) \d \bx_1 \right) \d \bx_2
                 = \int_{-R}^{-t} g(\bx_2) \d \bx_2 \,,
\end{align*}
where $g(\bx_2) = \int_{-\infty}^{\xi(\bx_2)} \nr(\bx_1, \bx_2) \gamma(\bx_1, \bx_2) \d \bx_1 -\int_{\xi(\bx_2)}^0 \nr(\bx_1, \bx_2) \gamma(\bx_1, \bx_2) \d \bx_1$. Moreover,
\begin{align*}
I_{2,1} &= \int_{-R}^{-R/2} \left( \int_{-\infty }^{\xi(\bx_2)} \nr(\bx_1, \bx_2) \gamma(\bx_1, \bx_2) |\bx_1| \d \bx_1
                -\int_{\xi(\bx_2)}^0 \nr(\bx_1, \bx_2) \gamma(\bx_1, \bx_2) |\bx_1| \d \bx_1 \right) \d \bx_2 \nonumber \\
           &\geq \int_{-R}^{-R/2} |\xi(\bx_2)|  g(\bx_2) \d \bx_2 = |\xi(-R)| \int_{-R}^{-t_0} g(\bx_2) \d \bx_2
           = |\xi(-R)| I^{t_0}_{1,1}, \end{align*}
for some $t_0 \in (-R, -R/2]$. Observe that the inequality above follows by replacing $|\bx_1|$ with its lower bound $|\xi(\bx_2)|$
in the first integral and by its upper bound $|\xi({\bx_2})|$ in the second.

We now observe that $|\xi(\bx_2)| = - \bx_2/\tan \theta - b/ \sin \theta$, where to remove the
absolute value we used the assumption that $\cos \theta \geq 4 b/R$. Therefore, $|\xi(\bx_2)|$ is a
decreasing and non-negative function of $\bx_2$. Using the Mean Value Theorem, Fact~\ref{fct:mvt}, we obtain
\begin{equation} \label{eq:mvt_bound}
I_{2,1} \geq \int_{-R}^{-R/2} |\xi(\bx_2)|  g(\bx_2) \d \bx_2 = |\xi(-R)| \int_{-R}^{-t_0} g(\bx_2) \d \bx_2 = |\xi(-R)| I^{t_0}_{1,1} \;.
\end{equation}
Thus,
$$I^{t_0}_{1,1} \leq I_{2,1}/|\xi(-R)|\leq -c \sin\theta /(2R) \leq -bc^2/ (16 R\beta )\;,$$
where we used that $\theta\geq  b c/(4\beta )$.
This completes the proof of Claim~\ref{clm:third}.
\end{proof}
Putting together the above claims, Lemma~\ref{lem:improving_w_perp} follows.
 \end{proof}

In the next lemma, we show that if Item 2 of Lemma~\ref{lem:improving_w_perp} is satisfied,
then an update step decreases the angle between the current vector $\vec v$ and the optimal
vector $\vec v^{\ast}$.

\begin{lemma}[Correlation Improvement]\label{lem:gradient}
For unit vectors $\vec v^{\ast}, \vec v \in \R^d$,
let $\vec {\hat g }\in \R^d$ such that
$\dotp{\vec {\hat g }}{\vec v^{\ast}} \geq \frac{c}{\beta }$, $\dotp{\vec {\hat g }}{\vec v} = 0$,
and  $\snorm{2}{\vec {\hat g }}\leq \beta$, with $c>0$ and $\beta\geq 1$.
Then, for $\vec v'=\frac{{\vec v}+\lambda{\vec{ \hat g}}}{\snorm{2}{{\vec v}+\lambda{\vec{\hat g}} }}$,
with $\lambda=\frac{c}{2\beta^3}$, we have that
$\dotp{\vec v'} {\vec v^{\ast}}\geq  \dotp{\vec v} {\vec v^{\ast}}+\lambda^2 \beta^2/2 $.
\end{lemma}
\begin{proof}
We will show that $\dotp{\vec v'}{\vec v^\ast}=\cos\theta'\geq\cos \theta+\lambda^2\beta^2$, where
$\cos\theta=\dotp{\vec v}{\vec v^{\ast}}$. We have that
\begin{equation}\label{eq:square_bound}
\snorm{2}{\vec v+\lambda\vec{ \hat g}}= \sqrt{1+\lambda^2\snorm{2}{\vec{ \hat g} }^2
+2\lambda \dotp{\vec{ \hat g} }{\vec{ v} } }\leq 1+\lambda^2\snorm{2}{\vec{ \hat g} }^2 \;,
\end{equation}
where we used that $\sqrt{1+a}\leq 1+a/2$.
Using the update rule, we have
\begin{align*}
\dotp{ \vec v'}{\vec v^{\ast}}&=  \dotp{ \vec v'}{(\vec v^{\ast})^{\perp_{\vec v}}}\sin\theta +  \dotp{ \vec v'}{\vec v}\cos\theta
= \frac{\lambda\dotp{\vec{ \hat g}}{ (\vec v^{\ast})^{\perp_{\vec v}} }}{\snorm{2}{\vec v+\lambda{\vec{\hat g}} }}\sin\theta +
\frac{\dotp{ {\vec v}+\lambda\vec{ \hat g}}{{\vec v}}}{\snorm{2}{{\vec v}+\lambda{\vec{\hat g}} }}\cos\theta \;.
\end{align*}
Now using Equation~\eqref{eq:square_bound}, we get
\begin{align*}
\dotp{\vec v'}{\vec v^{\ast}}
&\geq \frac{\lambda\dotp{\vec{ \hat g}}{(\vec v^{\ast})^{\perp_{\vec v}} }}{1+\lambda^2\snorm{2}{\vec{ \hat g} }^2}\sin\theta +
\frac{\cos\theta}{1+\lambda^2\snorm{2}{\vec{ \hat g} }^2}
= \cos\theta + \frac{\lambda\dotp{\vec{ \hat g}}{ (\vec v^{\ast})^{\perp_{\vec v}} }}{1+\lambda^2\snorm{2}{\vec{ \hat g} }^2}\sin\theta
+\frac{-\lambda^2\snorm{2}{\vec{ \hat g} }^2\cos\theta}{1+\lambda^2\snorm{2}{\vec{ \hat g} }^2}\;.
\end{align*}
Then, using that $\dotp{\vec {\hat g}}{\vec v^{\ast}}=\dotp{\vec{\hat g}}{(\vec v^{\ast})^{\perp_{\vec v}}\sin\theta}$,
we have  that $\dotp{\vec {\hat g}}{(\vec v^{\ast})^{\perp_{\vec v}}} \geq \frac{c}{\beta \sin\theta}$,
thus
\begin{align*}
\dotp{{\vec v}'}{\vec v^{\ast}}
&\geq \cos\theta + \frac{\lambda  c/\beta-\lambda^2\snorm{2}{\vec{ \hat g}}^2}{1+\lambda^2\snorm{2}{\vec{\hat g} }^2}
\geq 	\cos\theta + \frac{\lambda  c/\beta-\lambda^2\beta^2 }{1+\lambda^2\snorm{2}{\vec{ \hat g}}^2}
= \cos\theta + \frac {1}{2} \frac{\lambda c/\beta}{1+\lambda^2\snorm{2}{\vec{ \hat g} }^2} \;,
\end{align*}
where in the first inequality we used that $\snorm{2}{\vec{\hat g}} \leq \beta$ and in the
second that for $\lambda= c/(2\beta^3)$ it holds $c/\beta-\lambda\beta^2\geq  c/(2\beta)$.
Finally, we have that
\begin{align*}
\cos\theta'=\dotp{{\vec v}'}{\vec v^{\ast}} \geq \cos\theta + \frac{1}{2} \frac{\lambda c/\beta}{1+\lambda^2 (9\beta^2)}
\geq \cos\theta +\frac{1}{4}\lambda c/\beta = \cos\theta +\frac{1}{2} \lambda^2 \beta^2\;.
\end{align*}
This completes the proof.
\end{proof}

To analyze the sample complexity of Algorithm~\ref{alg:find_certificate}, we
require the following simple lemma, which bounds the sample complexity
of estimating the update function and testing the current candidate certificate.
The simple proof can be found in Appendix~\ref{ap:cert}.

\begin{lemma}[Estimating $\vec{g}$]\label{lem:algorithm_function_g}
Let $\D$ be a $(2, L, R, \beta)$-well-behaved distribution.
Given $N=O((d \beta^2/\eps^2)\log(d/\delta))$ i.i.d samples $(\sample{\bx}{i}, \sample{y}{i}))$
from $\D$, the estimator
$\vec{ \hat g}= \frac{1}{N}\sum_{i=1}^N \1_{B^{R/2}}\left(\sample{\bx}{i}\right) \sample{y}{i}  \sample{\bx}{i}$
satisfies the following with probability at least $1-\delta$:
\begin{itemize}
\item $\snorm{2}{\vec{ \hat g} - \vec{ g} } \leq \eps$, where $\vec g=\E_{(\vec x,y) \sim \D} [\1_{B^{R/2}}(\bx) \,  y \, \bx ]$, and
\item $\snorm{2}{\vec{ \hat g} } \leq e\beta+\eps\;.$
\end{itemize}
\end{lemma}

Before we proceed with the proof of Proposition~\ref{prop:cert-wb}, we show that
we can efficiently check for the certificate in
Line~\ref{alg1:line_check_certificate}  of
Algorithm~\ref{alg:find_certificate} with high probability.

\begin{lemma}\label{lem:check_certificate}
Let $\widehat{\D}_N$ be the empirical distribution obtained from $\D$ with $N=O(\log(1/\delta)/\eps^2)$ samples.
Then, with probability $1-\delta$, for every $t\in \R_+$,
$|\E_{(\vec x,y) \sim \D} \left[\1_{B^{t}}(\bx) \, y \right] -\E_{(\vec x,y) \sim \widehat{\D}_N} \left[\1_{B^{t}}(\bx) \, y \right]|\leq \eps\;.$
\end{lemma}

The proof of Lemma~\ref{lem:check_certificate} can be found in Appendix~\ref{ap:cert}.
We are now ready to prove Proposition~\ref{prop:cert-wb}.

\begin{proof}[Proof of Proposition~\ref{prop:cert-wb}]
Consider the $k$-th iteration of Algorithm~\ref{alg:find_certificate}.
Let $\sample{\vec g}{k}=\E_{(\vec x,y) \sim \D} [\1_{B_k^{R/2}}(\bx)y \bx  ]$,
where $B_k^{R/2}(\bx) = \{\bx : -R \leq \dotp{\bx}{\vec v^{(k)}} \leq -R/2\}$ and
$G:=\sqrt{b} (RL/A)^{O(1/\alpha)}$.
Moreover, let $\sample{\vec{ \hat g}}{k} = \frac{1}{N}\sum_{i=1}^N \1_{B_k^{R/2}}\left(\sample{\bx}{i}\right)  \sample{y}{i} \sample{\bx}{i}$
and note that from Lemma~\ref{lem:algorithm_function_g} we have that
given $N=O\left(d\beta^2/G^4\log(1/(LR))\log(d T/\delta)\right)$ samples,
for every iteration $k$, it holds that $\snorm{2}{\sample{\vec{ \hat g}}{k} -\sample{\vec{ g}}{k} }\leq G^2/(16 \beta)$
and $\snorm{2}{ \sample{\vec{ \hat g}}{k} }\leq e\beta+G^2/(16 \beta)\leq 3\beta$, with probability $1-\delta/T$.

We first show that if Condition 1 of Lemma~\ref{lem:improving_w_perp} is satisfied,
then Algorithm~\ref{alg:find_certificate} terminates at Line~\ref{alg1:return} returning a certifying vector.
The only issue is that we have access to the empirical distribution $\widehat{\D}_N$ instead of $\D$.
From Lemma~\ref{lem:check_certificate}, we have that the empirical expectation of Line~\ref{alg1:line_check_certificate} is
sufficiently close to the true expectation that appears in Condition 1 of Lemma~\ref{lem:improving_w_perp}, thus it is going to find it.

We now analyze the case when Condition 1 of Lemma~\ref{lem:improving_w_perp} is not true.
From Lemma~\ref{lem:improving_w_perp}, we immediately get that since Condition 1 is not satisfied, Condition 2 is true.
Then, using the update rule
$\sample{\vec v}{k+1}=\frac{\sample{\vec v}{k}+\lambda\sample{\vec{ \tilde g}}{k}}{\snorm{2}{\sample{\vec v}{k}+\lambda\sample{\vec{\tilde g}}{k} }}$
with $\lambda= G^2/(64\beta^3)$, where $\sample{\vec{ \tilde g}}{k}=\proj_{(\vec v^{(k)})^{\perp}}\sample{\vec{ \hat g}}{k}$
(here $\sample{\vec{ \tilde g}}{k}$ is the $\sample{\vec{ \hat g}}{k}$ with the component on the direction $\vec v^{(k)}$ removed).
Note that this procedure only decreases the norm of $\vec{\tilde{g}}$ (by the Pythagorean theorem).
Then, from Lemma~\ref{lem:gradient}, we have
$\dotp{\sample{\vec v}{k+1}} {\vec v^{\ast}}\geq \dotp{\sample{\vec v}{k}}{\vec v^{\ast}} +G^4/\beta^4$.

The update rule is repeated for at most $O(\beta^4/G^4)$ iterations. From
Lemma~\ref{lem:improving_w_perp}, we have that a certificate exists if the
angle with the optimal vector is sufficiently small.  Putting everything together,
our total sample complexity is
$N=\tilde{O}\left(\frac{d \beta^4}{b^2G^4}\right)\log(1/\delta)$. It is also clear that the runtime
is $\poly(N, d)$, which completes the proof.
\end{proof}

\subsection{Proof of Theorem~\ref{thm:cert-wb}}\label{ssec:wb-cert-thm}

To prove Theorem~\ref{thm:cert-wb}, we will use the iterative algorithm developed in
Proposition~\ref{prop:cert-wb} initialized with a uniformly random unit vector $\vec v_0$.
It is easy to show that such a random vector will have non-trivial correlation with
$\vec v^{\ast}$.

\begin{fact}[see, e.g., Remark 3.2.5 of \cite{Ver18}]\label{fct:random_initialization}
Let $\vec v$ be a unit vector in $\R^d$. For a random unit vector $\vec u\in \R^d$, with constant probability,
it holds $|\dotp{\vec v}{\vec u}|= \Omega(1/\sqrt{d})$.
\end{fact}

We now present the proof of Theorem~\ref{thm:cert-wb} putting together the machinery developed
in the previous subsections.

\begin{proof}[Proof of Theorem~\ref{thm:cert-wb}]
As explained in Section~\ref{ssec:cert-int}, we are looking for a certificate function $T_\bw(\bx)$
of the form given in Equation~\eqref{eq:certificate_form}. As argued in Section~\ref{ssec:reduction},
the search for such a certificate function can be simplified by projecting the samples to a $(d-1)$-dimensional subspace
via the perspective projection.

From Proposition~\ref{prop:reduction}, choosing
$\rho=O(\theta/\sqrt{d})$, there is a $c=(LR)^{O(1)}$ such that the resulting distribution $\Dperp$ is
$(2, c \theta/\sqrt{d}, \sqrt d/\theta, \beta \sqrt{d}/(c\theta)\log(\sqrt{d}/\theta))$-well-behaved and satisfies
the $(\tsyb,A d^{1/2}/(c\theta))$-Tsybakov noise condition.

From Fact~\ref{fct:random_initialization}, a random unit vector $\vec v\in \R^{d-1}$ with constant probability satisfies
$\dotp{\vec v}{{(\vec w^{\ast})^{\perp_{\vec w}}}}= \Omega(1/\sqrt{d})$.
We call this event $\cal E$.

From Proposition~\ref{prop:cert-wb}, conditioning on the event $\cal E$ and using
$\frac{\conb^4}{b^2}\left(\frac{A}{RL}\right)^{O(1/\tsyb)}\log(1/\delta)$ samples,
with probability $1-\delta$, we get a $(\vec v', R, t_0)$ such that
$$\E_{(\vec x,y) \sim \Dperp} \left[\1[-R \leq \dotp{\vec v'}{\vec x} \leq - t_0] \,  y \right]
\leq - \left(\theta L R/(\tsya d)\right)^{O(1/\alpha)}/\beta\;.$$

By inverting the transformation (Lemma~\ref{lem:certificate_reduction}), we get that
$$\E_{(\vec x,y) \sim  \D} \left[T_{\vec w}(\bx)\dotp{\vec x}{\vec w}  y \right] \leq - \left(\theta L R /(\tsya d)\right)^{O(1/\alpha)}/\beta \;.$$

Overall, we conclude that with constant probability Algorithm~\ref{alg:find_certificate} returns a valid certificate.
Repeating the process $k=O(\log(1/\delta))$ times, we can boost the probability to $1-\delta$.
The total number of samples for finding and testing these candidate certificates until we find a correct one with probability at least $1-\delta$
is $N= \left(\frac{d \, \tsya}{\theta L R }\right)^{O(1/\tsyb)}\log(1/\delta)$. It is also clear that the runtime is $\poly(N, d)$,
which completes the proof.
\end{proof}

\section{More Efficient Certificate for Log-Concave Distributions} \label{sec:lc}

In this section, we present a more efficient certificate algorithm for the
important special case of isotropic log-concave distributions. To achieve this,
we use Algorithm~\ref{alg:find_certificate} from the previous section starting
from a significantly better initialization vector.  To obtain such an
initialization,  we leverage the structure of log-concave distributions.  The
main result of this section is the following theorem.

\begin{theorem}[Certificate for Log-concave Distributions] \label{thm:cert-lc}
Let $\D$ be a distribution on $\R^{d} \times \{\pm 1\}$ that satisfies the
$(\tsyb,\tsya)$-Tsybakov noise condition with respect to an unknown halfspace
$f(\bx) =\sgn(\dotp{\vec w^{\ast}}{\bx})$ and is such that $\D_{\bx}$ is isotropic log-concave.
Let $\vec w$ be a unit vector that satisfies $\theta(\vec w,\vec w^{\ast})\geq \theta$,
where $\theta \in (0, \pi]$. There is an algorithm that, given as input $\vec w$, $\theta$,
and $N = \poly(d)\cdot\left(\frac{ \tsya}{\theta}\right)^{O(1/\tsyb^2)}\log(1/\delta)$ samples from
$\D$, it runs in $\poly(d,N)$ time, and with probability at least $1-\delta$
returns a certifying function $T_\bw:\R^d\mapsto \R_+$ such that
\begin{equation}\label{eqn:cert-lc}
\E_{(\vec x,y) \sim \D} \left[T_\bw(\bx) \, y \dotp{\vec w}{\bx} \right] \leq - \left(\frac{\theta}{\tsya}\right)^{O(1/\alpha^2)}\;.
\end{equation}
\end{theorem}

In other words, we give an algorithm whose sample complexity and running time
as a function of $d$ is a fixed degree polynomial, independent of the noise
parameters.

To establish Theorem~\ref{thm:cert-lc}, we apply Algorithm~\ref{alg:find_certificate}
starting from a better initialization vector. The main technical contribution of this section
is an efficient algorithm to obtain such a vector for log-concave marginals.

\begin{theorem}[Efficient Initialization for Log-Concave Distributions]\label{thm:init-lc}
Let $\D$ be a distribution on $\R^{d} \times \{\pm 1\}$ that satisfies the
$(\tsyb,\tsya)$-Tsybakov noise condition with respect to an unknown halfspace
$f(\bx) =\sgn(\dotp{\vec w^{\ast}}{\bx})$ and is such that $\D_{\bx}$ is isotropic log-concave.
There exists an algorithm that, given an $\eps>0$, a unit vector $\vec w$ such that
$\snorm{2}{\vec w^{\ast}-\vec w}=\Theta(\eps)$, and $N = \poly(d) \cdot (A/(\alpha\eps))^{O(1/\alpha)}$
samples from $\D$, it runs in $\poly(d,N)$ time, and with constant probability returns a unit
vector $\vec v$ such that
$\dotp{\vec v} {(\vec w^\ast)^{\perp_\bw }}\geq (\alpha\eps/A)^{O(1/\alpha)}$,
where $(\vec w^\ast)^{\perp_\bw }$ is the component of $\vec w^\ast$ perpendicular to $\vec w$.
\end{theorem}

\subsection{Intuition and Roadmap of the Proof} \label{ssec:lc-init-int}
Here we sketch the proof of Theorem~\ref{thm:init-lc} and point to the relevant
lemmas in the formal argument (Section~\ref{ssec:init-lc-proof}).
Given a weight vector $\bw$ of unit length, our goal is to find a \nnew{unit vector} $\bv$
that has non-trivial correlation with $\wperp$, i.e., such that $\dotp{\nnew{\wperp}}{\bv}$ is
roughly $\eps^{1/\alpha}$, where $\wstar$ is the optimal halfspace.

Our first step is to condition on a thin band around the current candidate
$\bw$ (similarly to Section~\ref{sec:cert-wb}, see Figure~\ref{fig:hardband}).
When the size of the band approaches $0$, we get an instance whose separating hyperplane is
perpendicular to $\wperp$ and has much larger Tsybakov noise rate.
After that, we would like (similarly to Section~\ref{sec:cert-wb}) to project the points on
the subspace $\wperp$. Instead of having a zero length band, we will instead
take a very thin band.  We have already seen in Section~\ref{sec:cert-wb}
that we can apply a perspective transformation in order to project the points on
$\wperp$ and obtain an instance that satisfies the Tsybakov noise condition (with somewhat worse parameters).
Unfortunately, for the current setting of log-concave distributions, we cannot use
the perspective projection, as it \emph{does not preserve the log-concavity} of the underlying distribution.
On the other hand, we know that log-concavity is preserved when we condition on convex sets
(such as the thin band we consider here) and when we perform orthogonal projections.

As we have seen (see Figure~\ref{fig:orth_proj}), an orthogonal projection will
create a ``fuzzy" region with arbitrary sign.  However, we can control the
probability of this ``fuzzy" region by taking a sufficiently thin random band.
In particular, instead of Tsybakov noise, we will end up with the following
noise condition: For some small $\xi > 0$, with probability $2/3$ the noise
$\eta(\bx)$ is bounded  above by \nnew{$1/2-\xi$}, and with probability roughly
$\xi^{\Theta(1)}$ we have $\eta(\bx) > 1/2$ (this corresponds to the
probability of the ``fuzzy" region).  For the proof of this statement and detailed
discussion on how the random band results in this above noise guarantee, see
Lemma~\ref{lem:noise-random-band}.

From this point on, we will be working in the subspace $\bw^{\perp}$ and
assume that the distribution satisfies the aforementioned noise condition. As we
have discussed, the marginal distribution on the examples remains log-concave and
it is not hard to make its covariance be close to the identity.  However,
conditioning on the thin slice may result in a distribution with large mean,
even though originally the distribution was centered.  \nnew{This is a non-trivial technical issue.}
We cannot simply translate the distribution to be origin-centered, as this would result
in a potentially very biased optimal halfspace.  Our proof crucially relies on the assumption of
having a distribution that is \emph{nearly} centered and at the same time
for the optimal halfspace to have \emph{small bias}.
\nnew{We overcome this obstacle in Step~\ref{step:one} below.}

Our approach is as follows:
\begin{enumerate}
  \item \label{step:one} First, we show that there is an efficient rejection sampling procedure
    that preserves log-concavity and gives us a distribution that is nearly isotropic
    (see Definition~\ref{def:apporx-isotro}). For the algorithm and its detailed proof of correctness,
    see Algorithm~\ref{alg:make_isotropic} and Lemma~\ref{lem:sampling}.
  \item \label{step:two} Then we show the following statement: Under the following assumptions
    \begin{itemize}
      \item[(i)] the $\bx$-marginal is nearly isotropic,
      \item[(ii)] the optimal halfspace has sufficiently small bias, and
      \item[(iiii)] the noise $\eta(\bx)$ is bounded away from $1/2$ with constant probability,
    \end{itemize}
    we can compute in polynomial time a vector $\bv$ with good correlation to the target $\wperp$.
    This is established in Proposition~\ref{prop:nontriavial_angle}.
\end{enumerate}

We start by describing our algorithm to transform the distribution to nearly
isotropic position (Step~\ref{step:one} above).
We avoid translating the samples by reweighting the distribution using rejection sampling.
\nnew{To achieve this, we find an approximate stationary point  of the non-convex objective}
$F(\vec r) = \snorm{2}{\E_{\bx \sim \D_{\bx}} [ \bx \max(1, \exp(- \dotp{\vec r}{\bx})]}^2$.
Notice that, since this is a non-convex objective as a function of $\vec r$, we can only use
(projected) SGD to efficiently find a stationary point. In particular, we show
that a $\gamma$-stationary point $\vec r$ of $F(\vec r)$ will make the above
norm of the expectation roughly $O(\gamma)$ (Claim~\ref{clm:stationary-point}).
Therefore, in time $\poly(d/\gamma)$, we find a reweighting of the initial distribution
whose mean is close to $\vec 0$. Given this point $\vec r$, we then perform rejection
sampling: We draw $\bx$ from the initial distribution $\D$ and accept it with
probability $ \max(1, \exp(- \dotp{\vec r}{\bx}))$, i.e., we ``shrink" the
distribution along the direction $\vec r$.

We now explain how to handle the setting that the distribution is approximately log-concave
(Step~\ref{step:two} above).
After we make our distribution nearly isotropic, we compute the degree-$2$ Chow parameters
of the distribution, i.e., the vector $\E_{(\bx, y) \sim \D}[y \bx]$ and
the matrix $\E_{(\bx, y) \sim \D}[y (\bx \bx^\intercal-\vec I)]$.  We show that there exists a degree-$2$
polynomial $p(\dotp{\wperp}{\bx})$ that correlates non-trivially with the
labels $y$ (Lemma~\ref{lem:quad-id}).  This means that $\wperp$ correlates
reasonably with the degree-$2$ Chow parameters. In particular, $\wperp$ has a
non-trivial projection on the subspace \nnew{$V$} spanned by the degree-$1$ Chow parameters
(this is a single vector) and the eigenvectors of the degree-$2$ Chow matrix with
large eigenvalues.  Our plan is to return a random unit vector of the subspace
$V$.  However, in order for this random vector to have non-trivial correlation
with $\wperp$, we also need to show that the dimension of $V$ is not very large.

The last part of our argument shows that $V$ has reasonably small dimension.
To prove this, we first show that the dimension of $V$ can be bounded above by
the variance of the projection of $\D$ onto $V$, $\D^{\proj_V}$,
$\var_{\bx \sim \D^{\proj_V}}[\snorm{2}{\bx}^2]$.
Then we make essential use of a recent ``thin-shell'' result about log-concave measures that
bounds from above $\var_{\bx \sim \D^{\proj_V}}[\snorm{2}{\bx}^2]$, see
Lemma~\ref{lem:elads} and Lemma~\ref{lem:dim-ub}.

\subsection{Proof of Theorem~\ref{thm:init-lc}} \label{ssec:init-lc-proof}

The proof of Theorem~\ref{thm:init-lc} requires a number of intermediate results.
As already mentioned, our initialization algorithm works by restricting $\D$ to a narrow band perpendicular to $\vec w$.
Unfortunately, this restriction will be log-concave but will no longer be isotropic, even in the directions perpendicular to $\vec w$. However, it will be close in the following sense.

\begin{definition}[$(\alpha, \beta)$-isotropic distribution]\label{def:apporx-isotro}
We say that a distribution $\D$ is $(\alpha, \beta)$-isotropic,
if for every unit vector $\vec u \in \R^d$, it holds
$|\E_{\x \sim \D}[\dotp{\vec x}{\vec u}]|\leq \alpha$ and $1/\beta\leq\E_{\x \sim \D}[\dotp{\vec x}{\vec u}^2]\leq \beta$.
\end{definition}

\paragraph{Useful Technical Tools.} We will require the following standard
anti-concentration result for low-degree multivariate polynomials under log-concave distributions.

\begin{lemma}[Theorem 8 of \cite{CW:01}]\label{lem:carbery-wright} Let $\D$ be a
log-concave distribution on $\R^d$ and $p:\R^d \mapsto R$ be a polynomial of degree at most $n$.
Then there is an absolute constant $C>0$ such that for any $0<q<\infty$ and $t\in \R_+$, it holds
$ \pr_{\x\sim \D}[|p(\x)|\leq t]\leq C q t^{1/n} \E_{\x\sim \D}[|p(\x)|^{q/n}]^{1/q}\;.$
\end{lemma}

The following statement is well-known.
(It follows for example by combining Theorem 5.14 of \cite{LV07} and Lemma 7 of \cite{KlivansLT09}.)

\begin{fact}\label{fact:lc-basics}
Let $\vec z$ be an isotropic log-concave distribution on $\R^d$ and let $\gamma(\cdot)$ be its density function.
There exists a constant $c_d>0$ such that:
\begin{enumerate}
\item For any $\vec z$ with $\snorm{2}{\vec z}\leq c_d$, we have that $\gamma(\vec z)\geq c_d$.
\item For any $\vec z$, we have that $\gamma(\vec z)\leq \nnew{1/c_d}\exp(-\nnew{1/c_d}\snorm{2}{\vec z})$.
\end{enumerate}
\end{fact}

Our proof makes essential use of the following ``thin-shell'' estimate
bounding the variance of the norm of any isotropic log-concave random vector.

\begin{lemma}[Corollary 13 of \cite{LV17}] \label{lem:elads}
Let $\D$ be any isotropic log-concave distribution on $\R^d$.
We have that $\var_{\vec x\sim \D}[\snorm{2}{\vec x}^2]\leq d^{3/2}\;.$
\end{lemma}

\noindent \nnew{In particular, it is important for our analysis that the above bound is sub-quadratic in $d$.}

Finally, we will require the following simple lemma bounding the sample complexity of approximating
the degree-$2$ Chow parameters of a halfspace under isotropic log-concave distributions.

\begin{lemma} \label{lem:chow-params-complexity}
Let $\D$ be an isotropic log-concave distribution on $\R^d$ and $\widehat{\D}_N$ be the empirical
distribution obtained from $\D$ with $N=\poly(d/\eps)$ samples. Then, with high constant probability,
we have
$ \snorm{2}{\E_{(\vec x,y)\sim \D}[y\vec x] -\E_{(\vec x,y)\sim \widehat{\D}_N}[y\vec x] }\leq \eps$ and
$ \snorm{F}{\E_{(\vec x,y)\sim \D}[y(\vec x\vec x^\intercal-\vec I)] -\E_{(\vec x,y)\sim \widehat{\D}_N}[y(\vec x\vec x^\intercal-\vec I)]}\leq \eps$.
\end{lemma}
The proof of this lemma can be found in Appendix~\ref{ap:logcon}.

\medskip

We now have the necessary tools to proceed with our proof.
We start by showing how we can find a vector $\vec v$ with non-trivial correlation with
$(\vec w^\ast)^{\perp_\bw }$ if \nnew{the marginal distribution is (approximately) isotropic.
Since in general this will not hold, we will then need to reduce to the isotropic case.}

\begin{proposition}\label{prop:nontriavial_angle}
Let $\D$ be a distribution on $\R^{d} \times \{\pm 1\}$ such that $\D_{\bx}$ is
$(\alpha, \beta)$-isotropic log-concave. Let $f(\bx) = \sgn(\dotp{\vec v^\ast}{\bx} - \theta)$ be
such that $\pr_{(\x,y)\sim \D}[y\neq f(\vec x)|\vec x] = \eta(\vec x)$, where for some $\xi>0$
we have  that $\pr_{\x\sim \D_{\x}}[\eta(\x)<1/2-\xi]\geq 2/3$
and $\pr_{\x\sim \D_{\x}}[\eta(\x)>1/2]\leq \xi'$, where
$\xi'$ is a constant degree polynomial in $\xi$\footnote{It is not difficult to verify
that $\xi' = \Theta(\xi^3)$ suffices.}.
Then, as long as $|\alpha|+|\theta|$ is less than a sufficiently small constant multiple of $1/(\log(1/\xi))$,
there exists an algorithm with sample complexity and runtime $\poly(d/\xi)$ that
with constant probability returns a unit vector $\vec v\in \R^d$ such that
$\dotp{\vec v}{\vec  v^\ast} >\poly(\xi)$.
\end{proposition}

\begin{proof}
For clarity of the analysis, we begin by presenting our algorithm for the case that $\D_\x$ is exactly isotropic log-concave.
We then show how the algorithm and its analysis can be modified for the approximate log-concave setting.

Our algorithm is fairly simple. We compute high-precision
estimates $\vec T_1'$ and $\vec T_2'$ of the vector $\vec T_1 := \E_{(\vec x,y)\sim \D}[y\vec x]$ and the matrix
$\vec T_2 := \E_{(\vec x,y)\sim \D}[y(\vec x\vec x^\intercal-\vec I)]$ respectively.
This can be easily done by taking $\poly(d/\eps)$ samples from $\D$ and using the empirical estimates
(see Lemma~\ref{lem:chow-params-complexity}). We then define $V$ to be the subspace spanned by $\vec T_1$ and the eigenvectors of $\vec T_2$ whose eigenvalue has absolute value at least $2\zeta$, for $\zeta$ some sufficiently
large \nnew{constant} power of $\xi$. The algorithm returns a uniform random unit vector $\vec v$ from $V$.

It is clear that the above algorithm has polynomial sample complexity and runtime.
We need to show that with constant probability it holds that
$\dotp{\vec v}{\vec v^\ast}>\poly(\xi)$. The desired statement will follow by establishing the following two claims:
\begin{enumerate}
\item The size of the projection of $\vec v^{\ast}$ onto $V$ is at least $\poly(\xi)$.
\item The dimension of $V$ is at most $\poly(1/\xi)$.
\end{enumerate}
The desired result then follows by noting that the median value of $|\dotp{\vec{v^\ast}}{\vec v}|$ is on the
order of $\snorm{2}{\proj_V(\vec v^\ast)}/\sqrt{\dim(V)}$, and observing
that the sign of the inner product is independent of its size.

To establish the first claim, we prove the following lemma for isotropic log-concave distributions.
\begin{lemma}\label{lem:quad-id}
Let $\D_\x$ be isotropic log-concave. There exists a degree-$2$ polynomial $p: \R \to \R$ such that
$\E_{\vec x\sim \D_\x}[p(\dotp{\vec v^\ast}{\vec x})]=0,$
$\E_{\vec x\sim \D_\x}[p(\dotp{\vec v^\ast}{\vec x})^2]=1$, and
$\E_{(\vec x,y)\sim \D}[y \, p(\dotp{\vec v^\ast}{\vec x})] = \Omega(\xi)$.
\end{lemma}
\begin{proof}
We consider the polynomial $$q(x) = (x-\theta)(x+1/\theta) = x^2 + (1/\theta-\theta)x-1$$
and we set $p(x)=q(x)/\sqrt{\E_{\vec x\sim \D_\x}[q(\dotp{\vec v^\ast}{\vec x})^2]}$.
It is easy to see that $\E_{\vec x\sim \D_\x}[p(\dotp{\vec v^\ast}{\vec x})]=0$ and
$\E_{\vec x\sim \D_\x}[p(\dotp{\vec v^\ast}{\vec x})^2]=1$. To show that
$\E_{(\vec x,y)\sim \D}[yp(\dotp{\vec v^\ast}{\vec x})] = \Omega(\xi)$,
we note that
$$\E_{(\vec x,y)\sim \D}[yp(\dotp{\vec v^\ast}{\vec x})] = \E_{\vec x\sim \D_\x}[(1-2\eta(\vec x))f(\vec x)p(\dotp{\vec v^\ast}{\vec x})] \;.$$
We observe that if $|\dotp{\vec v^\ast}{\vec x}| \leq  1/|\theta|$, then
$\sign(p(\dotp{\vec v^\ast}{\vec x}))=f(\vec x)$, where $f(\bx) = \sgn(\dotp{\vec v^\ast}  {\bx} - \theta)$.
Thus, unless $|\dotp{\vec v^\ast}{\vec x}| > 1/|\theta|$ or $\eta(\vec x)>1/2$ (which happens with
probability at most $\xi'$, a sufficiently high power of $\xi$), we have that
$(1-2\eta(\vec x))f(\vec x)p(\dotp{\vec v^\ast}{\vec x})\geq 0$ except with probability at most $\xi'$.

Let $I(\vec x)$ denote the indicator of the event $(1-2\eta(\vec x))f(\vec x)p(\dotp{\vec v^\ast}{\vec x})< 0$.
We have that
\begin{align*}
\E_{(\vec x,y)\sim \D}[y \, p(\dotp{\vec v^\ast}{\vec x})]
& =\E_{\vec x\sim \D_\x}[|(1-2\eta(\vec x))p(\dotp{\vec v^\ast}{\vec x})|]- 2\E_{\vec x\sim \D_\x}[|(1-2\eta(\vec x))p(\dotp{\vec v^\ast}{\vec x})|I(\vec x)] \\
& \geq \E_{\vec x\sim \D_\x}[|(1-2\eta(\vec x))p(\dotp{\vec v^\ast}{\vec x})|]
- 2\sqrt{\E_{\vec x\sim \D_\x}[I^2(\vec x)]\E_{\vec x\sim \D_\x}[p(\dotp{\vec v^\ast}{\x})^2]}\\
& \geq \E_{\vec x\sim \D_\x}[|(1-2\eta(\vec x))p(\dotp{\vec v^\ast}{\vec x})|] - 2\sqrt{\xi'}.
\end{align*}
Recall that by assumption there is at least a $2/3$ probability that $(1-2\eta(\vec x))\geq \xi$.

By anti-concentration of Gaussian polynomials, Lemma~\ref{lem:carbery-wright}, applied for $q=4$ and $n=2$,
we have that
$\pr_{\vec x\sim \D_\x}[|p(\dotp{\vec v^\ast}{\vec x})|\leq t]=O(\sqrt{t})$.
Thus, for small enough $t$, we have that
$|p(\dotp{\vec v^\ast}{\vec x})| = \Omega(1)$ with probability at least $2/3$.
Therefore, with probability at least $1/3$ both statements hold.
Since $|1-2\eta(\vec x)||p(\dotp{\vec v^\ast}{\vec x})| \geq 0$ for all $\vec x$, we have that
$\E_{\vec x\sim \D_\x}[|1-2\eta(\vec x)||p(\dotp{\vec v^\ast}{\vec x})|]= \Omega(\xi)$.
This completes our proof.
\end{proof}

Given Lemma \ref{lem:quad-id}, it is not hard to see that
$p(\dotp{\vec v^\ast}{\vec x}) = a(\dotp{\vec v^\ast}{\vec x})+b((\dotp{\vec v^\ast}{\vec x})^2-1)$
for some real numbers $a$ and $b$ with $|a|+|b|=\Theta(1)$.
We note that there is another way to compute $\E_{(\vec x,y)\sim \D}[yp(\dotp{\vec v^\ast}{\vec x})]$ relating it to
$\vec T_1$ and $\vec T_2$. In particular, we can write
\begin{align*}
\E_{(\vec x,y)\sim \D}[y \, p(\dotp{\vec v^\ast}{\vec x})]
& = a \E_{(\vec x,y)\sim \D}[y(\dotp{\vec v^\ast}{\vec x})] + b \E_{(\vec x,y)\sim \D}[y((\dotp{\vec v^\ast}{\vec x})^2-1)]\\
& = a \dotp{\vec v^\ast}{\E_{(\vec x,y)\sim \D}[y \vec x]}+ b \E_{(\vec x,y)\sim \D}[y((\vec v^\ast)^\intercal (\vec x\vec x^\intercal-\vec I)\vec v^\ast)]\\
& = a \dotp{\vec v^\ast}{\vec T_1} + b (\vec v^\ast)^\intercal \vec T_2 \vec v^{\ast}.
\end{align*}
Thus, Lemma \ref{lem:quad-id} implies that either $|\dotp{\vec v^\ast}{\vec T_1}|=\Omega(\xi)$ or
$|(\vec v^\ast)^\intercal \vec T_2 \vec v^\ast| = \Omega(\xi)$.

Assuming that $\vec T_1'$ and $\vec T_2'$ estimate $\vec T_1$ and $\vec T_2$ to error less than this quantity, i.e., $O(\xi)$,
the above implies that either $\dotp{\vec v^\ast} {\vec T_1' } = \Omega(\xi)$ or
$(\vec v^\ast)^\intercal \vec T_2' \vec v^\ast = \Omega(\xi).$
In the former case, we have that $\snorm{2}{\proj_V(\vec v^\ast)} \geq |\dotp{\vec v^\ast}{\vec T_1}| = \Omega(\xi).$
In the latter case, we note that since $V$ contains the span of all eigenvectors of $\vec T_2'$ with eigenvalue having absolute
value at least $\zeta$, it holds that
$|(\vec v^\ast)^\intercal \vec T_2' \vec v^\ast| \leq \zeta + \snorm{2}{\vec T_2'}\snorm{2}{\proj_V(\vec v^\ast)}$.
This will imply that in this case as well we have that $\snorm{2}{\proj_V(\vec v^\ast)}= \Omega(\xi)$,
if $\|\vec T_2'\|_{2}$ is $O(1)$.  To show this, we note that for any unit vector $\vec v$, we have
\begin{align*}
\vec v^\intercal \vec T_2 \vec v
& = \E_{(\vec x,y)\sim \D}[y(\vec v^\intercal(\vec x\vec x^\intercal - \vec I) \vec v)] = \E_{(\vec x,y)\sim \D}[y(\dotp{\vec v}{\vec x}^2-1)]
\leq \sqrt{\E_{(\vec x,y)\sim \D}[y^2]\E_{\vec x\sim \D_\x}[(\dotp{\vec v}{\vec x}^2-1)^2]} = O(1) \;.
\end{align*}
This completes the proof that the projection of $\vec v^\ast$ onto $V$ has size at least
$\poly(\xi)$.

It remains to show that the dimension of $V$ is at most $\poly(\xi)$.
We prove the following lemma:
\begin{lemma}\label{lem:dim-ub}
We have that  $\dim(V)=O(\zeta^{-4})$.
\end{lemma}
\begin{proof}
Let $V_{+}$ denote the subspace spanned by the eigenvectors of $\vec T_2'$ with eigenvalue at least
$\zeta$. Let $V_{-}$ denote the subspace spanned by eigenvectors of eigenvalue at most $-\zeta$.
Clearly $\dim(V) \leq \dim(V_+)+\dim(V_-)+1$. We will show that $\dim(V_+)=O(\zeta^{-4})$ and the
bound on $\dim(V_-)$ will follow symmetrically.

Let $m=\dim(V_+)$ and let $\vec P$ be the projection matrix that maps a vector $\vec z$ onto $V_+$.
Since $\vec T_2'$ is sufficiently close to $\vec T_2$, the restriction of $\vec T_2'$ to $V_+$ will have
all of its eigenvalues at least $\zeta/2$. Therefore, it holds that
\begin{align*}
m \zeta/2 & \leq \tr(\vec P\vec T_2) = \E_{(\vec x,y)\sim \D}[y\,\tr(\vec P(\vec x\vec x^\intercal-\vec I))] =
\E_{(\vec x,y)\sim \D}[y(\snorm{2}{\vec P\vec x}^2-m)]\\ &\leq \sqrt{\E_{(\vec x,y)\sim \D}[y^2]\E_{\vec x\sim \D_\x}[(\snorm{2}{\vec P\vec x}^2-m)^2]}
= \sqrt{\var[\snorm{2}{\vec P\vec x}^2]} \;.
\end{align*}
In other words, we have that
$$m^2 \zeta^2 \leq 4 \var[\snorm{2}{\vec P\vec x}^2] \;.$$
To conclude the proof, observe that $\vec P\vec x$ is a log-concave distribution in $m$ dimensions,
since projections preserve log-concavity. From Lemma~\ref{lem:elads}, we have that
$\var[\snorm{2}{\vec P\vec x}^2 ]= O(m^{3/2})$ and together with the above, we obtain that
$m=O(\zeta^{-4})$. This completes our proof.
\end{proof}

Thus far, we have shown the desired claim if the distribution is in
isotropic position, $\theta=O(1/\log(1/\xi))$, and we have access to sufficiently accurate approximations
$\vec T_1', \vec T_2'$ to the degree-$2$ Chow parameters with accuracy $\zeta/2$.
To handle the case that the distribution $\D$ is $(\alpha,\beta)$-isotropic, we can let $\vec z\sim \D_{\vec z}'$,
where $\vec z = \cov[\vec x]^{-1/2}(\vec x-\E_{\vec x\sim \D_\x}[\vec x])$, be an isotropic log-concave distribution.
We need to show that if we have good approximations of $\E_{\vec x\sim \D_\x}[\vec x]$ and $\cov[\vec x]$,
we can compute $O(\zeta)$-approximations to $\vec T_1$ and $\vec T_2$ for $\vec z$
(i.e., $\E_{(\vec z,y)\sim \D'}[y\vec z]$ and $\E_{(\vec z,y)\sim \D'}[y(\vec z\vec z^\intercal-\vec I)]$).
By taking $\poly(d/\zeta)$ samples, we can compute $\widehat{\vec m}$ and $\widehat{\vec M}$ such that
$\snorm{2}{\widehat{\vec m}-\E_{\vec x\sim \D_\x}[\vec x]}\leq \zeta/16$ and $\snorm{2}{\widehat{\vec M}-\cov[\vec x]}\leq \zeta/16$.
Let  $\widehat{\vec z} = \widehat{\vec M}^{-1/2}(\vec x-\widehat{\vec m})$.
Then we have that $\snorm{2}{\widehat{\vec z}-\vec z}\leq \zeta/4$.
Thus, we obtain that
$\snorm{2}{\E_{(\vec z,y)\sim \D'}[y\vec z] -\E_{(\vec z,y)\sim \D'}[y\widehat{\vec z}] } \leq \E_{(\vec z,y)\sim \D'}[\snorm{2}{\vec z-\widehat{\vec z} }]\leq \zeta/4$
and similarly that
$\snorm{2}{\E_{(\vec z,y)\sim \D'}[y(\vec z\vec z^\intercal-\vec I)] -\E_{(\vec z,y)\sim \D'}[y(\widehat{\vec z}\widehat{\vec z}^\intercal-\vec I)] }\leq \zeta/4 \;.$
By approximating the degree-$2$ Chow parameters $\vec T_1,\vec T_2$ to accuracy $\zeta/4$, we obtain
overall error $\zeta/2$.

We note that $(\vec z,y)$ satisfies our assumptions for the function
$$f'(\vec x) = \sgn\left((\vec v^\ast)^\intercal\cov[\vec x]^{1/2} \vec x -(\theta- \langle \vec v^\ast , \E_{\vec x\sim \D_\x}[\vec x] \rangle)\right) \;.$$
From our assumptions, we have that $|\theta-\dotp{\vec v^\ast }{ \E_{\vec x\sim \D_\x}[\vec x]}|= O(1/\log(1/\xi))$.
Using the aforementioned algorithm for $\vec z$, this allows us to compute a $\vec v$ so that with constant
probability $\vec v^\intercal \cov[\vec x]^{1/2} \vec v^\ast \geq \poly(\xi)$,
or $\dotp{\cov[\vec x]^{1/2} \vec v} {\vec v^\ast }\geq \poly(\xi)$.
This completes the proof.
\end{proof}

Thus far, we have dealt with the case that the mean of our \nnew{log-concave}
distribution is sufficiently close to zero. As already mentioned, this property
will not hold in general after projection.  The following
important lemma shows that by conditioning on a random thin band before
projecting onto $\vec w^\perp$, we obtain a log-concave distribution whose mean
has small distance from the origin.  Moreover, we show that the noise condition
of the instance after we perform this transformation satisfies the assumptions
of Proposition~\ref{prop:nontriavial_angle}. We note that this is the step that crucially
relies on picking a {\em random} thin band.

\begin{lemma}[Properties of Transformed Instance]\label{lem:noise-random-band}
Let $\D$ be a distribution on $\R^d \times \{\pm 1\}$ that satisfies the $(\alpha,A)$-Tsybakov noise condition
with respect to an unknown halfspace $f(\bx) =\sgn(\dotp{\vec w^{\ast}}{\bx})$ and is such that
$\D_{\bx}$ is isotropic log-concave.  Fix $\eps > 0$ and unit vector $\vec w$ such that $\theta(\vec w, \wstar) = \Theta(\eps)$.
Let $s$ be a sufficiently small multiple of $\eps$\footnote{We need $s$ to be smaller than the absolute constant of Fact~\ref{fact:lc-basics} for dimension $d=2$.}. Set $\xi = (\Theta(s/A))^{1/\alpha}$ and $s' = \Theta(\xi^3 \, s \, \eps)$.
Pick $x_0$ uniformly at random from $[s, 2s]$ and define the random band $B_{x_0} = \{\x \in \R^d : \dotp{\x}{\vec w} \in [x_0, x_0 + s'] \}$.

Define the distribution $\D^\perp = \D_{B_{x_0}}^{\proj_{\vec w^\perp}}$,
the classifier $f^\perp(\x^\perp) = \sgn( x_0/\tan\theta + \dotp{\x^\perp}{\wperp})$, and the noise function
$$\eta^\perp(\x^\perp) = \pr_{(\vec z, y) \sim \D^\perp}[ y \neq f^\perp(\vec z) | \vec z = \x^\perp] \,.$$
Then $\D^\perp$ is an $(O(1), O(1))$-isotropic log-concave distribution and, with probability at least $99\%$,
satisfies the following noise condition:
$$\pr_{\x^\perp \sim \D_\x^\perp}[\eta^\perp(\x^\perp) \leq 1/2 - \xi  ] \geq 2/3 \quad \textrm{ and } \quad
 \pr_{\x^\perp \sim \D_\x^\perp}[\eta^\perp(\x^\perp) \geq 1/2] \leq \xi^3 \;.$$
\end{lemma}
\begin{proof}
We first calculate how far the distribution $\D_{B_{x_0}}^{\proj_{\vec w^\perp}}$ is
from being isotropic.  Since our final goal is to have a distribution whose
mean is arbitrarily close to $\vec 0$, we need to bound the distance from $\vec 0$
of the mean of the distribution obtained after we condition on $B$ and project onto
$\bw^\perp$. The following claim shows that the mean and covariance of $\D_{B_{x_0}}^{\proj_{\vec w^\perp}}$
differ from these of the initial distribution $\D$ only by constant factors
(additive for the mean and multiplicative for the covariance).
\begin{claim}\label{clm:conditional-projected-bounds}
$\D_{B_{x_0}}^{\proj_{\vec w^\perp}}$ is $(O(1), O(1))$-isotropic.
\end{claim}

\noindent The proof of Claim~\ref{clm:conditional-projected-bounds} relies on
Fact~\ref{fact:lc-basics} and is given in Appendix~\ref{ap:logcon}.

It remains to prove how the noise condition changes via the transformation.
In our argument, we are going to repeatedly use the following anti-concentration,
and anti-anti-concentration properties of log-concave distributions
that follow directly from Fact~\ref{fact:lc-basics}.  In particular,
for every interval $[a,b]$, we have that:
\begin{enumerate}
  \item $\pr_{\bx   \sim \D_\bx}[\dotp{\vec\x}{\vec v}\in[a,b]]=O(b-a)$ (anti-concentration). \label{local:anti}
  \item If $|a|, |b|$ are smaller than some absolute constant (see Fact~\ref{fact:lc-basics}),
    then it also holds that $\pr_{\bx   \sim \D_\bx}[\dotp{\vec\x}{\vec v}\in[a,b]]=\Omega(b-a)$  (anti-anti-concentration).\label{local:antianti}
\end{enumerate}
Using the condition
$\theta(\vec w, \wstar) = \Theta(\eps)$, we can assume that
$\vec w^\ast= \lambda_1 \vec w+ \lambda_2 (\vec w^\ast)^{\perp_\bw }$,
where $\lambda_1 = \cos\theta$ and $\lambda_2 = \sin \theta$.
It holds $|\lambda_1| = 1- \Theta(\eps)$ and $\lambda_2 = \Theta(\eps)$.
Next we set $\vec x=(\x_\bw,\vec x^\perp)$, where $\x_\bw=\dotp{\vec w}{\vec x}$ and
$\vec x^\perp$ is the projection of $\vec x$ on the subspace $\vec w^\perp$.

For some $\zeta \in (0,1)$, set $\xi = (\zeta s/A)^{1/\alpha}/2$.
In what follows, we shall see that $\zeta$ is some absolute constant, i.e.,
that $\xi = (\Theta(s/A))^{1/\alpha}$.
Recall that the orthogonal projection on $\vec w^\perp$ creates a ``fuzzy"
region, i.e., a region where $\eta^\perp(\x^\perp) \geq 1/2$, see
Figure~\ref{fig:orth_proj}. We first show that we can control the
probability that we get points inside this ``fuzzy" region.  More, formally we
will show that
\begin{equation} \label{eqn:agnostic-tiny}
\pr_{(\x^\perp,y) \sim \D^\perp}[\eta^\perp(\x^\perp) \geq 1/2] \leq \xi^3 \;.
\end{equation}
Notice that in this part of the proof the randomness of $x_0$ is not important
and we are able to establish a stronger claim that holds for every band $B_{x_0}$.
Conditioned on $\x \in B_{x_0}$, i.e., $\x_{\bw} \in [x_0, x_0 + s']$, it holds that
$$
\dotp{\vec w^\ast}{\vec x}
=  \lambda_1 \x_\bw + \lambda_2 \dotp{(\vec w^{\ast})^{\perp_\bw}}{\vec x^\perp}
= \lambda_1 x_0 + \lambda_2 \dotp{(\vec w^{\ast})^{\perp_\bw}}{\vec x^\perp}+ \rho s'
\,,
$$
for some $\rho \in [-1,1]$ (recall that $|\lambda_1|\leq 1$).  Notice that when
$|\lambda_1 x_0 + \lambda_2 \dotp{(\vec w^{\ast})^{\perp_\bw}}{\vec x^\perp}| > s'$,
$f^\perp(\x^\perp)$ is equal to the sign of $\dotp{\wstar}{\x}$ (recall that $\lambda_2 > 0$),
and therefore we are outside of the fuzzy region, see
Figure~\ref{fig:orth_proj}.
Thus, we need to bound the probability of the event
$|\lambda_1 x_0 + \lambda_2 \dotp{(\vec w^{\ast})^{\perp_\bw}}{\vec x^\perp}| \leq s'$,
or equivalently
$\dotp{(\vec w^{\ast})^{\perp_\bw}}{\vec x^\perp} \in [-\lambda_1 x_0 - s', - \lambda_1 x_0 + s'] =: I_{x_0}^{s'}$.
We have that
\begin{align*}
  \pr_{\x^\perp \sim \D^\perp_\x}\left[
  \dotp{(\vec w^{\ast})^{\perp_\bw}}{\vec x^\perp} \in I_{x_0}^{s'}  \right]
  =
  \frac{ \pr_{\x \sim \D_\x}\left[ \dotp{(\vec w^{\ast})^{\perp_\bw}}{\vec x} \in I_{x_0}^{s'} \right]}
  {\pr_{\x \sim \D_\x}[\x \in B_{x_0}]}
  = O(s'/(\lambda_2 s))
  \leq \xi^3
  \,,
\end{align*}
where to bound the numerator we used the anti-concentration property of $\D$, Property~\ref{local:anti},
for the interval $I_{x_0}^{s'}$ of length $s'$,
and to bound the denominator we used the anti-anti-concentration,
Property~\ref{local:antianti}.  The last inequality holds because we have that
$\lambda_2 = \Theta(\eps)$ and also, from the assumptions of the lemma, we have
$s' = \Theta(\xi^3 s \eps)$. This proves~\eqref{eqn:agnostic-tiny}.

Now we deal with the case where $|\lambda_1 x_0 + \lambda_2 \dotp{(\vec w^{\ast})^{\perp_\bw}}{\vec x^\perp}| \leq s'$,
i.e., we are in the non-fuzzy region of Figure~\ref{fig:orth_proj}. This is where the randomness of $x_0$
helps us control the probability that the noise is close to $1/2$.
Recall that,
$$\eta^\perp(\vec x^\perp) =
\pr_{(\x^\perp, y) \sim D^\perp} \left[y\neq \sgn(\dotp{(\vec w^{\ast})^{\perp_\bw}}{\vec x^\perp}+x_0)\right]
= \int_{x_0}^{x_0+s'}\eta(\x_\bw,\x^\perp) \gamma(\x_\bw|\x^\perp)\d \x_\bw \;,
$$
where
$\gamma(\x_\bw|\x^\perp)$ is the density of $\D_{B_{x_0}}$ conditioned on $\x^\perp$,
that is
$\gamma(\x_\bw | \x^\perp) = \gamma(\x_\bw, \x^\perp)/\int \gamma(\x_\bw, \x^\perp) \d \x_\bw\,,$
and $\gamma$ is the density of the $\bx$-marginal of $\D_{B_{\x_0}}$.
Note that, from Lemma~\ref{lem:tsybakov_condition}, it follows that
$\pr[\eta(\vec x)\geq 1/2 -t\ |\ \x_\bw\in[s,2s+s'] ]= O(\frac{A}{s}t^{\alpha})$.
Therefore, $\pr[\eta(\vec x)> 1/2-2\xi\ |\ \x_\bw \in [s, 2s+s'] ]\leq \zeta$,
and it remains to prove that $\pr[\eta^\perp(\vec x^\perp)> 1/2-\xi]$
is at most a small constant multiple of $\zeta$
with high constant probability.

To prove this, let $M(\x)$ be the indicator of the event $\eta(\vec x)>1/2-2\xi$
and consider the random variable
$Y = \int_{x_0}^{x_0 +s'}M(\vec \x_\bw,\x^\perp)\gamma(\vec \x_\bw|\x^\perp)\d \vec \x_\bw$.
Observe that the randomness of $Y$ is over the randomly chosen $x_0$ and $\x^\perp$.
We will first show that the probability that the noise function $\eta^\perp(\x^\perp)$
exceeds $1/2-\xi$ can be bounded above by the probability that the random variable $Y$ exceeds $1/2$,
that is
\begin{align} \label{eq:Y-upper-bound}
\pr_{\x^\perp \sim \D^\perp_\x, x_0}[\eta^\perp(\x^\perp)>1/2-\xi]
\leq \pr_{\x^\perp \sim \D^\perp_\x, x_0}[Y\geq 1 /2 ]\,.
\end{align}
In fact, we show a stronger statement than Equation~\eqref{eq:Y-upper-bound}
that \nnew{holds for} any fixed $x_0 \in [s,2s]$.
To see this, let $\eta'(\x)=1/2 -2\xi (1-M(\x))$ and notice that $\eta'(\x)\geq \eta(\x)$
for every $\x$.  Then, it holds
\begin{align*}
  1/2-\xi < \eta^\perp(\x^\perp) &= \int_{x_0}^{x_0+s'}\eta(\x_\bw,\x^\perp)\gamma(\x_\bw|\x^\perp)\d \x_\bw
\leq\int_{x_0}^{x_0+s'}\eta'(\x_\bw,\x^\perp)\gamma(\x_\bw|\x^\perp)\d \x_\bw
\\
&= 1/2-2\xi + 2\xi\int_{x_0}^{x_0+s'}M(\x_\bw,\x^\perp)\gamma(\x_\bw|\x^\perp)\d \x_\bw
\;,
\end{align*}
which is equivalent to $\int_{x_0}^{x_0 +s'}M(\vec \x_\bw,\x^\perp)\gamma(\vec \x_\bw|\x^\perp)\d \vec \x_\bw = Y \geq 1/2$.

Our next step is to bound from above the probability of the event $Y \geq 1/2$.
For convenience, let $\phi(\x)$ be the density of the initial isotropic log-concave marginal $\D_\x$.
Thus, we have $\gamma(\vec \x_\bw, \x^\perp) = \phi(\vec \x_\bw, \x^\perp)/\pr_{\D}[B_{x_0}]$.
Moreover, set $Q = \min_{x_0 \in [s,2s]} \pr_{\D}[B_{x_0}]$ and recall that from
Properties~\ref{local:anti}, \ref{local:antianti} we have that
for any $x_0 \in [s,2s]$ it holds $\pr_{\D}[B_{\x_0}] = \Theta(s')$, and thus $Q = \Theta(s')$.
We can bound from above the expectation of $Y$, i.e.,
\begin{align*}
\E[Y]
&= \int_{s}^{2s}\frac{1}{s}\int_{x_0}^{x_0 +s'}\int_{\bw^\perp} M(\vec \x_\bw,\x^\perp) \frac{\phi(\vec \x_\bw,\x^\perp)}{\pr_{\D}[B_{x_0}]} \, \d \x^\perp\, \d \vec \x_\bw\,  \d x_0 \\
&\leq \frac{1}{s Q}
  \int_{s}^{2s}
  \int_{x_0}^{x_0 +s'}\int_{\bw^\perp} M(\vec \x_\bw,\x^\perp) \phi(\vec \x_\bw,\x^\perp) \, \d \x^\perp\, \d \vec \x_\bw\,  \d x_0 \\
&\leq \frac{s'}{s Q}\int_{s}^{2 s+s'}\int_{\bw^\perp} M(\vec \x_\bw,\x^\perp) \phi(\vec \x_\bw,\x^\perp) \d \x^\perp\, \d \vec \x_\bw \\
&\leq \frac{s' \pr_{\D}[\x_\bw \in [s,2s+s']]}{s Q} \pr[\eta(\x)>1/2-2\xi| \x_\bw \in [s, 2s+s'] ] \leq \frac{s'}{Q} \zeta = O(\zeta) \;,
\end{align*}
where to get the third inequality we used the fact that for any non-negative function $g(t)$ it holds
$$
\int_{s}^{2s} \int_{u}^{u +s'} g(t) \d t \d u = \int_{0}^{s'} \int_{s+u}^{2s+u} g(t) \d t \d u
\leq \int_{0}^{s'} \int_{s}^{2s+s'} g(t) \d t\d u = s' \int_{s}^{2s+s'} g(t) \d t \,.
$$
The final inequality follows from Properties \ref{local:anti} and \ref{local:antianti}.
By Markov's inequality, we obtain $\pr[Y\geq 1/2] = O(\zeta)$.
Therefore, combining this bound with Equation~\eqref{eq:Y-upper-bound}, we
obtain the probability that $\eta(\x^\perp) > 1/2 -\xi$ is at most
$$
\pr_{\x^\perp \sim \D^\perp_\x, x_0}[\eta^\perp(\x^\perp)>1/2-\xi]
\leq \pr_{\x^\perp \sim \D^\perp_\x, x_0}[Y\geq 1 /2 ]
= O(\zeta) \;.
$$
So, choosing $\zeta$ to be a sufficiently small absolute constant, we get that
$\pr_{(\x^\perp,y) \sim \D^\perp}[\eta^\perp(\x^\perp) \geq 1/2] \leq \xi^3$
and $\pr_{\x^\perp}[\eta^\perp(\x^\perp)>1/2-\xi]\leq 1/3$ with probability at least $99\%$.
This completes the proof.
\end{proof}

We next show how to efficiently decrease the mean of a \nnew{nearly identity covariance} log-concave distribution
and make it arbitrary close to zero. We achieve this by further conditioning.
In particular, we show that we can efficiently find a reweighting
of the conditional distribution on $\vec x^\perp$ such that it is approximately mean zero isotropic.
\nnew{The high-level idea to achieve this is, for some vector $\vec r$, to run rejection sampling,
where $\bx$ is kept with probability $\min(1, \exp(-\langle \vec r, \bx \rangle))$.
The problem is then to find $\vec r$. We do this by finding an approximate stationary point
of an appropriately defined constrained non-convex optimization problem.}

We will use the following standard fact about the convergence of projected stochastic gradient descent (PSGD)
to stationary points of smooth non-convex functions.
\nnew{Consider the constrained optimization setting of minimizing a (differentiable) function $F$
in the set $\cal X$. In this setting, a point $\bx$ is called $\eps$-stationary, $\eps>0$,
if for all $\vec u \in \cal X$ it holds $\dotp{\nabla F(\bx)}{\vec u-\vec x}  \geq -\eps\snorm{2}{\vec u-\vec x}$.
Note that if $\vec x\in \mathrm{int}(\cal X)$, i.e., $\vec x$ is not on the boundary of $\cal X$,
this inequality is equivalent to $\|\nabla F(\bx)\|_2\leq \eps$.}

\begin{fact}[see, e.g.,~\cite{GLZ16}, Corollary 4 and Equations (4.23) and (4.25)]\label{lem:sgd}
Let $\D$ be a distribution supported on $\R^d$.
Let $F:{\cal X}\mapsto \R$ be an $L$-smooth differentiable function on a compact convex set ${\cal X}\subset\R^d$
with diameter $D$. Let $\vec g:{\cal X}\times\R^d\mapsto \R^d$ be such that
$\E_{\bx \sim \D}[\vec g(\vec r, \x)]=\nabla F(\vec r)$ and $\E_{\bx \sim \D}[\snorm{2}{\vec g(\vec r, \x) }^2]\leq \sigma^2$,
for some $\sigma>0$. Then randomized projected SGD uses $T=O(\sigma^3 D^2 L^2/\eps^2)$ samples from $\D$,
runs $\poly(T, d)$ time, and returns a point $\vec r'$ such that with probability at least $2/3$,
$\vec r'$ is an $\eps$-stationary point \nnew{of $F$}.
\end{fact}

We show the following:

\begin{lemma}\label{lem:sampling}
Let $\D$ be an isotropic log-concave distribution on $\R^d$.  Let $\bw \in
\R^d$ be a unit vector and let $B = \{\bx \in \R^d : \dotp{\vec w}{\bx}
\in [a , b] \}$ for $a, b > 0$ smaller than some universal absolute constant.
There exists an algorithm that, given $\gamma>0$ and $\poly(d/\gamma)$
independent samples from $\D_{B}^{\proj_{\vec w^\perp}}$, runs in sample
polynomial time, and returns a vector $\vec r$ such that if $\vec z$ is obtained
from $\D_{B}^{\proj_{\vec w^\perp}}$ by rejection sampling, where a sample
$\x$ is accepted with probability $\min(1,e^{-\dotp{\vec r}{\x}})$, then:
\begin{itemize}
\item A sample is rejected with probability $p$, where $p \in(0,1)$ is an absolute constant.
\item The distribution of $\vec z$ is $(\gamma, O(1))$-isotropic log-concave.
\end{itemize}
\end{lemma}

\begin{proof}

For notational convenience, let $\D'=\D_{B}^{\proj_{\vec w^\perp}}$.
First, we note that for $\vec u$ any unit vector perpendicular to $\vec w$
and any $\vec r$ perpendicular to $\vec w$,
we can apply Fact~\ref{fact:lc-basics} to the projection of $\vec x$ onto the subspace
spanned by $\vec u, \vec w$ and $\vec r$.

We denote by $c$ the constant $c_3$ from Fact~\ref{fact:lc-basics}.
As a result, we have that the distribution of $\dotp{\vec u}{\vec x'}$ will have constant probability
density in a neighborhood of $0$ and will have exponential tails. Furthermore,
this will still hold after rejection sampling with probability $\min(1,e^{-\dotp{\vec r}{\vec x'}})$. This
implies that no matter what $\vec r$ is chosen, $\vec z$ will be approximately isotropic.
Moreover, $\vec z$ will be log-concave automatically,
because the rejection sampling multiplies the pdf by a log-concave
function. Furthermore, the probability of a sample being accepted will be at
least $\pr_{\x'\sim \D'}[\dotp{\vec r}{\vec x'}\leq 0]$, which is at least $c^4$.

It remains to prove the second condition of the lemma.
We let $R$ be a sufficiently large constant and apply projected SGD to find an
approximate stationary point of the non-convex function
$F(\vec r):=\snorm{2}{\vec g(\vec r)}^2$, where $\vec g(\vec r):=\E_{\x'\sim \D'}[\vec x'\min(1,\exp(-\dotp{\vec r}{\vec x'}))]$, in the feasible set  $\{ \vec r \in \R^d : \snorm{2}{\vec r}\leq R \}$.
Note that $\vec g(\vec r)$ is the mean of the distribution of $\vec z$.

We will need the following claim about the approximate stationary points of $F(\vec r)$.
\begin{claim}\label{clm:stationary-point}
Any interior point of the feasible region, i.e., a point $\vec r$ such that $\snorm{2}{\vec r}< R$,
has $\snorm{2}{\nabla F(\vec r)} = \Omega(\snorm{2}{g(\vec r)})$.
Moreover, $F$ has no stationary points on the boundary, i.e., on the set
$\{ \vec r \in \R^d : \snorm{2}{\vec r} =  R \}$.
\end{claim}
\begin{proof}
We show that the Jacobian matrix of $\vec g$ is negative definite.
In particular, for any vector $\vec u\neq \vec 0$, we have
\begin{align*}
\dotp{ \vec u} {\mathrm{Jac}(\vec g(\vec r)) \vec u }
& = \dotp{\vec u } {-\E_{\x'\sim \D'}\left[\vec x' (\vec x' )^\intercal\1\{\dotp{\vec r}{\vec x'}\geq 0\} \exp(-\dotp{\vec r}{\vec x'})\right]\vec u}\\
& = -\E_{\x'\sim \D'}[\1\{\dotp{\vec r}{\vec x'}\geq 0\}\exp(-\dotp{\vec r}{\vec x'})\dotp{\vec u}{\vec x'}^2]\\
&=-\E_{\x'\sim \D'}\left[\1\{\dotp{\vec r}{\vec x'}\geq 0\}\exp(-\dotp{\vec r}{\vec x'})\dotp{\frac{\vec u}{\snorm{2}{\vec u}}}{\vec x'}^2\right]\snorm{2}{\vec u}^2\\
&\leq-\frac{c^2}{4}e^{-c}\E_{\x'\sim \D'}\left[\1\{c \geq\dotp{\vec r}{\vec x'}\geq 0\}\1\left\{\dotp{\frac{\vec u}{\snorm{2}{\vec u}}}{\vec x'}\geq c/2\right\}\right]\snorm{2}{\vec u}^2\\
& \leq -\frac{c^5}{24}e^{-c}\snorm{2}{\vec u}^2\;,
\end{align*}
where we used Fact~\ref{fact:lc-basics} which gives $\dotp{ \vec u} {\mathrm{Jac}(\vec g(\vec r)) \vec u }=-O(\snorm{2}{\vec u}^2)$.
Observe that the gradient of $F$ at $\vec r$ is
$\nabla F(\vec r)= 2\mathrm{Jac}(\vec g(\vec r)) \vec g(\vec r)$, where $\mathrm{Jac}(\vec g(\vec r))$
is the Jacobian of $\vec g$ at point $\vec r$, thus
$\snorm{2}{\nabla F(\vec r)}\geq \dotp{\vec u}{\mathrm{Jac}(\vec g(\vec r))\vec g(\vec r)}/\snorm{2}{\vec u}$
for any vector $\vec u$. Setting $\vec u=\vec g(\vec r)$, we have that  $\snorm{2}{\nabla F(\vec r)} = \Omega(\snorm{2}{g(\vec r)})$.

It remains to prove that there is no stationary point on the boundary.
That is, for a point $\vec r$ with $\snorm{2}{\vec r}=R$, the gradient of $F$ at $\vec r$
is a negative multiple of $\vec r$. It is easy to see that, using Fact~\ref{fact:lc-basics}
for $R$ at least a sufficiently large constant,
$\dotp{\vec g(\vec r)}{ \vec r} <0$.
So, if the gradient of $F$ at $\vec r$ is a negative multiple of $\vec r$, we have that
$$0 < \dotp{\vec g(\vec r)}{ \nabla F(\vec r)}= 2 \dotp{\vec g(\vec r) }{\mathrm{Jac}(\vec g(\vec r)) \vec g(\vec r) }< 0 \;,$$
which is a contradiction.
\end{proof}

As a result, an internal stationary point of $F$ must have $\snorm{2}{\vec g(\vec r)}$ close to
$0$, which would imply that the conditional distribution $\vec z$ with that $\vec r$ has mean less
than $\gamma$. In the following claim, we prove that $F(\vec r)$ is smooth with respect the Euclidean norm.
See Appendix~\ref{ap:logcon} for the proof.

\begin{claim} \label{clm:smooth-logcon}
The function $F(\vec r)$ is $L$-smooth, for some $L=\poly(d)$.
\end{claim}

Thus, by Fact~\ref{lem:sgd}, running Stochastic Gradient Descent for $T=\poly(d/\gamma)$,
we obtain a $\gamma$-stationary point, assuming we have an unbiased estimator for the gradient of $F$.
Note that by taking two independent samples $\x^{(1)}$ and $\x^{(2)}$ from $\D'$
and setting $\hat{\vec g}(\vec r,\vec x)=\vec x\min(1,\exp(-\dotp{\vec r}{\vec x}))$,
the quantity $2\mathrm{Jac}(\vec {\hat g}(\vec r,\x^{(1)}))\vec {\hat g}(\vec r,\x^{(2)})$
is an unbiased estimator for $\nabla F(\vec r)$.
This completes our proof.
\end{proof}

\begin{algorithm}[H]
\caption{Computing a Good Initialization Vector} \label{alg:find-nontrivial-correlation}
\begin{algorithmic}[1]
\Procedure{WarmStart}{$(\tsya, \tsyb), \eps, \vec w, \D$}\\
\textbf{Input:} Samples from an $O(\gamma, O(1))$-isotropic log-concave distribution that satisfies
the $(\tsyb,\tsya)$-Tsybakov noise condition, and a unit vector $\vec w$ such that $\theta(\vec w, \wstar) = \Theta(\eps)$.\\
\textbf{Output:} A vector $\vec v$ such that $\dotp{\vec v} {(\vec w^\ast)^{\perp_\bw }}\geq \nnew{ (\alpha \eps/A)^{O(1/\alpha)}}$.\\

\State $s\gets \Theta(\alpha \eps /\log(A\log(A)/(\alpha \eps)))$, \nnew{$\xi\gets (\Theta (A/s))^{1/\alpha}$, $s'=\Theta(\xi^3 s \eps)$}
\State \nnew{$N \gets  \poly(d) \cdot (A/(\alpha\eps))^{O(1/\alpha)}$}
\State  Let $x_0$ be a uniform random number on $[s,2s]$.
\State Let $\D'$ denote $\D$ conditioned on $\dotp{\vec w}{\vec x}\in[x_0,x_0 +s']$ and projected onto $\vec w^\perp$.
\State $\widehat{\D}\gets$ \textsc{MakeIsotropic}$(\D',1/\log(1/\xi),N)$
\State  $\bar{\vec {\x}} \gets \E_{\bx\sim \widehat{\D}_\x}[\x]$; $\bar{\vec X}\gets \E_{\bx\sim \widehat{\D}_\x}[\x\x^\intercal]$
\State Normalize all samples in $\widehat{\D}$ with $\bar{\vec {\x}} $ and $\bar{\vec X}$
\State $\vec T'_1 \gets \E_{(\vec x,y)\sim \widehat{\D}}[y\vec x]$ and $\vec T'_2\gets \E_{(\vec x,y)\sim \widehat{\D}}[y(\vec x\vec x^\intercal-\vec I)]$
\State Let $V$ be the subspace spanned by $\vec T'_1$ and the eigenvectors of $\vec T'_2$ whose eigenvalues have absolute value at least $\xi$.
\State $\textbf{return} $ a random vector in $V$.
\EndProcedure
	\end{algorithmic}
\end{algorithm}

\begin{algorithm}[H]
\caption{Putting the Distribution in Nearly-Isotropic Position} \label{alg:make_isotropic}
\begin{algorithmic}[1]
\Procedure{MakeIsotropic}{$\D_B^{\proj_{\vec w^\perp}}$, $\gamma$, $N$} \\
\textbf{Input:} Samples from the log-concave distribution $\D_B^{\proj_{\bw^\perp}}$, i.e., the log-concave distribution $\D$ conditioned on a
band $B = \{\bx: \dotp{\bw}{\bx} \in [a, b]\}$ and then projected onto $\bw^\perp$. \\
\textbf{Output:} \nnew{$N $ i.i.d. samples} from a $(\gamma,O(1))$-isotropic log-concave distribution obtained
from $\D_B^{\proj_{\bw^\perp}}$ \nnew{by rejection sampling.}\\

\State Let $F(\vec r)=\snorm{2}{\E_{\x\sim \D_B^{\proj_{\bw^\perp}}} [\vec x \min(1,\exp(-\dotp{\vec r}{\vec x}))]}^2$
\State Runs SGD on $F$ to obtain a $\gamma$-stationary point $\vec r'$. \Comment{Takes $\poly(d/\gamma)$ time.}
\State $S\gets \emptyset$
\State  \textbf{while} $|S|\leq N$
 \State  \qquad Draw sample $(\bx,y)$ from $\D_B^{\proj_{\bw^\perp}}$.
\State    \qquad $S\gets S\cup\{(\x,y)\}$ with probability $\min(1, \exp(-\dotp{\vec r'}{\bx}))$.
\State  Let $\D_S$ be the uniform distribution from $S$. \\

\textbf{return} the sample $\D_S$.
\EndProcedure
\end{algorithmic}
\end{algorithm}

We are now ready to prove Theorem \ref{thm:init-lc}.

\begin{proof}[Proof of Theorem \ref{thm:init-lc}]
Using the condition $\theta(\vec w, \wstar) = \Theta(\eps)$, we can assume that
$\vec w^\ast= \lambda_1 \vec w+ \lambda_2 (\vec w^\ast)^{\perp_\bw }$, where
$|\lambda_1| = 1- \Theta(\eps)$, and $\lambda_2 = \Theta(\eps)$.  If
$\dotp{\vec w}{\vec w^\ast}=0$, then we can directly apply
Proposition~\ref{prop:nontriavial_angle} to obtain a vector with non-trivial
correlation. For the general case, we show how we can construct a distribution
that satisfies the conditions of Proposition~\ref{prop:nontriavial_angle}.

Let $s$ be a sufficiently small multiple of $\alpha \eps /\log(A\log(A)/(\alpha
\eps))$, $\xi = (\Theta(s/A))^{1/\alpha}$, and let $s' = \xi^3 s \eps$.
Finally, let $x_0$ be a uniform random number in $[s, 2s]$. Consider the
conditional distribution on the random band $B_{x_0} = \{ \dotp{\bw}{\x} \in
  [x_0, x_0 + s']$ and projected onto $\vec w^\perp$, i.e.,
  $\D_{B_{x_0}}^{\proj_{\vec w^\perp}} : = \D^\perp$.

Set $\x^\perp = \proj_{\bw^\perp} \x$, $f^\perp(\vec x^\perp) = \sgn\left(\dotp{\vec x^\perp} {(\vec
w^{\ast})^{\perp_\bw}}+\frac{\lambda_1 x_0}{\lambda_2}\right)$, and
$\eta^\perp(\x^\perp)= \pr_{(\bx^\perp, y) \sim \D^\perp}[y\neq f^\perp(\vec z)|\vec
z=\x^\perp]$. Using Lemma~\ref{lem:noise-random-band}, we get that $\D^\perp$ is
$(O(1),O(1))$-isotropic and with high probability it holds
$\pr_{\x^\perp \sim \D_\x^\perp}[\eta^\perp(\x^\perp) \leq 1/2 - \xi  ] \geq 2/3$ and
$\pr_{\x^\perp \sim \D_\x^\perp}[\eta^\perp(\x^\perp) \geq 1/2] \leq \xi^3$.

At this point, we have that $\D^\perp$ is approximately isotropic, but may be
relatively far from mean $0$ (the mean can be at constant distance from the
origin, whereas we need it to be roughly $1/\log(1/\xi)$).  To overcome this
issue, we apply Lemma~\ref{lem:sampling}.  We define $\bar{\D}$ to be the
distribution of $\vec z$ that is produced according to Lemma~\ref{lem:sampling}
with $\gamma$ a small multiple of ${1/\log(1/\xi)}$, and consider the
distribution on $\vec z$ and $y$. Notice that $y$ is a noisy version of
$f^\perp(\vec x)$ (\nnew{with noise rate} $\eta^\perp(\vec x^\perp)$), because
rejection sampling does not increase the noise rate. Moreover, the mean of
$\dotp{\vec z} {(\vec w^{\ast})^{\perp_\bw}}+\frac{\lambda_1 x_0}{\lambda_2}$
is at most $\gamma+O(s/\eps)$, which is a sufficiently small multiple of {$1/\log(1/\xi)$}.
This means that we can apply Proposition \ref{prop:nontriavial_angle} to the distribution
on $(\vec z,y)$, yielding our final result.
\end{proof}

\subsection{Proof of Theorem~\ref{thm:cert-lc}} \label{ssec:cert-lc-proof}

Using Theorem~\ref{thm:init-lc}, we can prove
Theorem~\ref{thm:cert-lc}. The proof is similar to the proof of Theorem~\ref{thm:cert-wb},
but we additionally need to guess how far the current guess $\bw$  is from $\bw^\ast$.

\begin{proof}[Proof of Theorem~\ref{thm:cert-lc}]
First, we guess a value $\eps$ such that $\snorm{2}{\vec w-\vec w^\ast}=\Theta(\eps)$,
where $\eps=\nnew{\Omega}(\theta)$. From Proposition~\ref{prop:reduction}, for
$\rho=\nnew{O( \theta (\alpha \eps/A)^{O(1/\alpha)})}$,
the distribution $\Dperp$ is $(2, \Omega( \rho),1/\rho ,O( 1/ \rho\log(1/\rho))$-well-behaved
and satisfies the $(\tsyb,O(A / \rho))$-Tsybakov noise condition, where we used
(from Fact~\ref{fact:lc-basics}) that the values $L, R$ are absolute constants. Using
Theorem~\ref{thm:init-lc}, a random unit vector $\vec v\in \R^{d}$ with constant
probability $\delta_1$ satisfies
$\dotp{\vec v}{{(\vec w^{\ast})^{\perp_{\vec w}}}}\geq  (\alpha \eps/A)^{O(1/\alpha)}$.
We call this event $\cal E$. Conditioning on the event $\cal E$,
from Proposition~\ref{prop:cert-wb},
using $\frac{\conb^4}{\theta^2}\left(\frac{A}{\theta\alpha}\right)^{O(1/\tsyb^2)}\log(1/\delta)$
samples, with probability $1-\delta$, we get a $(\vec v', R, t_0)$ such that
$$\E_{(\vec x,y) \sim \Dperp} \left[\1[- R \leq \dotp{\vec v'}{\vec x} \leq - t_0]  y \right]
\leq - \left( \theta\alpha/\tsya \right)^{O(1/\alpha^2)}/\beta\;.$$
Using Lemma~\ref{lem:certificate_reduction}, we get that
$$\E_{(\vec x,y) \sim  \D} \left[T_{\vec w}(\bx)\dotp{\vec x}{\vec w}  y \right]\leq- \left( \theta \alpha/\tsya\right)^{O(1/\alpha^2)}/\beta \;.$$
Conditioning on the event $ {\cal E}^c$, where ${\cal E}^c$ is the complement of $\cal E$,
Algorithm~\ref{alg:find_certificate} either returns a certificate or returns nothing.
Thus, by taking $k=O(\log(1/\delta))$ random vectors, we get that the probability that event ${\cal E}^c$
happens is at most $(1-\delta_1)^k\leq e^{-\delta_1 k}$. Thus, by taking $O(\log 1/\delta)$ random vectors and
running Algorithm~\ref{alg:find_certificate} with confidence $\delta/\log(1/\delta)$, we get a certificate
with probability $1-2\delta$. Moreover, the number of samples needed to construct the empirical
distribution is $\left(\frac{ \tsya}{\theta \alpha }\right)^{O(1/\tsyb^2)}\log(1/\delta)$.
Finally, to guess the value of $\eps$, it suffices to run the algorithm for the values $\theta,2 \theta,\ldots ,1$
which will increase the complexity by a $\log(1/\theta)$ factor. This completes the proof of Theorem~\ref{thm:cert-lc}.
\end{proof}

\section{Learning a Near-Optimal Halfspace via Online Convex Optimization}\label{sec:ogd}

In this section we present a black-box approach that uses our certificate algorithms
from the previous sections to learn halfspaces in the presence of Tsybakov noise.
In more detail, we provide a generic result showing that one can apply a certificate oracle
in a black-box manner combined with online gradient descent to learn the unknown halfspace.
We note that an essentially identical approach, with slightly different formalism,
was given in \cite{DKTZ20b}.

Using the aforementioned approach, we establish the two main algorithmic results
of this paper.

\begin{theorem}[Learning Tsybakov Halfspaces under Well-Behaved Distributions] \label{thm:pac-wb}
Let $\D$ be a $(3, L, R, U, \beta)$-well-behaved isotropic distribution on $\R^{d} \times \{\pm 1\}$ that
satisfies the $(\tsyb,\tsya)$-Tsybakov noise condition with respect to an unknown halfspace
$f(\bx) = \sgn(\dotp{\vec w^{\ast}}{\bx} )$. There exists an algorithm that  draws
$N=   \beta^4 \left(\frac{d \, U \, \tsya}{R L \, \eps}\right)^{O(1/\tsyb)} \log\left(1/\delta\right)$ samples from $\D$,
runs in $\poly(N,d)$ time, and computes a vector $\wh{\vec w}$ such that, with probability $1-\delta$,
we have that $\err_{0-1}^{\D_{\bx}}(h_{\wh{\bw}}, f) \leq \eps$.
\end{theorem}

For the important special case of log-concave distributions on examples,
we give a more efficient learning algorithm.

\begin{theorem}[Learning Tsybakov Halfspaces under Log-concave Distributions] \label{thm:pac-lc}
Let $\D$ be a distribution on $\R^{d} \times \{\pm 1\}$ that satisfies the $(\tsyb,\tsya)$-Tsybakov noise condition
with respect to an unknown halfspace $f(\bx) = \sgn(\dotp{\vec w^{\ast}}{\bx} )$ and is such that
$\D_{\bx}$ is isotropic log-concave. There exists an algorithm that  draws
$N= \poly(d) \cdot\left(\frac{\tsya}{ \eps}\right)^{O(1/\tsyb^2)} \log\left(1/\delta\right)$ samples from $\D$,
runs in $\poly(N,d)$ time, and computes a vector $\wh{\vec w}$ such that,
with probability $1-\delta$, we have that $\err_{0-1}^{\D_{\bx}}(h_{\wh{\bw}}, f) \leq \eps$.
\end{theorem}

To formally describe the approach of this section, we require the notion of a
{\em certificate oracle}. A certificate oracle is an algorithm that, given a
candidate weight vector $\bw$ and an accuracy parameter $\rho> 0$, it
returns a certifying function $T(\bx)$. Recall that a certifying function is a non-negative function
that satisfies
$
\E_{(\bx, y) \sim \D}[T(\bx) y \dotp{\bx}{\vec w}] \leq - \rho
$
for some $\rho > 0$.  We have already described how to efficiently implement
such an oracle in Section~\ref{sec:cert-wb}.

\begin{definition}[Certificate Oracle] \label{def:certificate_oracle}
Let $\D$ be a distribution on $\R^d \times \{ \pm 1\}$that satisfies the $(\tsyb,\tsya)$-Tsybakov noise
condition with respect to an unknown halfspace $f(\bx) = \sgn(\dotp{\vec w^{\ast}}{\bx})$.
For  a decreasing function $\rho(\cdot): \R_+  \mapsto \R_+$, we
define $\cal{C}(\vec w, \theta, \delta)$ to be the following $\rho$-certificate
oracle:  For any unit vector $\vec w$ and $\theta > 0$, if $\theta(\vec w,\vec w^\ast)\geq \theta$,
then a call to $\cal{C}(\vec w, \theta, \delta)$, with probability at least $1-\delta$,
returns a function $T(\bx)$, with $\snorm{\infty}{T} \leq 1$ such that
$$\E_{(\bx, y) \sim \D}[T(\bx) y \dotp{\bx}{\vec w}] \leq - \rho(\theta) \,,$$
and with probability at most $\delta$ returns ``FAIL''.
\end{definition}

\begin{remark}\label{rem:one-sided}
We note that the above oracle provides a ``one-sided'' guarantee in the following sense.
When the candidate vector $\bw$ satisfies $\theta(\bw, \wstar) \geq \theta$,
the oracle is required to return a certifying function $T$ with high probability.
But it may also return such a function when $\theta(\bw, \wstar) \leq \theta$.
In other words, the oracle is not required to output ``FAIL" with high probability when
$\bw$ is nearly parallel to $\wstar$. We show that an one-sided oracle of non-optimality
suffices for our purposes.
\end{remark}

\begin{remark}\label{rem:oracle-props}
By Fact~\ref{obs:optimal_condition}, the optimal halfspace
$\wstar$ satisfies $\E_{(\bx, y) \sim \D}[T(\bx) \, y \dotp{\bx}{\wstar}] \geq 0$
for any non-negative function $T$. Therefore, as $\vec w$ approaches $\wstar$,
we have that
$$\lim_{\theta(\vec w, \wstar) \to 0} \inf_{T: \snorm{\infty}{T} \leq 1} \E_{(\bx, y) \sim \D}[T(\bx) \, y \dotp{\bx}{\vec w}] = 0,$$
where $\snorm{\infty}{T}$ is the $\ell_\infty$ norm for functions, i.e., $\snorm{\infty}{T} = \sup_{\bx \in \R^d} |T(\bx)|$.
That is, $\lim_{\theta \to 0} \rho(\theta) = 0$ and it is natural that the non-negative function $\rho(\theta)$
is a decreasing function of the (lower bound on the) angle between $\bw$ and $\wstar$.
Intuitively, the closer $\bw$ is to $\wstar$, the harder it is to find a certifying
function $T$ that makes $\E_{(\bx, y) \sim \D}[T(\bx) \, y \dotp{\bx}{\bw}]$ sufficiently negative.
Moreover, if our goal is to estimate the vector $\wstar$ within angle $\eps$,
we can always give the oracle this worst-case target angle, i.e., $\theta = \eps$.
Finally, notice that when the distribution $\D$ is isotropic, we have $\rho(\theta) \leq 1$,
as follows from $\snorm{\infty}{T} \leq 1$ and the Cauchy-Schwarz inequality.
\end{remark}

Given a certificate oracle, the following result shows we can efficiently approximate
the optimal halfspace using projected online gradient descent.

\begin{proposition}[Certificate-Based Optimization] \label{pro:certificate_optimization}
Let $\D$ be a $(3, L, R, \beta)$-well-behaved isotropic distribution on $\R^{d} \times \{\pm 1\}$
that satisfies the $(\tsyb,\tsya)$-Tsybakov noise condition with respect to an unknown halfspace
$f(\bx) = \sgn(\dotp{\vec w^{\ast}}{\bx})$, and let $\cal C$ be  a $\rho$-certificate oracle.
There exists an algorithm that makes at most
$T = \frac{1}{\rho^2(\eps)} \frac 1\alpha\left(\frac{A}{R\;L}\right)^{O(1/\tsyb)}$
calls to $\cal{C}(\cdot)$, draws
$N = d   \frac{T\beta^2}{ \rho^2(\eps) }\log\left(\frac{dT}{\delta \rho(\eps)}\right)$ samples from $\D$,
runs in time \nnew{$\poly(T, N, d)$}, and computes a weight vector $\widehat{\vec w}$ such that
with probability $1-\delta$ we have that $\theta(\widehat{\vec w}, \wstar) \leq \eps$ .
\end{proposition}

The algorithm establishing Proposition~\ref{pro:certificate_optimization} is given in pseudocode
in Algorithm~\ref{alg:OPGD-oracle}. In the remaining part of this section, we provide a proof sketch of
Proposition~\ref{pro:certificate_optimization}. The full argument is given in Appendix~\ref{ap:ogd}.

\begin{proof}[Proof Sketch.]
The main idea of the algorithm is to provide a sequence of adaptively chosen
convex loss functions to an Online Convex Optimization algorithm,
for example Online Gradient Descent (OGD). In more detail, we construct these loss functions
using our certificate oracle $\cal{C}$.  At round $t$, we call the certificate oracle
to obtain a certifying function $T(\vec x)$ and set
$$\ell_t(\vec w) = - \dotp{ \E_{(\vec x,y) \sim \D} \left[(T(\vec x) + \lambda)  y \bx \right] }{\vec w} \,,
$$
where $\lambda > 0$ acts similarly to a
regularizer.  The term $\lambda \dotp{\E_{(\bx,y)\sim \D}[y \bx]}{\bw}$
prevents the trivial vector $\vec w = \vec 0$ from being a valid solution
(in the sense of one that minimizes regret, see also the full proof in Appendix~\ref{ap:ogd}).

The crucial property of the above sequence of loss functions is that they are positive and bounded away from $0$
when $\bw$ is far from $\wstar$. Their value will always be greater than (roughly) $\rho(\eps)$,
given the guarantee of our certificate oracle from Definition~\ref{def:certificate_oracle} for $\theta = \eps$
and assuming that the regularizer $\lambda$ is sufficiently small.

We then provide this convex loss function to the OGD
algorithm that updates the guess according to the gradient of $\ell_t(\vec w)$.
Our analysis follows from the regret guarantee of OGD.
Since we provide convex (and in particular linear) loss functions to OGD, we
know the average regret will converge to $0$ as $T \to \infty$ with a
convergence rate roughly $O(1/\sqrt{T})$.  This means that the oracle can only
succeed in returning certifying functions for a bounded number of rounds, since
every time the oracle succeeds, OGD suffers loss of at least $\rho(\eps)$.
Therefore, after roughly $1/\rho(\eps)^2$ rounds the regret will be so small
that for at least one round the certificate oracle must have failed.
Our algorithm then stops and returns the halfspace of that iteration.
Even though our certificate is ``one-sided", we know that
the probability that it failed with $\theta(\bw, \wstar)$ being larger than
$\eps$ is very small, which implies that we have indeed found a vector $\bw$
very close to $\wstar$.
\end{proof}

\begin{algorithm}[H]
\caption{Learning Halfspaces with Tsybakov Noise using a $\rho$-certificate oracle $\cal C$} \label{alg:OPGD-oracle}
\begin{algorithmic}[1]
\Procedure{ALG}{$\eps,\delta,\D ,\cal C$}
\Comment{$\eps$: accuracy, $\delta$: confidence} \\
\textbf{Input:} $\D$ is a $(3, L, R, \beta)$-well-behaved distribution that satisfies the $(\tsyb,\tsya)$-Tsybakov noise condition, and
$\cal C$ is a $\rho$-certificate oracle.\\
\textbf{Output:} A vector $\widehat{\vec w}$ such that $\err_{0-1}^{\D_{\bx}}(h_{\wh{\bw}}, f) \leq \eps$ with probability at least $1-\delta$.
\State ${\vec w}^{(0)} \gets \vec e_1$
\State $T \gets   \frac {1}{\rho(\eps)^2\alpha}\left(\frac{A}{R\;L}\right)^{O(1/\tsyb)}$
\State Draw $N= \tilde{O}\left(d\cdot  \frac{T\beta^2}{ \rho^2(\eps) }\log\left(\frac{1}{\delta }\right) \right)$ samples from $\D$
to form the empirical distribution $\widehat{\D}$
\State \textbf{for} $t = 1, \dots, T$ \textbf{do}
\State\qquad $\eta_t \gets 1/(\sqrt{t} +\rho(\eps))$
\State   \qquad \textbf{if} $ {\vec w}^{(t-1)}= \vec 0$ \textbf{then}
\State \quad \qquad Set $\hat{\ell_t}(\vec w)\gets \dotp{\vec w}{-\E_{(\vec x,y) \sim \widehat{\D}}\left[\frac {\rho(\eps)} {2} y \bx \right] }$
\State\qquad   \quad ${\vec w}^{(t)} \gets \Pi_{\cal B}\left({\vec w}^{(t-1)} - \eta_t \nabla_{\vec w} \hat{\ell_t}\left( {\vec w}^{(t-1)}\right)\right)$
\State \qquad  \textbf{else}
\State\qquad  \quad $\textsc{Ans} \gets {\cal{C}}({\vec w}^{(t-1)}/\snorm{2}{{\vec w}^{(t-1)}}, \eps, \delta/T)$ \label{alg:cerficate}
\vspace{2mm}
\State\qquad \quad  \textbf{if} $ \textsc{Ans} = \textsc{FAIL}$
\textbf{then}
\State \qquad \qquad \quad \textbf{return} ${\vec w}^{(t-1)}$\label{alg:final_return}
\State   \qquad $ T_{{\vec w}^{(t)}}(\x) \gets \textsc{Ans}$
\State \qquad Set $\hat{\ell_t}(\vec w)\gets \dotp{\vec w}{-\E_{(\vec x,y) \sim \widehat{\D}}\left[\left(T_{{\vec w}^{(t)}}(\vec x) +\frac {\rho(\eps)}{2} \right) y \bx \right] }$
\State\qquad  ${\vec w}^{(t)} \gets \Pi_{\cal B}\left({\vec w}^{(t-1)} - \eta_t \nabla_{\vec w} \hat{\ell_t}\left( {\vec w}^{(t-1)}\right)\right)$ \label{alg:OPGDstep}
\Comment{${\cal B}=\{\vec x \in \R^d : \snorm{2}{\vec x}\leq 1\}$}
\EndProcedure
\end{algorithmic}
\end{algorithm}

Given  Proposition~\ref{pro:certificate_optimization}, it is straightforward
to prove our main results.  Here we give the proof for the case of log-concave densities
and provide a similar argument for well-behaved distributions in Appendix
~\ref{ap:ogd}.

\begin{proof}[Proof of Theorem~\ref{thm:pac-lc}]
First, we require a $\rho$-certificate oracle for log-concave distributions.
The algorithm of Theorem~\ref{thm:cert-lc} returns a function $T_\bw$
such that $\E_{(\vec x,y) \sim \D} \left[T_\bw(\bx) y \dotp{\vec w}{\bx} \right] \leq - \left(\theta /\tsya\right)^{O(1/\alpha^2)}$.
From the definition of $T_{\bw}$ (i.e., Equation~\eqref{eq:certificate_form}), it is clear that
$\snorm{\infty}{T_\bw}\leq \frac{1}{\min_{\x \in B} |\dotp{\vec w}{\vec x} | } \leq \left(\frac{\log A}{\alpha \theta}\right)^{O(1/\alpha)}$, where $B$ is the band from Equation~\eqref{eq:certificate_form}.
Note that the function $T_\bw/\snorm{\infty}{T_\bw}$ satisfies the conditions of the $\rho$-certificate oracle.
Thus, by scaling the output of the algorithm of Theorem~\ref{thm:cert-lc},
we obtain a $\left(\theta \alpha/\tsya\right)^{O(1/\alpha^2)}$-certificate oracle.
From Proposition~\ref{pro:certificate_optimization}, this gives us an algorithm that returns a vector
$\widehat{\bw}$ such that   $\theta(\widehat{\vec w}, \wstar) \leq \frac{\eps}{\log^2(1/\eps)}$ with probability $1-\delta$. Using
the fact that for log-concave distributions
$\err_{0-1}^{\D_{\bx}}(h_{\wh{\bw}}, f) \leq O\left(\log^2(1/\eps)\theta(\wh{\vec w},\vec w^\ast)\right)+\eps$ (Claim~\ref{lem:angle_zero_one})
the result follows.
\end{proof}

\bibliographystyle{alpha}
\bibliography{allrefs}
\clearpage
\appendix

\section{Omitted Proofs from Section~\ref{sec:cert-wb}}\label{ap:cert}

\subsection{Proof of Claim~\ref{clm:first}}
 \begin{proof} [Proof of Claim~\ref{clm:first}]
To bound from below the expectation $I_{2,2}$, we use the fact that the distribution is $(2, L, R, \beta)$-well-behaved.
For $I_{1,2}^{R/2}$, we have
\begin{align*}
I_{2,2}=\E_{\vec x \sim \D_\bx}
\left[\1_{B_3^{R/2}}(\bx)  \nr(\bx)  |\bx_1| \right]&=\int_{B_3^{R/2}}|\bx_1|\nr(\bx)\gamma(\bx)\d \bx\\
&\geq\int_{0}^{R/\sqrt{2}}\int_{R/(2\sqrt{2})}^{R/\sqrt{2}}\bx_1\nr(\bx_1,\bx_2)\gamma(\bx_1,\bx_2)\d \bx_1\bx_2\\
&\geq \frac{R}{2\sqrt{2}}\int_{R/2}^{R/\sqrt{2}}\int_{R/(2\sqrt{2})}^{R/\sqrt{2}}\nr(\bx_1,\bx_2)\gamma(\bx_1,\bx_2)\d \bx_1\bx_2 \\
&\geq\frac{R}{2\sqrt{2}} C_\alpha^A \left(\int_{R/2}^{R/\sqrt{2}}\int_{R/(2\sqrt{2})}^{R/\sqrt{2}}\gamma(\bx_1,\bx_2)\d \bx_1\bx_2\right)^{1/\alpha} \\
&\geq \frac{R}{4} C_\alpha^A \left(\frac{R^2 L}{16}\right)^{1/\alpha}\;,
\end{align*}
where we used Lemma~\ref{lem:tsybakov_expectation}, and we bound from below the integral by a smaller square region,
i.e., $[R/2,R/\sqrt{2}]\times[R/(2\sqrt{2}),R/\sqrt{2}]$. For $I_{2,2}$, we have
\begin{align*}
I_{1,2}^{R/2} = \E_{\vec x \sim \D_\bx} \left[ \1_{B_3^{R/2}}(\bx)  \nr(\bx) \right]&=\int_{B_3^{R/2}}\nr(\bx)\gamma(\bx)\d \bx \\
&\geq C_\alpha^A \left(\int_{B_3^{R/2}}\gamma(\bx)\d \bx \right)^{1/\alpha} \\
&\geq C_\alpha^A \left(\int_{0}^{R/\sqrt{2}}\int_{R/(2\sqrt{2})}^{R/\sqrt{2}}\gamma(\bx_1,\bx_2)\d \bx_1\bx_2\right)^{1/\alpha} \\
&\geq C_\alpha^A \left(\frac{R^2 L}{4}\right)^{1/\alpha}\;,
\end{align*}
where we used Lemma~\ref{lem:tsybakov_expectation}. Thus,
$$I_{1,2}^{R/2} \geq C_\alpha^A \left(\frac{R^2 L}{4}\right)^{1/\alpha}= (RL/A)^{O(1/\alpha)}~~ \text{ and } ~~ I_{2,2} \geq  (RL/A)^{O(1/\alpha)} \,.$$
This completes the proof of Claim~\ref{clm:first}.
\end{proof}

\subsection{Proof of Claim~\ref{clm:sec}}
\begin{proof}[Proof of Claim~\ref{clm:sec}]
Recall that $\xi(\bx_2) = \bx_2/\tan \theta + b/\sin \theta$.
We have that
$$I_1^{R/2} \leq  \E_{\vec x \sim \D_\bx} \left[ \1_{B_1^{R/2}}(\bx)  \nr(\bx) \right]
    -I_{1,2}^{R/2} \leq  \E_{\vec x \sim \D_\bx}\left[ \1_{B_1^{R/2}}(\bx)  \nr(\bx) \right]-\Gamma/2\;.$$
We can bound from below the first term as follows
\begin{align*}
\E_{\vec x \sim \D_\bx} \left[ \1_{B_1^{R/2}}(\bx)  \nr(\bx) \right]
&\leq \int_{-R}^{-R/2} \int_{-\infty}^{\xi(\bx_2)} \gamma(\bx_1, \bx_2) \d \bx_1 \d \bx_2
\leq \int_{-R}^{-R/2} \int_{-\infty}^{\xi(-R)} \gamma(\bx_1, \bx_2) \d \bx_1 \d \bx_2 \\
&\leq  \pr[\bx_2\geq |\xi(-R)|]\leq \exp(1- |\xi(-R)|/\beta) \;.
\end{align*}
Note that $|\xi(-R)|=(R\cos\theta -b) /\sin\theta\geq 3b/\sin\theta$, thus using the assumption $\theta<b \Gamma /(4\beta ) $, we obtain
$\exp(1-|\xi(-R)|/\beta)\leq \Gamma/4$, and therefore $I_1^{R/2}\leq-\Gamma/4\;,$
completing the proof of Claim~\ref{clm:sec}.
\end{proof}

\subsection{Proof of Lemma~\ref{lem:algorithm_function_g}}

We start with a useful fact about the sub-exponential random variables.
\begin{fact}[see, e.g., Corollary of Proposition 2.7.1 in~\cite{Ver18}] \label{fct:tails}
Let $X$ be sub-exponential random variable with tail parameter $\beta$. For any function
$f:\R\mapsto \R$, the random variable $X f(X)- \E[X f(X)]$ is zero mean sub-exponential with tail parameter $O(\beta\sup|f|)$.
\end{fact}

Using Fact~\ref{fct:tails}, we can bound from above the sample complexity
needed to construct $\widehat{D}$.

\begin{proof}[Proof of Lemma~\ref{lem:algorithm_function_g}]
Let $\vec{ \hat g}= \frac{1}{N}\sum_{i=1}^N \1_{B^{R/2}}\left(\sample{\bx}{i}\right) \sample{y}{i}  \sample{\bx}{i}$.
For any $\vec u\in \R^{d}$, we have that
\begin{align}\label{eq:bound_norm}
|\dotp{\vec u}{\vec g}|&\leq\E_{\vec x \sim \D_\bx} [|\dotp{\vec u}{\vec x}| ]
= \int_{0}^{\infty}\pr_{\bx\sim \D_\bx}[|\dotp{\vec u}{\vec x}|\geq t ] \d t
\leq \int_{0}^{\infty}\exp(1-t/\beta)\d t=e\beta\;,
\end{align}
thus $\snorm{2}{\vec g}\leq e\beta$.
Next we prove that the random variable $X=\1_{B^{R/2}}(\bx) y \bx- \vec{ g}$ is zero-mean
with sub-exponential tails. First, we clearly have that $\E[X]=0$.
Using Fact~\ref{fct:tails}, it follows that $X$ is sub-exponential with tail parameter $\beta'=O(\beta)$.
We will now use the following Bernstein-type inequality.
\begin{fact}\label{lem:bernstein}
Let $X_1, X_2, \ldots, X_N$ be independent zero-mean sub-exponential random variables with tail
parameter $\conb\geq 1$. There exists an absolute constant $c>0$ such that for every $\eps>0$ we have
$$\pr\left[\left|\sum_{i=1}^N X_i \right|\geq \eps N\right] \leq 2 \exp \left(-c N\eps^2 /\conb^2\right)\;.$$
\end{fact}
Using Fact~\ref{lem:bernstein}, we have that for every $1\leq j\leq d$
it holds
$$\pr\left[\left| \vec{ \hat g_j}- \vec{ g_j} \right| \geq \eps/\sqrt{d}\right] \leq
2 \exp \left(-cN\eps^2 /\left(d\conb'^2\right)\right) \;.$$
Thus, taking $N= O\left((d\beta^2/\eps^2)\log(d/\delta)\right)$, we get that
$\snorm{2}{\vec{ \hat g} - \vec{ g} } \leq \eps$ with probability $1-\delta$.
For the second statement, using the triangle inequality and Equation~\eqref{eq:bound_norm}
the result follows.
\end{proof}

\subsection{Proof of Lemma~\ref{lem:check_certificate}}
The proof requires a couple of known probabilistic facts. The first one is
the bounded-difference inequality.

\begin{fact}[see, e.g., Theorem 2.2 of \cite{DL:01}]\label{lem:bounded_difference}
Let $X_1,\ldots, X_d \in \cal X$ be independent random variables and
let $f:{\cal X}^d \mapsto \R$. Let $c_1,\ldots,c_d$ satisfy
$$\sup_{x_1,\ldots, x_d, x_i'} \left| f(x_1,\ldots, x_i,\ldots, x_d) - f(x_1,\ldots, x_i',\ldots, x_d) \right|\leq c_i$$
for $i\in[d]$. Then we have that
$\pr\left[f(X) - \E[f(X)]\geq  t \right] \leq \exp\bigg(- 2t^2 /\sum_{i=1}^d c_i^2\bigg) \;.$
\end{fact}

We additionally require the symmetrization of the empirical distribution.

\begin{fact}[see, e.g., Exercise 8.3.24 of \cite{Ver18}]\label{lem:sym}
Let $\cal F$ be a class of measurable real-valued functions.
Let $X_1,\ldots, X_N$ be $N$ i.i.d. samples from a distribution $\D$.
Then
$$\E\left[ \sup_{f\in \cal F}\left|\frac{1}{N} \sum_{i=1}^N f(X_i)- \E[f(X)]\right|\right]
\leq 2 \E\left[ \sup_{f\in \cal F} \left| \frac{1}{N}\sum_{i=1}^N \eps_i f(X_i)\right|\right]\;,$$
where the $\eps_i$'s are independent Rademacher random variables.
\end{fact}

The last fact we need connects the symmetrization with the VC dimension.

\begin{definition}[VC dimension]
A collection of sets $\mathcal{F}$ is said to \emph{shatter} a set $S$ if for all
$S' \subseteq S$, there is an $F \in \mathcal{F}$ so that $F \cap S = S'$.
The VC dimension of $\mathcal{F}$, denoted $\textsc{VC}(\mathcal{F})$,
is the largest $n$ for which there exists an $S$ with $|S| = n$ such that $\mathcal{F}$ shatters $S$.
\end{definition}

We note that a collection of sets $\mathcal{F}$ over a ground set
is equivalent to a class of Boolean-valued functions on the same ground set.
With this terminology, we have the following fact.

\begin{fact}[VC Inequality, see, e.g.,~\cite{DL:01} or Theorem 8.3.3 in \cite{Ver18}]\label{lem:vc}
Let $\cal F$ be a class of Boolean-valued functions with
$\textsc{VC}({\cal F})\geq 1$.  Let $X_1, \ldots, X_N$ be $N$ i.i.d. samples from a distribution $\D$.
Then
$$\E_{\eps_i}\left[ \sup_{f\in \cal F} \left| \frac{1}{N}\sum_{i=1}^N \eps_i f(X_i)\right|\right] \leq C \sqrt{\textsc{VC}({\cal F})/N}\;,$$
where $C>0$ is an absolute constant and the $\eps_i$'s are independent Rademacher random variables.
\end{fact}

We are ready to bound the sample complexity required to check if
Algorithm~\ref{alg:find_certificate} finds a certificate.

\begin{proof}[Proof of Lemma~\ref{lem:check_certificate}]
The proof is a simple application of the VC inequality. In more detail, we first use the
bounded-difference inequality and then, using the symmetrization, we can apply
the VC inequality to obtain the desired result.

For $N=O(\log(1/\delta)/\eps^2)$, we apply Fact~\ref{lem:bounded_difference} for the function
$$f((X_1,Y_1),\ldots,(X_N,Y_N))=\sup_{t\in \R_+}\left|\E_{(\vec x,y) \sim \D} \left[\1_{B^{t}}(\bx) \, y \right] -
\frac{1}{N}\sum_{i=1}^N\left[\1_{B^{t}}(X_i) \, Y_i \right]\right| \;,$$
noting that $c_i=2/N$ for all $i\leq N$. Therefore, with probability at least $1-\delta$, we have that
$$ \sup_{t\in \R_+}\left| \E_{(\vec x,y) \sim \D} \left[\1_{B^{t}}(\bx) \, y \right] - \frac{1}{N}\sum_{i=1}^N\left[\1_{B^{t}}(X_i) \, Y_i \right]\right|
\leq \E\left[ \sup_{t\in \R_+}\left|\E_{(\vec x,y) \sim \D} \left[\1_{B^{t}}(\bx) \, y \right] -\frac{1}{N}\sum_{i=1}^N\left[\1_{B^{t}}(X_i) \, Y_i \right]\right| \right]+\eps \;.$$
Then, by Fact~\ref{lem:sym}, we have that
\begin{align*}
\E \left[\sup_{t\in \R_+}\left|\E_{(\vec x,y) \sim \D} \left[\1_{B^{t}}(\bx) \, y \right] - \frac{1}{N}\sum_{i=1}^N\left[\1_{B^{t}}(X_i) \, Y_i \right]\right|\right]
&\leq 2 \E_{\eps_i}\left[ \sup_{t\in \cal \R_+} \left| \frac{1}{N}\sum_{i=1}^N \eps_i Y_i  \1_{B^{t}}(X_i)\right|\right]\\
&= 2 \E_{\eps_i}\left[ \sup_{t\in \cal \R_+} \left| \frac{1}{N}\sum_{i=1}^N \eps_i \1_{B^{t}}(X_i)\right|\right]\;,
\end{align*}
where the last inequality follows from the fact that $Y_i \eps_i$ and $\eps_i$
have the same distribution (because $\eps_i$ and $Y_i$ are independent).
Finally, using the fact that the class of indicators of the form $\1\{x \leq t\}$ has VC dimension $1$,
Fact~\ref{lem:vc} implies that
$$\E_{\eps_i}\left[ \sup_{t\in \cal \R_+} \left| \frac{1}{N}\sum_{i=1}^N \eps_i \1_{B^{t}}(X_i)\right|\right]
= O(\sqrt{1/N})= O(\eps) \;.$$
Putting everything together completes the proof.
\end{proof}

\subsection{Useful Technical Lemma}

We are going to use the following simple fact about Tsybakov noise that shows
that large probability regions will also have large integral even if we
weight the integral with the noise function $1-2\eta(\bx)>0$.  Notice that
larger noise $\eta(\bx)$ makes $1-2\eta(\bx)$ closer to $0$, and therefore tends to
reduce the probability mass of the regions where $\eta(\bx)$ is large.
A similar lemma can be found in \cite{tsybakov2004optimal}.
\begin{lemma} \label{lem:tsybakov_expectation}
Let $\D$ be a distribution on $\R^{d} \times \{\pm 1\}$ that satisfies the $(\tsyb,\tsya)$-Tsybakov noise condition.
Then for every measurable set $S \subseteq \R^d$ it holds
$\E_{\vec x \sim D_\bx}[ \1_S(\bx) (1- 2 \eta(\bx))] \geq C_{\tsyb}^\tsya \lp( \E_{\vec x \sim \D_\bx}[ \1_S(\bx)] \rp)^{\frac 1 \tsyb}$,
where $C_{\tsyb}^\tsya = \tsyb \left( \frac{1-\tsyb}{\tsya} \right)^{\frac{1-\tsyb}{\tsyb}}$.
\end{lemma}

\noindent See~\cite{DKTZ20b} for the simple proof.

\section{Omitted Proofs from Section~\ref{sec:lc}}\label{ap:logcon}

\subsection{Proof of Lemma~\ref{lem:chow-params-complexity}}
\begin{proof}[Proof of Lemma~\ref{lem:chow-params-complexity}]
For the first condition, the lemma follows from Lemma~\ref{lem:algorithm_function_g}.
For the second condition, let
$\vec X= \E_{(\vec x,y)\sim \D}[y(\vec x\vec x^\intercal-\vec I)]$ and $\vec {\widehat{X}}=\E_{(\vec x,y)\sim \widehat \D}[y(\vec x\vec x^\intercal-\vec I)]$.
We are going to bound the variance, so we can apply Chebyshev's inequality. For $0<i,j\leq d$, we have
\begin{align*}
\var_{(\x,y)\sim \D}[\vec {\widehat{X}}_{ij}] &=\frac{1}{N}\var_{(\x,y)\sim \D}[\vec X_{ij}]
\leq \frac{1}{N}\E_{(\x,y)\sim \D}[\vec X_{ij}^2]= \frac{1}{N}\E_{(\x,y)\sim \D}[y^2(\vec x_{i}\vec x_{j}-1)^2]\\
&\leq \frac{2}{N}\left(\E_{\x\sim \D_\x}[\vec x_{i}^2\vec x_{j}^2]+1\right)\leq \frac{2}{N}\left(\sqrt{\E_{\x\sim \D_\x}[\vec x_{i}^4]\E_{\x\sim \D_\x}[\vec x_{j}^4]}+1\right)=O(1/N)\;,
\end{align*}
where the last inequality follows from the fact that the marginals of a log-concave density have sub-exponential tails.
Thus, from Chebyshev's inequality, for $0<i,j\leq d$, we have that
$$\pr_{(\x,y)\sim \D}[ |\vec {\widehat{X}}_{ij}-\vec {{X}}_{ij} |\geq \eps/d ]=O\left(\frac{d^2}{\eps^2 N}\right)\;.$$
Choosing $N=O(d^4/\eps^2)$, we have that $\snorm{F}{\vec X-\vec{\widehat{X}}}\leq \eps$
with high constant probability. This completes the proof.
\end{proof}

\subsection{Proof of Claim~\ref{clm:conditional-projected-bounds}}

\begin{proof}[Proof of Claim~\ref{clm:conditional-projected-bounds}]
For notational convenience, let $\D^\perp=\D_{B_{x_0}}^{\proj_{\vec w^\perp}}$.
Fix any unit vector $\vec u \in \bw^{\perp}$. Without loss of generality,
we may assume that $\bw = \vec e_1$ and $\vec u = \vec e_2$.
Denote by $\gamma(\x_1, \x_2)$ the marginal density of $\D$ on the first two coordinates.
We have that
$$\E_{\x \sim \D^\perp}[|\x^\intercal \vec u|]
= \frac{1}{\pr_{\D}[B_{x_0}]} \int |\x_2| \1\{x_0 \leq \x_1 \leq x_0+s'\} \gamma(\x_1, \x_2) \d \x_1 \d \x_2 \;.$$
From Fact~\ref{fact:lc-basics}, we have that
$\gamma(\x_1, \x_2) \leq (1/c) \exp(-|\x_2|/c)$, for some absolute constant $c >0$.
Therefore,
$$\frac{1}{\pr_{\D}[{B_{x_0}}]} \int_{-\infty}^\infty \int_{x_0}^{x_0+s'} |\x_2| \gamma(\x_1, \x_2) \d \x_1 \d \x_2
\leq \frac{s'}{c \pr_{\D}[{B_{x_0}}]}  \int_{-\infty}^\infty |x_2| e^{-|\x_2|/c} \d \x_2
= O(1) \;,$$
where we used that $x_0, x_0+s'$ are sufficiently small and it holds $\pr_{\D}[{B_{x_0}}] = \Theta(s')$,
see Fact~\ref{fact:lc-basics}.

We next bound the covariance.  Pick a unit vector $\vec u \in \bw^\perp$.
Without loss of generality, we may assume that $\vec u = \vec e_2$.
Let $\theta = \vec e_2^\intercal \E_{\x \sim \D^\perp}[\x]$ be
the projection of the mean of $\D^\perp$ on the direction $\vec e_2$.
To bound the maximum and minimum eigenvalues of the covariance matrix
of $\D^\perp$, we need to bound from above and below the following expectation:
$$\E_{\x \sim \D^\perp}[(\bx_2 - \theta)^2]
= \frac{1}{\pr_{\D}[{B_{x_0}}]} \int_{-\infty}^\infty \int_{x_0}^{x_0+s'} (\bx_2 - \theta)^2 \gamma(\x_1, \x_2) \d \x_1 \d \x_2 \,.$$
We first bound it from below. Using again Fact~\ref{fact:lc-basics} we know
that, for the same absolute constant $c$ as above, it holds that $\gamma(\x_1, \x_2) \geq c$
for points with distance smaller than $c$ from the origin. Therefore,
$$\frac{1}{\pr_{\D}[{B_{x_0}}]} \int_{-\infty}^\infty  \int_{x_0}^{x_0+s'}  (\bx_2 - \theta)^2 \gamma(\x_1, \x_2) \d \x_1 \d \x_2 \geq
\frac{c}{\pr_{\D}[{B_{x_0}}]} \int_{-c/\sqrt{2}}^{c/\sqrt{2}}  (\bx_2 - \theta)^2 \d \x_2  \int_{x_0}^{x_0+s'} \d \x_1 = \Omega(1) \;,
$$
where we used again the fact that $\pr_{\D}[{B_{x_0}}] = \Theta(s')$ and also picked
the worst case $\theta$ to minimize the above expression, i.e., $\theta = 0$.
We next bound the covariance eigenvalues from above. Using again the fact that
$\gamma(\x_1, \x_2) \leq c \exp(-c |\x_2|)$ for some absolute constant $c >0$,
we compute
$$\frac{1}{\pr_{\D}[{B_{x_0}}]} \int_{-\infty}^\infty  \int_{x_0}^{x_0+s'}  (\bx_2 - \theta)^2 \gamma(\x_1, \x_2) \d \x_1 \d \x_2
\leq \frac{1}{c \pr_{\D}[{B_{x_0}}]} \int_{-\infty}^\infty  \int_{x_0}^{x_0+s'} (\bx_2 - \theta)^2 e^{-|\x_2|/c} \d \x_1 \d \x_2
= O(1) \;,
$$
where we used the fact that $\theta = O(1)$, as already shown above,
and that $\pr_{\D}[{B_{x_0}}] = \Theta(s')$. This completes the proof.
\end{proof}

\subsection{Proof of Claim~\ref{clm:smooth-logcon}}
\begin{proof}[Proof of Claim~\ref{clm:smooth-logcon}]
To prove that $F$ is $L$-smooth, we need to show that
$\sup_{\snorm{2}{\vec r} \leq R} \snorm{2}{\nabla^2 F(\vec r)} \leq L$,
for some $L>0$.  We have
\begin{align*}
G(\vec r) &:= \nabla F(\vec r) = - 2\E_{\bx \sim \D_\x}[\x \x^\intercal \1\{\dotp{\vec r}{\x}\geq 0\} e^{-\dotp{\vec r}{\x}}]
\E_{\bx \sim \D_\x}[\x \min(1, e^{-\dotp{\vec r}{\x}})]\\
             &= - 2\E_{\bx \sim \D_\x}[\x \x^\intercal g_1(\vec r^\intercal \x) ] \E_{\bx \sim \D_\x}[\x g_2(\vec r^\intercal \x) ] \;,
\end{align*}
where $g_1(t) = \1\{ t \geq 0 \} e^{-t}$
and $g_2(t) = \min(1, e^{-t})$.
Using the product rule, we obtain that the derivative of $G(\vec r)$
at $\vec r$, $D G|_{\vec r}$, is the following linear function from $\R^d$
to $\R^d$:
$$
D G|_{\vec r} \vec h =
-2
\E_{\x \sim \D_\x} [\x \x^\intercal g_1'(\x^\intercal \vec r) \x^\intercal \vec h ]
\E_{\x \sim \D_\x} [\x g_2(\x^\intercal \vec r)]
-2
\E_{\x \sim \D_\x} [\x \x^\intercal g_1(\x^\intercal \vec r)]
\E_{\x \sim \D_\x} [\x  g_2'(\x^\intercal \vec r) \x^\intercal \vec h] \;,
$$
where $g'_1(t) = \delta(t) e^{-t} - \1\{t \geq 0\} e^{-t}$
(here by $\delta$ we denote the Dirac delta function),
and $g'_2(t) = -\1\{t \geq 0\} e^{-t}$.
To show that $F$ is smooth, we need to bound the operator norm of $D G|_{\vec r}$, i.e.,
$$
\sup_{\vec h: \snorm{2}{\vec h} = 1} \snorm{2}{D G|_{\vec r} \vec h} \,.
$$
Using the triangle and Cauchy-Schwarz inequalities, we can bound the first term
as follows:
\begin{align*}
&\snorm{2}{\E_{\x \sim \D_\x} [\x \x^\intercal g_1'(\x^\intercal \vec r) \x^\intercal \vec h ]
\E_{\x \sim \D_\x} [\x g_2(\x^\intercal \vec r)]}\\
&\leq \snorm{2}{\E_{\x \sim \D_\x} [\x \x^\intercal g_1'(\x^\intercal \vec r) \x^\intercal \vec h ]} \snorm{2}{ \E_{\x \sim \D_\x} [\x g_2(\x^\intercal \vec r)] }\\
&\leq \snorm{2}{\E_{\x \sim \D_\x} [\x \x^\intercal \x^\intercal \vec h\ \delta(\x^\intercal \vec r) e^{-\x^\intercal \vec r} ]} +
\snorm{2}{\E_{\x \sim \D_\x} [\x \x^\intercal \x^\intercal \vec h]}\snorm{2}{ \E_{\x \sim \D_\x} [\x]} \;.
\end{align*}
We will first handle the term
$\snorm{2}{\E_{\x \sim \D_\x} [\x \x^\intercal \x^\intercal \vec h\ \delta(\x^\intercal \vec r)]}$.
To simplify notation, we may set without loss of generality $\vec r = \vec e_1$.  We have
$$\E_{\x \sim \D_\x} [\x \x^\intercal \x^\intercal \vec h\ \delta(\x^\intercal \vec r) e^{-\x^\intercal \vec r} ]
=\E_{\x' \sim \D_\x'}[\x' (\x')^\intercal (\x')^\intercal \vec h\gamma(0)]\;,$$
where $\D_\x'$ is the distribution $\D_\x$ conditioned on $\x_1=0$, and $\gamma(0)$ is the one-dimensional
p.d.f. at point $0$ (which is bounded by a universal constant for log-concave distributions).
Note that $\D_\x'$ is still log-concave.

Since $\D_\x$ is $(O(1),O(1))$-isotropic, it holds
$$\snorm{2}{\E_{\x' \sim \D_\x'}[\x' (\x')^\intercal (\x')^\intercal \vec h} \leq
\E_{\x \sim \D_\x'} [\snorm{2}{\x' (\x')^\intercal( \x')^\intercal \vec h}]\leq \E_{\x \sim \D_\x'} [\snorm{2}{\x'}^3]\leq \poly(d)\;,
$$
where we used that $\snorm{2}{\vec A \vec B}\leq \snorm{2}{\vec A}\snorm{2}{\vec B}$, and that $\snorm{2}{\vec h}=1$.
Similarly, $\snorm{2}{\E_{\x \sim \D_\x} [\x \x^\intercal \x^\intercal \vec h]}\leq \poly(d)$.
Finally,
$$\snorm{2}{ \E_{\x \sim \D_\x} [\x] }\leq \E_{\x \sim \D_\x}  [\snorm{2}{\x}]\leq \poly(d)\;.$$
Putting everything together, we get that $L=\poly(d)$, which completes the proof.
\end{proof}
 \section{Omitted Proofs from Section~\ref{sec:ogd}}\label{ap:ogd}

\subsection{Proof of Proposition~\ref{pro:certificate_optimization}}
We will require the following standard regret bound
from online convex optimization.

\begin{lemma}[see, e.g., Theorem 3.1 of \cite{hazan2016introduction}]\label{lem:online_optimization}
Let ${\cal V}\subseteq \R^n$ be a non-empty closed convex set with diameter $K$.
Let $\ell_1,\ldots, \ell_T$ be a sequence of T convex functions $\ell_t: {\cal V}\mapsto \R$
differentiable in open sets containing $\cal V$, and let $G=\max_{t\in[T]}\snorm{2}{\nabla_{\bw} \ell_t}$.
Pick any $\vec w^{(1)}\in \cal V$ and set $\eta_t=\frac{K}{G\sqrt{t}}$ for $t\in[T]$. Then, for all
$\vec u\in \cal V$, we have that
$\sum_{t=1}^{T}( \ell_t(\vec w^{(t)}) -\ell_t(\vec u))\leq \frac 32 GK\sqrt{T}$.
\end{lemma}

For the set ${\cal B}$, i.e., the unit ball with respect the $\snorm{2}{\cdot}$, the diameter $K$ equals to $2$.
We will show that the optimal vector $\wstar$ and our current candidate vector $\sample{\vec w}{t}$
have a separation in the value of $\ell_t$. Since we do not have access to $\ell_t$ precisely,
we need a function $\hat{\ell}_t$, which is close to $\ell_t$ with high probability.
The following simple lemma gives us an efficient way to compute an approximation $\hat{\ell}_t$ of $\ell_t$.

\begin{lemma}[Estimating the function $\ell_t$]\label{lem:algorithm_function_ell}
Let $\D$ be a $(3, L, R, \beta)$-well-behaved distribution and  $T_{\bw}(\bx)$ be
the non-negative function given by a $\rho$-certificate oracle.
Then after drawing $O(d\beta^2/\eps^2\log(d/\delta))$ samples from $\D$,
with probability at least $1-\delta$, the empirical distribution $\widehat{\D}$
satisfies the following conditions:
\begin{itemize}
\item $\left| \E_{(\bx,y) \sim \widehat{\D}} [(T_{\bw}(\vec x) +\frac{\rho}{2}) y\dotp{\vec u}{\vec x}]-
\E_{(\bx,y) \sim \D} [(T_{\bw}(\vec x) +\frac{\rho}{2}) y\dotp{\vec u}{\vec x}] \right| \leq \eps$,
for any  $\vec u \in {\cal B}$.

\item $\snorm{2}{\E_{(\bx,y) \sim \widehat{\D}} [(T_{\bw}(\vec x) +\frac{\rho}{2}) y\vec x]} \leq 1+\frac{\rho}{2}+\eps$.\end{itemize}
\end{lemma}
\begin{proof}
The proof of this lemma is similar to the proof of Lemma~\ref{lem:algorithm_function_g}.
Let $\hat{\vec g}= \E_{(\bx,y) \sim \widehat{\D}} [(T_{\bw}(\vec x) +\frac{\rho}{2}) y\vec x]$ and
$\vec g=\E_{(\bx,y) \sim \D} [(T_{\bw}(\vec x) +\frac{\rho}{2}) y\vec x]$.
For any unit vector $\vec u$, we have
\begin{align*}
|\dotp{\vec u}{\vec g}|\leq \E_{\bx \sim \D_\bx} [|T_{\bw}(\vec x)||\dotp{\vec u}{ \vec x}| ]+
\frac{\rho}{2} \E_{\bx \sim \D_\bx} [|\dotp{\vec u}{ \vec x}| ] \leq 1+\frac{\rho}{2}\;,
\end{align*}
where we used that $|T(\bx)|\leq 1$ and that the distribution $\D_\bx$ is in isotropic position.
Moreover, from Fact~\ref{fct:tails}, the random variable $X= (T_{\bw}(\vec x) +\frac{\rho}{2}) y \vec x - \vec g$
is sub-exponential with tail bound $\beta'=O(\beta)$.
Thus, the rest of proof follows as in Lemma~\ref{lem:algorithm_function_g}.
\end{proof}

The last item we need to proceed with our main proof is to establish
that when the oracle $\cal C$ in Step~\ref{alg:cerficate} of Algorithm~\ref{alg:OPGD-oracle}
returns a function $T_{\sample{\bw}{t}}$, then there exists a function $\ell_t$
for which our current candidate vector $\sample{\bw}{t}$ and the optimal vector $\wstar$ are not close.

\begin{lemma}[Error of $\ell_t$]\label{lem:expectation_error}
Let $\sample{\vec w}{t} \in \mathcal{B}$ and $\wstar$ be the optimal weight vector.
For $g_t(\vec x)= -(T_{\sample{\bw}{t}}(\vec x)+ \frac{\rho}{2} )$ and
$\ell_t(\vec w)= \E_{(\bx,y) \sim \D} [\dotp{g_t(\vec x) y {\vec x}}{\vec w}]$,
where $T_{\sample{\bw}{t}}(\vec x)$ is the function given by a $\rho$-certificate oracle,
we have that
$$\ell_t\left( \wstar \right) \le -\rho \alpha\left(\frac{R\;L}{A}\right)^{O(1/\tsyb)}
	\quad\mathrm{ and } \quad  \ell_t({\vec w}^{(t)}) \geq \snorm{2}{{\vec w}^{(t)}}\frac{\rho}{2}\;.$$
\end{lemma}
\begin{proof}
Without loss of generality, let $\wstar={\vec e_1}$.
From Fact~\ref{obs:optimal_condition} and the definition of
$\eta(\bx)$, we have that for every $t\in[T]$, it holds
$\ell_t(\wstar) \leq -\lambda\E_{\vec x\sim \D_{\bx}}[|\dotp{\wstar}{\vec x}|(1-2\eta(\bx))]$.
To bound from above this expectation, we use the $\bounded$-bound properties.
We have that
\begin{align}
\E_{\vec x\sim \D_{\bx}}[|\dotp{\wstar}{\vec x}|(1-2\eta(\bx))]& \geq \frac{R}{4} C_{\tsyb}^{\tsya}\left(\frac{R^3\;L}{2}\right)^{1/\tsyb} \nonumber\;,
\end{align}
where in the last inequality we used Lemma~\ref{lem:tsybakov_expectation}. Therefore,
$\ell_t\left( \wstar \right) \le -\frac \rho 2 \, \frac{R}{4} C_{\tsyb}^{\tsya}\left(\frac{R^3\;L}{2}\right)^{1/\tsyb}$.
Then we bound from below $\ell_t({\vec w^{(t)}})$ as follows
\begin{align*}
\ell_t({\vec w^{(t)}})
&=- \E_{(\bx,y) \sim \D} \left[(T_{\sample{\bw}{t}}\left( \bx \right)+\lambda) \dotp{{\vec w^{\left( t \right) }}}{\bx}y\right] \ge \snorm{2}{{\sample{\vec w}{t} }} \rho - \E_{\bx \sim \D_{\x}} \left[\frac \rho 2\dotp{{\vec w^{(t)}}}{\bx}y\right] \\
& \ge \snorm{2}{{\vec w^{(t)}}} \rho- \frac \rho 2 \sqrt{ \E_{\bx \sim \D_{\x}} \left[\dotp{\vec w^{(t)}}{\bx}^2\right] }
\geq \snorm{2}{{\vec w}^{(t)}} \frac \rho 2  \;,
\end{align*}
where we used the Cauchy-Schwarz inequality and the fact that $\x$ is in isotropic position.
\end{proof}

We are ready to prove Proposition~\ref{pro:certificate_optimization}.

\begin{proof}[Proof of Proposition~\ref{pro:certificate_optimization}]
Let $G=  \alpha\left(\frac{R\;L}{A}\right)^{O(1/\tsyb)}$.
Assume, in order to reach a contradiction, that for all steps $t\in[T]$ it holds that
$\theta\left( \sample{\vec w}{t},\wstar \right)\ge \eps $. For each step $t$, let  $T_{\sample{\vec w}{t} }(\bx)$
be the non-negative function output by the oracle ${\cal C}(\sample{\vec w}{t},\eps,\delta/T)$.
Note that
$$\E_{(\bx,y) \sim \D}[T_{\sample{\vec w}{t} }(\bx)y\dotp{\vec w^{(t)}}{\bx}]\leq-\snorm{2}{\vec w^{(t)}} \frac \rho2 \;.$$
Let $ \hat{\ell_t}(\vec w)$ be the empirical estimator of
$\ell_t\left( \vec w \right)=\E[\hat{\ell_t}(\vec w)]=  -\E_{(\bx,y) \sim \D} [\dotp{\left(T_{\sample{\vec w}{t} }(\bx) +\frac{\rho}{2} \right) y {\vec x}}{\vec w}]$. Using Lemma~\ref{lem:algorithm_function_ell},
for $N=O\left(\frac{d \beta^2}{\rho^2 G^2}\log\left( \frac{T}{\delta}\right) \right)$ samples, we have that
$\pr\left[ |\hat{\ell_t}(\vec w^{(t)})-\ell_t(\vec w^{(t)})|\geq \frac 14G\rho\right]\leq \frac{\delta}{2T}$ and
$\pr\left[ | \hat{\ell_t}(\vec w^{\ast})-\ell_t(\vec w^{\ast})|\geq \frac 14 G\rho\right]\leq \frac{\delta}{2T}$.

From Lemma~\ref{lem:expectation_error}, for every  step $t$, we have that
$ \ell_t({\vec w}^{(t)}) \geq \frac 12\snorm{2}{{\vec w}^{(t)}}\rho\geq 0$ and
$\ell_t\left( \wstar \right) \le -\rho G$, thus, with probability at least $1-\frac{\delta}{T}$,
$\hat{\ell_t}({\vec w}^{(t)}) \geq-\frac 14G\rho$ and $\hat{\ell_t}\left( \wstar \right) \le -\frac 34G\rho $.
Using Lemma~\ref{lem:online_optimization},  we get
\[ \frac 1T \sum_{t=1}^{T}\left(\hat{\ell_t}\left( {\vec w^{(t)}} \right)   - \hat{\ell_t}\left( {\vec w^*} \right)\right)
\leq \frac{1+\frac \rho2 + \frac 14\rho G }{\sqrt{T}}\;.\]
By the union bound, it follows that with probability at least $1-\delta$, we have that
\[ \frac12G\rho \leq 	\frac 1T \sum_{t=1}^{T}\left(\hat{\ell_t}\left( {\vec w^{(t)}} \right)   - \hat{\ell_t}\left( {\vec w^{\ast}} \right)\right)
\leq \frac{4}{\sqrt{T}} \;,\]
which leads to a contradiction for $T=\frac{16}{(\rho G)^2}$.

Thus, either there exists $t\in [T]$ such that
$\theta\left( \sample{\vec w}{t},\wstar \right)< \eps $, which the algorithm returns in Step~\ref{alg:final_return},
or the oracle $\cal C$ did not provide a correct certificate, which happens with probability at most $\delta$.
Moreover, the algorithm calls the certificate $T$  times and the number of samples needed to construct
the empirical distribution $\widehat{\D}$ is
$$O(T \, N)=\frac{d\beta^2}{\rho^4}\log\left( \frac{1}{\delta \rho}\right) \frac 1\alpha\left(\frac{A}{R\;L}\right)^{O(1/\tsyb)} \;.$$
This completes the proof.
\end{proof}

Using Proposition~\ref{pro:certificate_optimization} and our certificate algorithms,
we obtain the following parameter estimation result for halfspaces with Tsybakov noise.

\begin{theorem}[Parameter Estimation of Tsybakov Halfspaces Under Well-Behaved Distributions]  \label{thm:main_angle}
Let $\D$ be a $(3, L, R, \beta)$-well-behaved isotropic distribution on $\R^{d} \times \{\pm 1\}$ that
satisfies the $(\tsyb,\tsya)$-Tsybakov noise condition with respect to an unknown halfspace $f(\bx) = \sgn(\dotp{\vec w^{\ast}}{\bx} )$.
There exists an algorithm that draws $N=\beta^4 \left(\frac{d \, \tsya}{R L  \eps}\right)^{O(1/\tsyb)} \log\left(1/\delta\right)$
samples from $\D$, runs in $\poly(N,d)$ time, and computes a vector $\wh{\vec w}$ such that with probability $1-\delta$
we have $\theta(\wh{\vec w}, \wstar) \leq \eps$.
\end{theorem}

We note here that Theorem~\ref{thm:main_angle} does not require the ``$U$ bounded'' condition of the underlying distribution on examples
that is required in our Theorem~\ref{thm:pac-wb}.  Recall that this condition corresponds to an anti-concentration property of the data distribution.
With this additional property, Theorem~\ref{thm:pac-wb} follows easily from Theorem~\ref{thm:main_angle},
since it allows us to translate the small angle guarantee of Theorem~\ref{thm:main_angle} to the
zero-one loss.

\begin{proof}[Proof of Theorem~\ref{thm:main_angle}]
We start by noting how to obtain a $\rho$-certificate oracle for $(3, L, R, \beta)$-well-behaved
distributions. The algorithm of Theorem~\ref{thm:cert-wb},  returns a function $T_\bw$
such that $\E_{(\vec x,y) \sim \D} \left[T_\bw(\bx) y \dotp{\vec w}{\bx} \right] \leq - \frac 1\beta\left(\theta L R/(d\tsya)\right)^{O(1/\alpha)}$.
By definition, the function $T_{\bw}$ (i.e., Equation~\eqref{eq:certificate_form}) is bounded, namely
$\snorm{\infty}{T_\bw}\leq \frac{1}{\min_{\x \in B} |\dotp{\vec w}{\vec x} | } \leq O\left(\frac{d}{ \theta}\right)$, where $B$ is the band from Equation~\eqref{eq:certificate_form}.
Therefore, the function  $T_\bw/\snorm{\infty}{T_\bw}$ satisfies the conditions of a $\rho$-certificate oracle. Thus,
by scaling the output of the algorithm of Theorem~\ref{thm:cert-lc}, we obtain a
$\frac{1}{\beta}\left(\theta L R/(d\tsya)\right)^{O(1/\alpha)}$-certificate oracle.
From Proposition~\ref{pro:certificate_optimization}, this gives us an algorithm that returns a vector
$\widehat{\bw}$ such that   $\theta(\widehat{\vec w}, \wstar) \leq\eps$ with probability $1-\delta$.
\end{proof}

To prove Theorem~\ref{thm:pac-wb}, we need the following claim for $\boundedU$-well-behaved
distributions.
\begin{claim}[see, e.g., Claim 2.1 of \cite{DKTZ20}] \label{lem:angle_zero_one}
Let $\D_{\bx}$ be an $\boundedU$-well-behaved distribution on $\R^d$.
Then, for any $0< \eps \leq 1$, we have that $\err_{0-1}^{\D_{\bx}}(h_{\vec u},h_{\vec v})
\leq U  \beta^2\log^2\left(1/\eps \right)\cdot \theta(\vec v, \vec u) + \eps \;.$
\end{claim}

\begin{proof}[Proof of Theorem~\ref{thm:pac-wb}]
Running Algorithm~\ref{alg:OPGD-oracle} for
$\eps' = \frac{\eps  }{2 U\conb^2} \frac{1}{\log^2(2 / \eps)}$, by
Theorem~\ref{thm:main_angle}, Algorithm~\ref{alg:OPGD-oracle}
outputs a $\widehat{\vec w}$ such that
$\theta(\widehat{\vec w},\wstar)\leq \frac{\eps }{2 U \beta^2} \frac{1}{2 \log^2(1/ \eps)}$,
then from Claim~\ref{lem:angle_zero_one}, we
have $\err_{0-1}( h_{\widehat{\vec w}},f)\leq \eps$.
\end{proof}

\appendix

\end{document}